\documentclass[letterpaper]{article} 
\usepackage{aaai2026}  
\usepackage{times}  
\usepackage{helvet}  
\usepackage{courier}  
\usepackage[hyphens]{url}  
\usepackage{graphicx} 
\usepackage{natbib}  
\usepackage{caption} 
\usepackage[utf8]{inputenc}
\usepackage{microtype}
\usepackage{siunitx}

\usepackage{booktabs}
\usepackage{multirow}
\usepackage{subcaption}
\usepackage{makecell}
\usepackage{adjustbox}
\usepackage{arydshln}

\usepackage{amsmath}   
\usepackage{amssymb}   
\usepackage{amsfonts}  
\usepackage{amsthm}    
\usepackage{nicefrac}  

\usepackage{algorithm}
\usepackage{algpseudocode}
\usepackage{tabularx}

\theoremstyle{plain}
\newtheorem{theorem}{Theorem}[section]
\newtheorem{proposition}[theorem]{Proposition}

\usepackage{amsmath}   
\usepackage{amssymb}   
\usepackage{amsfonts}  
\usepackage{amsthm}    
\usepackage{nicefrac}  

\usepackage{algorithm}
\usepackage{algpseudocode}
\usepackage{tabularx}

\DeclareMathOperator*{\argmin}{arg\,min}

\DeclareMathOperator{\sigmoid}{sigmoid}

\DeclareMathOperator{\RevIN}{RevIN}
\DeclareMathOperator{\InverseRevIN}{InverseRevIN}
\DeclareMathOperator{\EMADecomposition}{EMADecomposition}
\DeclareMathOperator{\TrendMLP}{TrendMLP}
\DeclareMathOperator{\EmbedOp}{Embed}
\DeclareMathOperator{\rfftOp}{rfft}
\DeclareMathOperator{\FreMLP}{FreMLP}
\DeclareMathOperator{\irfftOp}{irfft}
\DeclareMathOperator{\StrongSeasonalMLP}{StrongSeasonalMLP}
\DeclareMathOperator{\WeakSeasonalMLP}{WeakSeasonalMLP}
\DeclareMathOperator{\DimPermute}{DimPermute}
\DeclareMathOperator{\Concat}{Concat}
\DeclareMathOperator{\GeLUOp}{GeLU}
\DeclareMathOperator{\SigmoidOp}{Sigmoid}
\DeclareMathOperator{\DropoutOp}{Dropout}
\DeclareMathOperator{\BroadcastOp}{Broadcast}

\usepackage{newfloat}
\usepackage{listings}
\DeclareCaptionStyle{ruled}{labelfont=normalfont,labelsep=colon,strut=off} 
\floatstyle{ruled}
\newfloat{listing}{tb}{lst}{}
\floatname{listing}{Listing}
\title{MDMLP-EIA: Multi-domain Dynamic MLPs with Energy Invariant Attention for Time Series Forecasting}
\author{
    Hu Zhang \textsuperscript{\rm 1,\rm 2},
    Zhien Dai \textsuperscript{\rm 3}\thanks{Corresponding author},
    Zhaohui Tang \textsuperscript{\rm 3},
    Yongfang Xie \textsuperscript{\rm 3}
}
\affiliations{
    \textsuperscript{\rm 1}College of Computer Science and Engineering, Changsha University, China \\ 
    \textsuperscript{\rm 2}School of Mathematics and Statistics, Central South University, China\\ 
    \textsuperscript{\rm 3}School of Automation, Central South University, China\\ 
    zhanghu@csu.edu.cn,  zhiendai@csu.edu.cn, zhtang@csu.edu.cn, yfxie@csu.edu.cn

}

\usepackage{bibentry}

\begin{document}

\maketitle

\begin{abstract}
Time series forecasting is essential across diverse domains. While MLP-based methods have gained attention for achieving Transformer-comparable performance with fewer parameters and better robustness, they face critical limitations including loss of weak seasonal signals, capacity constraints in weight-sharing MLPs, and insufficient channel fusion in channel-independent strategies.
To address these challenges, we propose MDMLP-EIA (Multi-domain Dynamic MLPs with Energy Invariant Attention) with three key innovations.
First, we develop an adaptive fused dual-domain seasonal MLP that categorizes seasonal signals into strong and weak components. It employs an adaptive zero-initialized channel fusion strategy to minimize noise interference while effectively integrating predictions.
Second, we introduce an energy invariant attention mechanism that adaptively focuses on different feature channels within trend and seasonal predictions across time steps. This mechanism maintains constant total signal energy to align with the decomposition-prediction-reconstruction framework and enhance robustness against disturbances.
Third, we propose a dynamic capacity adjustment mechanism for channel-independent MLPs. This mechanism scales neuron count with the square root of channel count, ensuring sufficient capacity as channels increase.
Extensive experiments across nine benchmark datasets demonstrate that MDMLP-EIA achieves state-of-the-art performance in both prediction accuracy and computational efficiency.
\end{abstract}

\begin{links}
    \link{Code}{https://github.com/zh1985csuccsu/MDMLP-EIA}
\end{links}

\section{Introduction}

Time series forecasting (TSF) has extensive applications in air quality monitoring \cite{zheng2015forecasting}, traffic flow prediction \cite{yin2021deep}, energy management \cite{amasyali2018review}, economic analysis \cite{zhang2017stock}, and industrial indicator forecasting \cite{zhang2025multihorizon}. In recent years, deep learning methods for time series prediction have received widespread research attention \cite{shao2024exploring}. These methods include RNN-based approaches (e.g., LSTNet\cite{lai2018modeling}, Witran \cite{jia2023witran}, DAN \cite{li2024learning}), CNN-based methods (e.g., SciNet \cite{liu2022scinet}, TimesNet \cite{wu2022timesnet}, Moderntcn \cite{luo2024moderntcn}), GNN-based methods (e.g., StemGNN \cite{cao2020spectral}, FourierGNN \cite{yi2023fouriergnn}, Msgnet \cite{cai2024msgnet}), and Transformer-based methods (e.g., Informer \cite{zhou2021informer}, PatchTST \cite{nie2022time}, Crossformer \cite{zhang2023crossformer}, Ada-MSHyper \cite{shang2024ada}).

\begin{figure*}[!htb]
    \centering
    \includegraphics[width=0.70\linewidth]{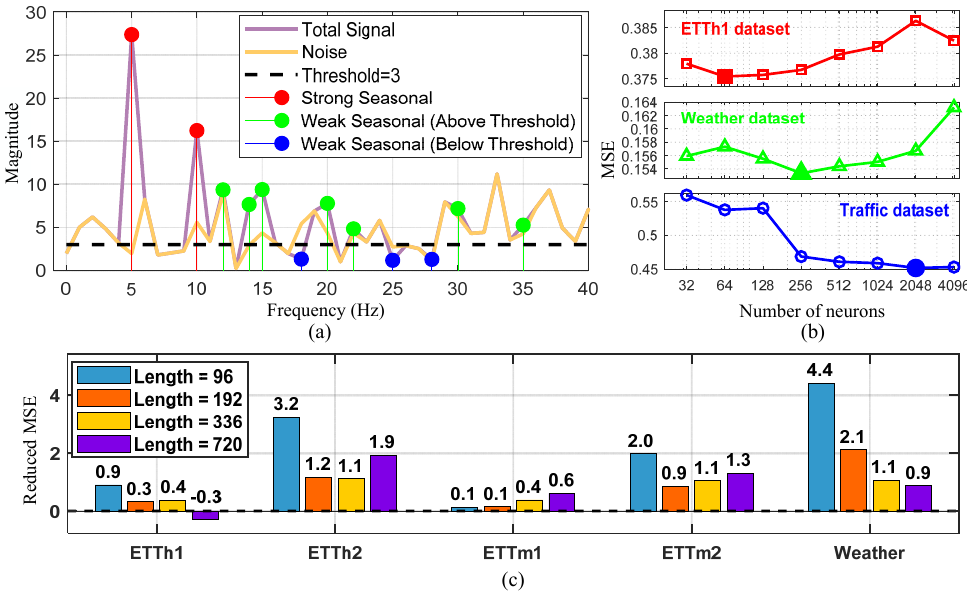}
    \caption{Limitations of current MLP-based methods.
    (a) Loss of weak seasonal signals: Weak seasonal signals closely resemble noise signals and are difficult to distinguish from them. When applying frequency domain amplitude restrictions to reduce noise, some weak seasonal signals are inevitably lost.
    (b) Different capacity requirements of channel-independent MLPs for prediction tasks with varying numbers of channels: For prediction tasks with larger numbers of channels (such as Traffic), MLPs require greater capacity to meet the predictive demands of each channel. Conversely, for prediction tasks with fewer channels (such as ETTh1), MLPs require smaller capacity to prevent model overfitting.
    (c) Reduced Mean Squared Errors (MSEs) achieved by the proposed energy invariant attention at varied prediction lengths in some datasets.
    }
    \label{fig:Limitation}
\end{figure*}

These deep learning models achieve good prediction results in certain scenarios, but typically suffer from complex structures, slow training/inference speeds, and reduced robustness due to numerous parameters, particularly with limited training data \cite{zhou2021informer}. DLinear \cite{zeng2023transformers} addresses these issues by introducing a simple single-layer linear model that achieves Transformer-comparable prediction accuracy with fewer parameters. This success has sparked extensive research into MLP-based time series prediction methods, including FreTS \cite{yi2023frequency}, CycleNet \cite{lin2024cyclenet}, FilterNet \cite{yi2024filternet}, and Amplifier \cite{fei2025amplifier}. These methods enhance temporal feature representation and improve MLP performance through trend-seasonal decomposition and frequency domain signal processing.
However, several limitations still remain:
\textbf{\textit{1) Loss of weak seasonal signal}}: 
Softshrink is commonly employed to reduce noise in the frequency domain \cite{yi2023frequency}. However, this approach may inadvertently remove weak seasonal signals that have low energy but contain valuable cyclical information for time series prediction (see Figure \ref{fig:Limitation}(a)). Although Amplifier \cite{fei2025amplifier} attempts to address the loss of weak seasonal signals by transferring high-energy components from low-frequency to high-frequency regions through spectrum flipping, it simultaneously amplifies unwanted noise along with useful weak seasonal signals.
\textbf{\textit{2) Capacity limitation of weight-sharing MLPs and insufficient channel fusion for channel-independent strategy}}: 
The channel-independent strategy employs weight sharing to merge values from different time steps across each channel dimension, and it is considered as a robust prediction strategy with fewer model parameters \cite{han2024capacity}. As the channel-independent MLPs need to handle temporal fusion of features across multiple different channels, their capacity is closely related to the fusion performance. However, current channel-independent MLPs usually use fixed capacity, which limits their performance when processing tasks with varying numbers of channels, as shown in Figure \ref{fig:Limitation}(b). 
Furthermore, the channel-independent strategy may result in insufficient channel fusion for TSF, particularly for popular forecasting methods via seasonal-trend decomposition \cite{fei2025amplifier}\cite{wang2024timemixer}\cite{stitsyuk2025xpatch}\cite{wu2021autoformer}\cite{zhou2022fedformer}. 
These methods directly sum trend and seasonal predictions to generate final outputs, potentially limiting the model's ability to represent different channels across time steps, as shown in Figure \ref{fig:Limitation}(c).

To address these issues, we propose MDMLP-EIA with the following main contributions: 

1) To effectively capture weak seasonal signals while minimizing noise interference, we design an \textit{\textbf{adaptive fused dual-domain seasonal MLP}}. We partition seasonal signals into strong and weak components, processing them through frequency-domain MLP and standard MLP respectively. Especially, since both noise and weak seasonal signals exhibit similar low energy characteristics in the frequency domain, we develop an \textit{\textbf{Adaptive Zero-initialized Channel Fusion (AZCF) strategy}} that selectively integrates strong and weak seasonal predictions while suppressing noise contamination. 
The idea is different with the spectrum flipping in Amplifier \cite{fei2025amplifier}, which transfers high-energy components from low-frequency to high-frequency regions and unavoidably amplifies unwanted noise along with useful weak seasonal signals. 

2) We propose a \textbf{\textit{dynamic capacity adjustment (DCA) mechanism}} for channel-independent MLPs. 
This mechanism adjusts neuron count based on the square root of channel numbers, ensuring adequate MLP capacity scaling. For small channel counts, minimal inter-channel redundancy requires rapid capacity growth; for large channel counts, increased redundancy permits slower growth, preventing parameter explosion and overfitting.

3) To address insufficient channel fusion, we introduce an \textit{\textbf{energy invariant attention mechanism (EIA)}} that adaptively weights trend and seasonal predictions while maintaining theoretical consistency with decomposition frameworks. Unlike conventional weighted fusion that may suffer from energy inconsistency and signal distortion, EIA preserves total signal energy, ensuring stable fusion and enhanced robustness.

4) Nine benchmark experiments achieve state-of-the-art accuracy and efficiency.

\section{Related works}
\paragraph{MLP-based forecasting methods} N-BEATS \cite{oreshkin2019n} stacks multiple fully connected layers with forward and backward residuals for univariate TSF. 
LightTS \cite{zhang2207less} applies an MLP-based structure on top of two delicate down-sampling strategies for fast TSF. 
DEPTS \cite{fan2022depts} employs local and seasonal MLP modules to capture local trends and global periodicity in time series, respectively.
NHITS \cite{challu2023nhits} introduces multi-rate data sampling and hierarchical interpolation to focus on time series components of varying frequencies.
TimeMixer \cite{chen2023tsmixer} uses feature-dimension and time-dimension weight-sharing MLPs for efficient time series prediction with minimal parameters.
DLinear \cite{zeng2023transformers} designs a simple single-layer linear model to establish dynamic relationships between input and output time series.
SOFTS \cite{han2024softs} employs a centralized structure to extract global key representations distributed to individual channels, effectively capturing channel correlations.

\paragraph{Forecasting in the frequency domain} FITS \cite{xu2023fits} introduces a low-pass filter and implements MLP interpolation in the frequency domain for efficient time series prediction. 
FilterNet \cite{yi2024filternet} designs plain and contextual shaping filters to capture key temporal patterns in time series.
CycleNet \cite{lin2024cyclenet} proposes a residual cycle forecasting technique that iteratively extracts periodic patterns from time series.
Peri-midFormer \cite{wu2024peri} decomposes time series by cycle and utilizes a pyramid structure to merge and predict signals of different periodicities.
SparseTSF \cite{lin2024sparsetsf} simplifies time series prediction to cross-period trend prediction by decomposing series into sub-sequences through fixed-period downsampling.
FreTS \cite{yi2023frequency} uses a frequency-domain MLP with dual frequency learners across channel and temporal dimensions for time series representation and prediction.

\paragraph{Forecasting via seasonal-trend decomposition} 
Autoformer \cite{wu2021autoformer} replaces self-attention with auto-correlation for trend-seasonal decomposition.
FEDformer \cite{zhou2022fedformer} utilizes frequency enhanced blocks and attention to improve time series representation in the frequency domain.
TimeMixer \cite{wang2024timemixer} applies multiscale downsampling to time series signals followed by trend-seasonal decomposition and prediction for each downsampled series.
Amplifier \cite{fei2025amplifier} enhances low-energy components through spectrum flipping during frequency-domain processing. 
xPatch \cite{stitsyuk2025xpatch} utilizes patching and channel-independence techniques with Exponential Moving Average (EMA) for seasonal-trend decomposition in time series prediction.

\section{MDMLP-EIA}
\label{headings}
\begin{figure*}[!htb]
    \centering
    \includegraphics[width=1.0\linewidth]{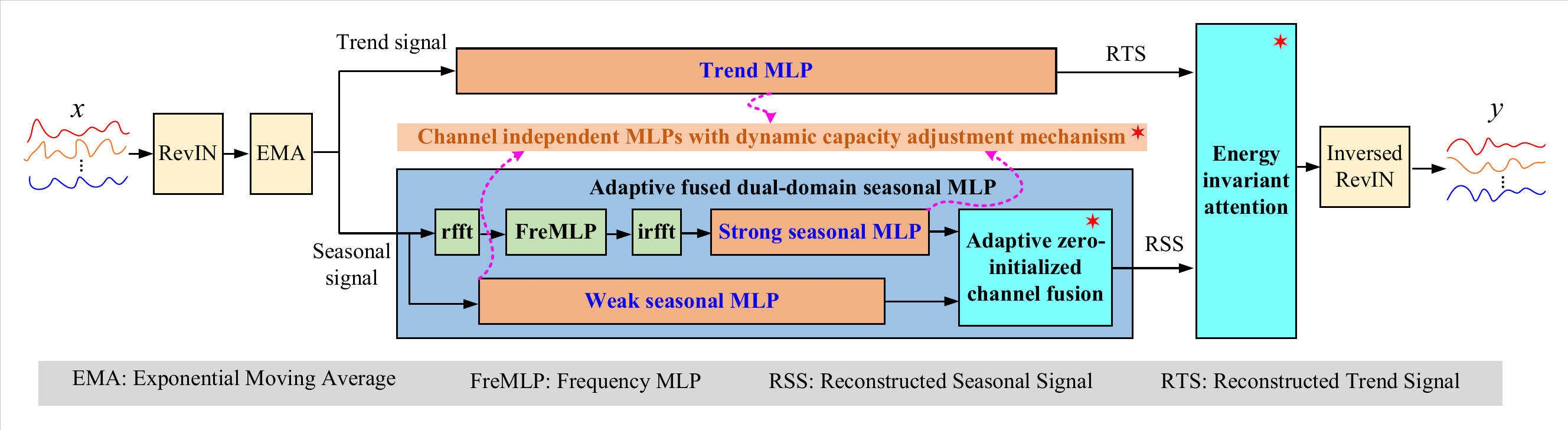}
    \caption{MDMLP-EIA overall architecture. 
    (i) RevIN normalization and EMA decompose input series into trend and seasonal components;   
    (ii) Trend component feeds into trend MLP; seasonal component processes through our \textbf{adaptive fused dual-domain seasonal MLP} with \textbf{adaptive zero-initialized channel fusion});
    (iii) Our \textbf{energy invariant attention module} merges trend and seasonal predictions; 
    (iv) Our \textbf{Dynamic capacity adjustment mechanism} optimizes multi-domain MLPs.
    }
    \label{fig:MDMLP-EIA}
\end{figure*}

Figure \ref{fig:MDMLP-EIA} illustrates our proposed MDMLP-EIA framework (\textit{Appendix A shows the overall pseudocode}).
Let \textit{L} denote input historical timesteps, \textit{Q} future timesteps, and \textit{C} channels per timestep. The input series $x \in \mathbb{R}^{L \times C}$ undergoes RevIN normalization, and EMA decomposes \textit{x} into trend component \textit{$x_1$} and seasonal component \textit{$x_2$}. 
Then, a trend MLP generates trend prediction \textit{$y_1$} from \textit{$x_1$} (\textit{See Appendix B for details}).
Meanwhile, our proposed  \textbf{\textit{1) adaptive fused dual-domain seasonal MLP}} produces seasonal prediction \textit{$y_2$} from \textit{$x_2$}. 
After that, an \textbf{\textit{2) energy-invariant attention}} module is designed to integrate \textit{$y_1$} and \textit{$y_2$}. 
Finally, the output series $y \in \mathbb{R}^{Q \times C}$ is obtained by inverse RevIN normalization. 
All the above multi-domain MLPs are channel-independent and adopt our proposed \textbf{\textit{3) DCA mechanism}}.

\subsection{Adaptive fused dual-domain seasonal MLP}
The adaptive fused dual-domain seasonal MLP uses two parallel networks: one processes strong seasonal signals, the other extracts weak seasonal patterns.

\subsubsection{Frequency MLP and strong seasonal MLP} 
Unlike xPatch \cite{stitsyuk2025xpatch} which uses CNN, we employ a frequency-temporal learner \cite{yi2023frequency} with embedding, real FFT, frequency MLP, and inverse real FFT to reconstruct features with strong seasonal patterns \textit{$f_{s1}\in \mathbb{R}^{C \times L\times E}$}. A strong seasonal MLP then learns predictions  $y_{21} \in \mathbb{R}^{Q \times C}$ from $f_{s1}$. \textit{See Appendix C for details}.

\subsubsection{Weak seasonal MLP} 
The frequency-temporal learner \cite{yi2023frequency} uses softshrink processing to reduce noise but inevitably loses weak seasonal signals with low amplitude. Therefore, we construct a weak seasonal MLP (\textit{detailed in Appendix D}) to extract weak seasonal predictions $y_{22} \in \mathbb{R}^{Q \times C}$ from seasonal signal $x_{2} \in \mathbb{R}^{C \times L}$.

\subsubsection{AZCF mechanism}
\label{sec:azicf}
Existing time-series signal fusion methods typically use attention mechanisms to weight signals across channel and temporal dimensions. Given strong seasonal prediction $y_{21}$ and weak seasonal prediction $y_{22}$, the complete seasonal prediction $y_2$ is calculated as
\begin{equation}
y_2 = \alpha_1 \odot y_{21} + \alpha_2 \odot y_{22}  ,
\label{eq:attn strong and weak}
\end{equation}
where $\alpha_1$, $\alpha_2$ $\in \mathbb{R}^{Q \times C}$ are obtained indirectly through Query, Key, and Value transformations. However, this indirect approach introduces unnecessary complexity and neglects the significant signal-to-noise ratio differences between strong and weak seasonal predictions.
Thus, we propose AZCF mechanism with three key characteristics:

\textbf{Single-parameter fusion.} We introduce an adaptive weight coefficient $\alpha$ to modulate the weak seasonal prediction's contribution:
\begin{equation}
y_2 = y_{21} + \alpha \odot y_{22},
\label{eq:alpha strong and weak}
\end{equation}
where $\alpha$ $\in \mathbb{R}^{Q \times C}$. 
This single-parameter fusion first ensures strong signal quality, then selectively enhances weak signals while reducing model parameters.

\textbf{Channel-dimension fusion.} Since both $y_{21}$ and $y_{22}$ are obtained through channel-independent methods, the fusion strategy should respect this separation by making independent decisions along the channel dimension, learning which features have more reliable weak seasonal signals. Fusion along the time dimension would disrupt seasonal patterns' temporal coherence and may propagate noise across channels. Therefore, we define:
\begin{equation}
\alpha=\{\alpha_1, \alpha_2,..., \alpha_C\}\in\mathbb{R}^{1\times C}
\end{equation}
and expand it to $\alpha \in \mathbb{R}^{Q \times C}$ by broadcasting across the output time step dimension:
\begin{equation}
\alpha[q,c] = \alpha[c], \quad \forall q \in \{1,2,...,Q\}.
\end{equation}

\textbf{Zero initialization.} We initialize $\alpha$ as a zero vector, starting from the most reliable strong seasonal prediction:
\begin{equation}
\alpha_{\text{init}} = \mathbf{0} \in \mathbb{R}^{1\times C}.
\end{equation}
This gives us the initial prediction:
\begin{equation}
y_2^{\text{init}} = y_{21} + \alpha_{\text{init}} \odot y_{22} = y_{21}.
\label{eq:zero init}
\end{equation}
Equations \eqref{eq:alpha strong and weak} and \eqref{eq:zero init} together essentially constitute a \textit{\textbf{progressive learning strategy}}. This strategy starts from a reliable foundation and allows $\alpha$ to gradually increase only when weak seasonal predictions genuinely improve results, preventing over-reliance on potentially noisy signals during early training.
Compared to conventional attention mechanisms, our method requires only $C$ parameters versus significantly more  (typically $O(C^2)$ for Query, Key, Value transformations). With $C$ ranging from tens to hundreds, this represents substantial reduction, making our model more lightweight and reducing overfitting risk.
\textit{See Appendix E for theoretical proof. }

\subsection{EIA mechanism}
Conventional decomposition-based methods typically generate the final prediction $y_{3}$ by directly summing the trend prediction $y_{1}$ and seasonal prediction $y_{2}$ ($y_3 = y_1 + y_2$). However, they may limit the model's adaptive capacity to handle varying temporal dynamics and feature importance across channels and time steps.

\subsubsection{Normalized attention fusion}
A possible solution is to introduce normalized attention for adaptive component fusion using weighted combination:
\begin{equation}
y_3 = \beta \odot y_1 + (1-\beta) \odot y_2, 
\label{eq:standard_weighted_fusion}
\end{equation}
where $\beta \in \mathbb{R}^{Q \times C}$ represents the learnable attention weight for $y_{1}$, and $\odot$ denotes element-wise multiplication.
Since $\beta + (1-\beta) = 1$ for any $\beta \in [0,1]$, $\beta \odot y_1 + (1-\beta) \odot y_2$ represents a convex combination providing enhanced expressiveness by adaptively balancing trend and seasonal contributions. However, this creates \textbf{energy inconsistency challenges}. While time series decomposition theory reconstructs signals through $x = x_{1} + x_{2}$ (preserving total magnitude), learned weights make signal magnitude variable and $\beta$-dependent. This will cause: 1) \textit{\textbf{amplification of reconstruction errors}} with extreme $\beta$ values, 2) \textit{\textbf{inconsistent signal scaling}} across time steps and channels, and \textit{\textbf{3) noise amplification}} from uncontrolled magnitude variations that violate decomposition framework guarantees.

\subsubsection{EIA}
To address these limitations while preserving normalized attention fusion capabilities, we propose \textit{EIA} that generates the fused prediction $y_3 \in \mathbb{R}^{Q \times C}$ as:
\begin{equation}
y_3 = 2 \times (\beta \odot y_1 + (1-\beta) \odot y_2),
\label{eq:energy_invariant_attention}
\end{equation}
where \textbf{$2\times$}  ensures energy invariance by compensating for implicit normalization in the convex combination, maintaining consistent signal magnitude regardless of $\beta$ values. 
The proposed EIA mechanism provides: 1) \textit{\textbf{Perfect reconstruction preservation}}: When $\beta = 0.5$, it reduces to $y_3 = y_1 + y_2$, matching traditional decomposition methods; 2) \textit{\textbf{Stable adaptive fusion}}: When $\beta \neq 0.5$, the model adaptively emphasizes components while maintaining constant energy, avoiding signal amplification/attenuation issues in standard weighted fusion. This \textit{\textbf{energy conservation property}} inherently provides noise suppression and acts as a natural regularizer: while allowing adaptive emphasis on different components, it prevents excessive amplification that could propagate noise. 
Meanwhile, the complementary weights ensure that information from both components is always preserved, creating robust fusion that is less sensitive to individual component errors. 
\textit{See Appendix F for detailed $\beta$ calculation and  theoretical proof. }

\subsection{DCA mechanism}
As shown in the middle of Figure \ref{fig:MDMLP-EIA}, the trend MLP, strong seasonal MLP, and weak seasonal MLP all employ a DCA mechanism that links the number of neurons \textit{n} in the MLP to the number of channels \textit{C} through a dynamic adjustment coefficient \textit{cof}:
\begin{equation}
\textit{cof} = \lfloor \sqrt{C}/\tau \rfloor,
\label{eq:Dynamic capacity adjustment mechanism}
\end{equation}
where $\lfloor \cdot \rfloor$ represents the ceiling operation, and $\tau$ is an adjustable coefficient (set to 5 in this paper). \textit{See Appendix G for detailed DCA mechanism.}

\subsection{Model loss and learning rate adjustment scheme}
We implement the arctangent loss function and sigmoid learning rate adjustment for time series prediction. Please refer details in xPatch \cite{stitsyuk2025xpatch}.

\begin{table*}[ht]
\scriptsize
\centering
\setlength\tabcolsep{1.5pt}
\begin{tabular}{c|cc|cc|cc|cc|cc|cc|cc|cc|cc|cc|cc|cc}
\hline
\multicolumn{1}{c}{Method} &
\multicolumn{2}{|c}{\shortstack{MDMLP-EIA \\ (\textbf{Ours})}} &
\multicolumn{2}{|c}{\shortstack{xPatch \\ (2025)}} &
\multicolumn{2}{|c}{\shortstack{Amplifier \\ (2025)}} &
\multicolumn{2}{|c}{\shortstack{TimeMixer \\ (2024)}} &
\multicolumn{2}{|c}{\shortstack{CycleNet \\ (2024)}} &
\multicolumn{2}{|c}{\shortstack{SOFTS \\ (2024)}} &
\multicolumn{2}{|c}{\shortstack{iTransformer \\ (2024)}} &
\multicolumn{2}{|c}{\shortstack{FilterNet \\ (2024)}} &
\multicolumn{2}{|c}{\shortstack{FITS \\ (2024)}} &
\multicolumn{2}{|c}{\shortstack{SparseTSF \\ (2024)}} &
\multicolumn{2}{|c}{\shortstack{FreTS \\ (2023)}} &
\multicolumn{2}{|c}{\shortstack{Dlinear \\ (2023)}} \\
\hline
\multicolumn{1}{c|}{Metric}
& MSE & MAE & MSE & MAE & MSE & MAE & MSE & MAE & MSE & MAE & MSE & MAE & MSE & MAE & MSE & MAE & MSE & MAE & MSE & MAE & MSE & MAE & MSE & MAE \\
\hline
ETTh1     &0.436 & \textbf{0.424} & 0.442 & 0.429 & 0.446 & \underline{0.426} & 0.449 & 0.428 & \textbf{0.432} & 0.427 & \underline{0.435} & 0.428 & 0.452 & 0.441 & 0.450 & 0.430 & 0.483 & 0.463 & 0.450 & \underline{0.426} & 0.471 & 0.455 & 0.447 & 0.438 \\
ETTh2     & \textbf{0.361} & \textbf{0.386} & 0.373 & 0.394 & 0.370 & \underline{0.391} & \underline{0.369} & \underline{0.391} & 0.383 & 0.404 & 0.371 & 0.395 & 0.380 & 0.401 & \underline{0.369} & \underline{0.391} & 0.385 & 0.405 & 0.385 & 0.398 & 0.462 & 0.462 & 0.426 & 0.429 \\
ETTm1     & \textbf{0.378} & \textbf{0.376} & \underline{0.385} & 0.402 & \underline{0.385} & \underline{0.380} & 0.386 & \underline{0.380} & 0.386 & 0.394 & 0.386 & 0.383 & 0.402 & 0.398 & 0.386 & 0.382 & 0.406 & 0.391 & 0.408 & 0.393 & 0.405 & 0.401 & 0.395 & 0.388 \\
ETTm2     & \textbf{0.271} & \textbf{0.314} & 0.279 & 0.319 & 0.275 & 0.317 & \underline{0.272} & \textbf{0.314} & \textbf{0.271} & \textbf{0.314} & 0.278 & 0.317 & 0.286 & 0.324 & 0.273 & \underline{0.315} & 0.287 & 0.331 & 0.281 & 0.320 & 0.290 & 0.341 & 0.287 & 0.322 \\
electricity & \textbf{0.167} & \textbf{0.257} & 0.185 & 0.267 & 0.173 & \underline{0.260} & 0.183 & 0.265 & \underline{0.169} & \underline{0.260} & 0.870 & 0.763 & 0.178 & 0.264 & 0.205 & 0.281 & 0.226 & 0.304 & 0.216 & 0.281 & 0.197 & 0.277 & 0.213 & 0.288 \\
Exchange  & \underline{0.361} & 0.403 & 0.371 & 0.407 & 0.382 & 0.414 & 0.363 & \underline{0.402} & 0.398 & 0.430 & \textbf{0.358} & \underline{0.402} & 0.377 & 0.412 & 0.410 & 0.437 & 0.451 & 0.470 & 0.433 & 0.467 & 0.455 & 0.424 & 0.362 & \textbf{0.394} \\
Solar     & 0.238 & \textbf{0.237} & 0.244 & 0.240 & 0.281 & 0.262 & \textbf{0.225} & 0.240 & \underline{0.235} & 0.269 & 0.264 & 0.250 & 0.236 & \underline{0.239} & 0.280 & 0.266 & 0.427 & 0.360 & 0.303 & 0.282 & 0.277 & 0.271 & 0.335 & 0.318 \\
Traffic   & 0.481 & 0.287 & 0.500 & 0.283 & 0.504 & 0.300 & 0.509 & 0.280 & 0.485 & 0.312 & \underline{0.428} & \underline{0.252} & \textbf{0.414} & \textbf{0.248} & 0.459 & 0.288 & 0.691 & 0.427 & 0.589 & 0.322 & 0.579 & 0.329 & 0.654 & 0.353 \\
Weather   & \textbf{0.239} & \textbf{0.262} & 0.247 & 0.266 & 0.246 & 0.267 & \underline{0.243} & \underline{0.263} & 0.254 & 0.279 & 0.250 & 0.267 & 0.256 & 0.272 & 0.244 & 0.265 & 0.258 & 0.277 & 0.261 & 0.278 & 0.247 & 0.276 & 0.271 & 0.292 \\
\hline
\multicolumn{1}{c|}{count of $1^{st}$} 
& 5 & 7 & 0 & 0 & 0& 0 & 1 & 1 & 2 & 1 & 1 & 0 & 1 & 1 & 0 & 0 & 0 & 0 & 0 & 0 & 0 & 0 & 0 & 1 \\
\hline
\end{tabular}
\caption{
Average long-term forecasting results with unified settings using a lookback window of $L = 96$. 
All results are averaged across four different prediction lengths: $ T = \{96,192,336,720\}$ for all datasets. 
The best-performing model is shown in \textbf{boldface}, and the second-best model is \underline{underlined}.
\textit{ See Table 1 in Appendix H for the full results.}}
\label{tab:experiments}
\end{table*}

\begin{table*}[ht]
  \scriptsize
  \centering
  \setlength\tabcolsep{1.5pt}
  \begin{tabular}{c|cc|cc|cc|cc|cc|cc|cc|cc|cc|cc|cc|cc}
    \hline
    \multicolumn{1}{c|}{\multirow{2}{*}{Dataset}} &
    \multicolumn{2}{c|}{\shortstack{MDMLP-EIA \\ \textbf{Ours}}}  & \multicolumn{2}{c|}{\shortstack{xPatch \\ (2025)}} & 
    \multicolumn{2}{c|}{\shortstack{Amplifier \\ (2025)}} & 
    \multicolumn{2}{c|}{\shortstack{TimeMixer \\ (2024)}} &
    \multicolumn{2}{c|}{\shortstack{CycleNet \\ (2024)}} & 
    \multicolumn{2}{c|}{\shortstack{SOFTS \\ (2024)}} & \multicolumn{2}{c|}{\shortstack{iTransformer \\ (2024)}} & 
    \multicolumn{2}{c|}{\shortstack{FilterNet \\ (2024)}}  & \multicolumn{2}{c|}{\shortstack{FITS \\ (2024)}} & 
    \multicolumn{2}{c|}{\shortstack{SparseTSF \\ (2024)}}& \multicolumn{2}{c|}{\shortstack{FreTS \\ (2023)}} & 
    \multicolumn{2}{c}{\shortstack{Dlinear \\ (2023)}} \\
    \cline{2-25}
    \multicolumn{1}{c|}{} 
    & MSE & MAE   & MSE & MAE & MSE & MAE & MSE & MAE & MSE & MAE & MSE & MAE & MSE & MAE & MSE & MAE & MSE & MAE & MSE & MAE & MSE & MAE & MSE & MAE\\
    \hline
    ETTh1       & \textbf{0.410} & \textbf{0.420} & 0.421 & 0.426 & 0.421 & \underline{0.425} & 0.423 & \underline{0.425} & 0.452 & 0.443 & \underline{0.418} & 0.430 & 0.447 & 0.453 & 0.428 & 0.430 & 0.884 & 0.642 & 0.438 & 0.434 & 0.480 & 0.471 & 0.425 & 0.433 \\
    ETTh2       & \textbf{0.335} & \textbf{0.378} & 0.348 & 0.387 & 0.344 & \underline{0.385} & 0.345 & 0.387 & 0.358 & 0.395 & 0.373 & 0.403 & 0.364 & 0.402 & 0.365 & 0.399 & \underline{0.339} & 0.386 & 0.369 & 0.401 & 0.912 & 0.649 & 0.488 & 0.470 \\
    ETTm1       & \underline{0.351} & \textbf{0.367} & 0.354 & 0.372 & 0.359 & \underline{0.371} & 0.358 & 0.374 & 0.356 & 0.373 & \textbf{0.349} & 0.374 & 0.366 & 0.389 & 0.370 & 0.381 & 0.358 & 0.372 & 0.359 & 0.375 & 0.373 & 0.384 & 0.355 & 0.372 \\
    ETTm2       & \textbf{0.253} & \textbf{0.305} & \underline{0.255} & 0.308 & 0.260 & 0.312 & \textbf{0.253} & \underline{0.306} & \underline{0.255} & 0.307 & 0.260 & 0.314 & 0.267 & 0.321 & 0.259 & 0.309 & 0.303 & 0.351 & 0.256 & 0.310 & 0.274 & 0.333 & 0.260 & 0.316 \\
    Electricity & \underline{0.157} & \underline{0.249} & 0.160 & 0.251 & 0.162 & 0.254 & 0.160 & 0.251 & \textbf{0.156} & \textbf{0.246} & 0.703 & 0.650 & 0.160 & 0.250 & 0.172 & 0.267 & 0.434 & 0.491 & 0.171 & 0.259 & 0.167 & 0.261 & 0.162 & 0.256 \\
    Exchange    & \underline{0.362} & 0.404 & 0.371 & 0.407 & 0.383 & 0.414 & 0.363 & \underline{0.403} & 0.398 & 0.431 & \textbf{0.358} & \underline{0.403} & 0.377 & 0.412 & 0.410 & 0.437 & 0.452 & 0.471 & 0.431 & 0.459 & 0.456 & 0.424 & \underline{0.362} & \textbf{0.395} \\
    Solar       & \textbf{0.190} & \textbf{0.213} & 0.201 & \underline{0.222} & 0.207 & \textbf{0.213} & 0.228 & 0.244 & 0.199 & 0.223 & 0.235 & 0.261 & \underline{0.197} & 0.225 & 0.217 & 0.234 & 0.263 & 0.272 & 0.223 & \textbf{0.213} & 0.226 & 0.228 & 0.239 & 0.234 \\
    Traffic     & \textbf{0.390} & 0.255 & \underline{0.392} & \textbf{0.249} & 0.406 & 0.264 & 0.397 & \underline{0.250} & 0.405 & 0.266 & 0.428 & 0.253 & 0.396 & 0.251 & 0.401 & 0.269 & 0.446 & 0.308 & 0.432 & 0.261 & 0.418 & 0.287 & 0.421 & 0.283 \\
    Weather     & \textbf{0.217} & \textbf{0.248} & 0.222 & \underline{0.251} & \underline{0.221} & \underline{0.251} & 0.227 & 0.255 & 0.223 & 0.252 & 0.226 & 0.256 & 0.233 & 0.265 & 0.231 & 0.261 & 0.238 & 0.273 & 0.227 & 0.259 & 0.229 & 0.263 & 0.240 & 0.276 \\
    \hline
    1st         & 6 & 6 & 0 & 1 & 0 & 1 & 1 & 0 & 1 & 1 &  2& 0& 0&0 &0 &0 &0 &0 &0 &1 &0 &0 &0 &1 \\
    \hline
  \end{tabular}
  \caption{
  Average long-term forecasting results with hyperparameter optimization. All results are averaged across four different prediction lengths: $ T = \{96,192,336,720\}$ for all datasets. The best-performing model is shown in \textbf{boldface}, and the second-best model is \underline{underlined}.
  \textit{See Table 2 in Appendix H for the full results.}}
  \label{tab:experiments2}
\end{table*}

\section{Experiments}
\subsection{Experimental setup}
\paragraph{Datasets} 
We conduct extensive experiments on nine real-world multivariate time series datasets: the ETT benchmark (specifically ETTh1, ETTh2, ETTm1, and ETTm2) \cite{zhou2021informer}, Weather \cite{autoformer21}, Traffic \cite{Sen2019}, Electricity \cite{autoformer21}, Exchange-Rate \cite{lai2018modeling}, and Solar-Energy \cite{lai2018modeling}. All datasets are preprocessed and normalized as described in \cite{nie2022time} and \cite{liu2023itransformer}, and split into training, validation, and test sets with a 7:2:1 ratio.

\paragraph{Baselines} 
Our experimental baselines include recent state-of-the-art TSF models: xPatch (2025) \cite{stitsyuk2025xpatch}, Amplifier (2025) \cite{fei2025amplifier}, TimeMixer (2024) \cite{wang2024timemixer}, CycleNet (2024) \cite{lin2024cyclenet}, SOFTS (2024) \cite{han2024softs}, iTransformer (2024) \cite{liu2023itransformer}, FilterNet (2024) \cite{yi2024filternet}, FITS (2024) \cite{xu2023fits}, SparseTSF (2024) \cite{lin2024sparsetsf}, DLinear (2023) \cite{zeng2023transformers}, and FreTS (2023) \cite{yi2023frequency}.

\paragraph{Implementation details}
All the experiments are implemented using PyTorch \cite{paszke2019pytorch}, and conducted on 4 A100 GPUs. We evaluate performance using MSE and Mean Absolute Error (MAE). For fair comparisons, several proven effective strategies \cite{stitsyuk2025xpatch}, including RevIN, arctangent loss, and sigmoid learning rate adjustment, are applied to all baseline models.

\subsection{Main results}

\begin{figure}[ht]
  \centering
  \begin{subfigure}[b]{0.42\textwidth}
    \centering
    \includegraphics[width=\linewidth]{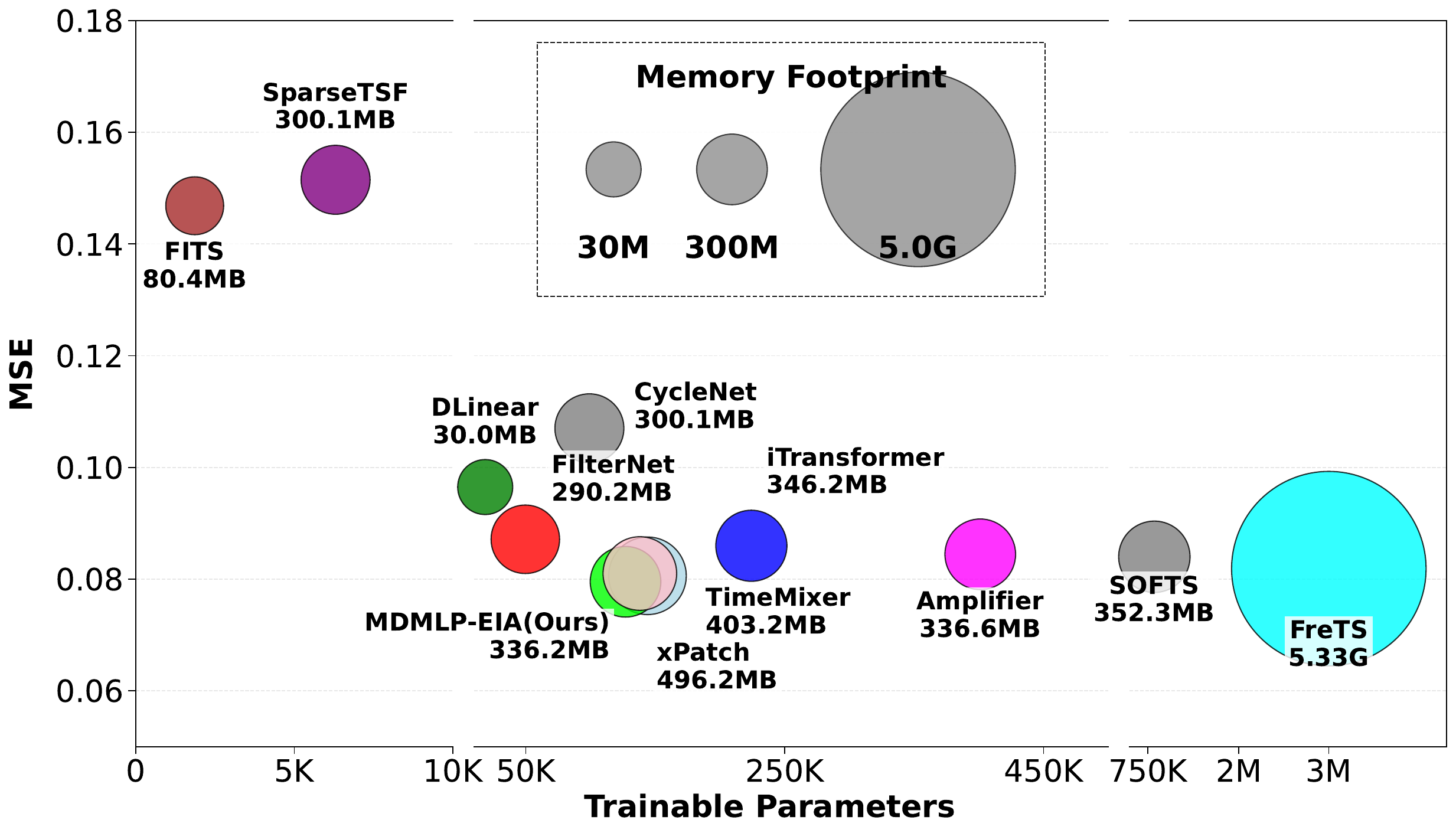}  
    \caption{Exchange (8 Variables, with training batch size is 256)}
    \label{fig:exchange_model_comparison} 
  \end{subfigure}
  \hfill  
  \begin{subfigure}[b]{0.42\textwidth}
    \centering
    
    \includegraphics[width=\linewidth]{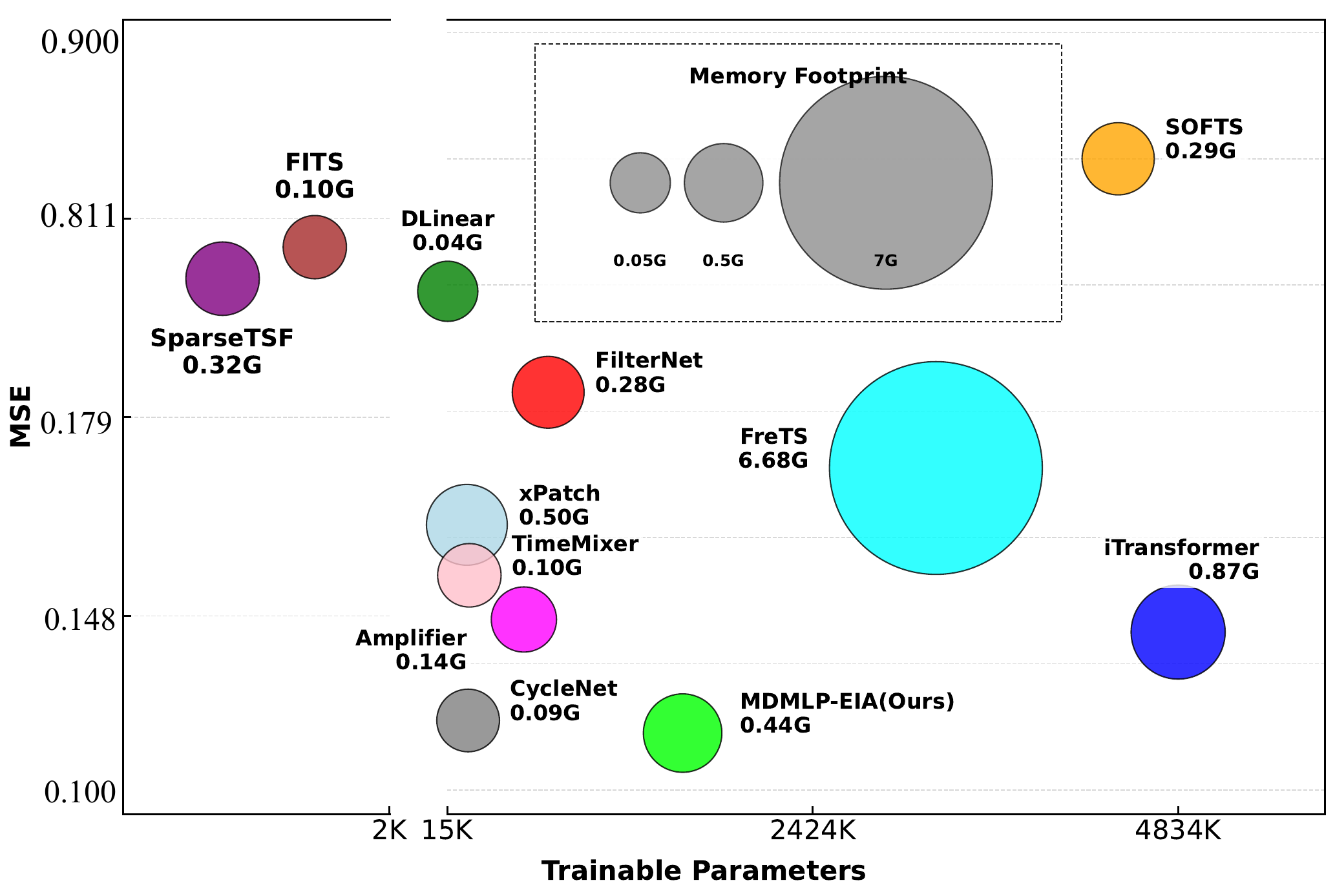} 
    \caption{Electricity (321 Variables, with training batch size 8)}
    \label{fig:electricity_model_comparison} 
  \end{subfigure}
 \caption{Model effectiveness and efficiency comparison on the Exchange and Electricity datasets.}
  \label{fig:effectiveness} 
\end{figure}


\paragraph{Unified experimental settings} 
We use unified settings based on the forecasting protocol proposed by TimesNet \cite{wu2023timesnet}: a lookback window of $L=96$ and prediction horizons of $T=\{ 96, 192, 336, 720 \}$ for all datasets.
Table \ref{tab:experiments} details the long-term forecasting results with the unified settings. MDMLP-EIA (Ours) achieves the best MSE on 5 out of 9 datasets and the best MAE on 7 out of 9 datasets. 
In particular, compared to the two most recent methods, Amplifier (2025) \cite{fei2025amplifier} and xPatch (2025) \cite{stitsyuk2025xpatch}, MDMLP-EIA (ours) surpasses them by approximately 4.14\% and 2.91\% in MSE, and 2.29\% and 1.37\% in MAE, respectively.
Furthermore, compared to CycleNet (2024) \cite{lin2024cyclenet}, the method with the most comparable performance, MDMLP-EIA (ours) shows an advantage of approximately 2.68\% in MSE and 4.66\% in MAE.
\textit{Full results under unified experimental settings are provided in Appendix H.}

\paragraph{Hyperparameter search} 
We assess how lookback length affects model performance to identify optimal values for each baseline. Models benefiting from longer sequences undergo additional hyperparameter optimization following xPatch \cite{stitsyuk2025xpatch}. We use official implementations with original configurations as baselines, then tune lookback length and other hyperparameters accordingly. The results are shown in Table \ref{tab:experiments2}. 
Our model achieves the most first-place results, with best MSE on 6/9 datasets and best MAE on 6/9 datasets.
Notably, MDMLP-EIA (ours) surpasses the two most recent methods, Amplifier (2025) \cite{fei2025amplifier} and xPatch (2025) \cite{stitsyuk2025xpatch}, by approximately 3.48\% and 2.20\% in MSE, and 1.73\% and 1.15\% in MAE, respectively.
Furthermore, compared to CycleNet (2024) \cite{lin2024cyclenet} and TimeMixer (2024) \cite{wang2024timemixer}, the methods with the most comparable performance, MDMLP-EIA (ours) shows advantages of approximately 4.82\% and 3.27\% in MSE, and 3.37\% and 2.17\% in MAE, respectively.
\textit{Full results with hyperparameter search are provided in Appendix H.}

\textit{More visualization results are provided in Appendix I.}

\subsection{Model analysis}
\paragraph{Validation of adaptive weak signal fusion}

\begin{table}[th]
\centering
\tiny
\setlength{\tabcolsep}{2pt}
\begin{tabularx}{\linewidth}{@{}lXX|XX|XX|XX|XX|XX|XX@{}} 
\toprule
\multirow{2}{*}{Dataset} & \multicolumn{2}{c|}{w/o WS} & \multicolumn{2}{c|}{MLP-F} & \multicolumn{2}{c|}{DWL-F} & \multicolumn{2}{c|}{CWA-F} & \multicolumn{2}{c|}{RCF} & \multicolumn{2}{c|}{CTF} & \multicolumn{2}{c}{AZCF} \\ 
\cmidrule(lr){2-3} \cmidrule(lr){4-5} \cmidrule(lr){6-7} \cmidrule(lr){8-9} \cmidrule(lr){10-11} \cmidrule(lr){12-13} \cmidrule(lr){14-15}
                         & MSE         & MAE         & MSE         & MAE         & MSE            & MAE         & MSE            & MAE         & MSE            & MAE         & MSE            & MAE            & MSE            & MAE           \\
\midrule
ETTh1                    & 0.442      & 0.428      & 0.440      & 0.428      & \textbf{0.437} & 0.429          & 0.444          & 0.427          & 0.441          & 0.429          & 0.438          & 0.427          & \textbf{0.437} & \textbf{0.424} \\
ETTh2                    & 0.366      & 0.389      & 0.368      & 0.391      & 0.369          & 0.392          & 0.370          & 0.392          & 0.370          & 0.393          & 0.374          & 0.394          & \textbf{0.362} & \textbf{0.387}  \\
ETTm1                    & \textbf{0.377} & 0.379      & 0.382      & 0.379      & 0.378          & 0.378 & 0.382          & 0.380          & 0.379          & 0.379          & 0.378          & \textbf{0.377}          & 0.378          & \textbf{0.377}          \\
ETTm2                    & \textbf{0.271} & \textbf{0.314} & \textbf{0.271} & \textbf{0.314} & 0.277          & 0.317          & 0.272          & \textbf{0.314} & 0.275          & 0.316          & 0.276          & 0.316          & \textbf{0.271} & \textbf{0.314} \\
Weather                  & 0.241      & 0.263      & 0.242      & 0.265      & \textbf{0.240} & \textbf{0.262} & \textbf{0.240} & 0.263          & 0.241 & 0.263          & 0.241          & \textbf{0.262}          & \textbf{0.240} & \textbf{0.262}           \\
\midrule
Average                  & 0.339      & 0.355      & 0.341      & 0.355      & 0.340 & 0.356 & 0.342 & 0.355 & 0.341 & 0.356 & 0.341 & 0.355 & \textbf{0.338} & \textbf{0.353} \\
\bottomrule
\end{tabularx}
\caption{Validation results for AZCF, averaged across prediction lengths.
Best results for each dataset are in \textbf{boldface}. The final row ('Average') shows the mean performance across all five datasets. \textit{Full per-prediction length results are available in Table 3, Appendix H.}}
\label{tab:ablation_x2_avg}
\end{table}

To validate the effectiveness of introducing the weak seasonal MLP in the MDMLP-EIA model and the proposed AZCF mechanism for fusing strong and weak seasonal predictions, we compare AZCF against the following six ablation cases: 
\begin{enumerate}
\item \textbf{w/o WS}: Without the weak seasonal MLP;
\item \textbf{MLP-F}: Uses Multi-Layer Perceptron(MLP) to concatenate strong and weak seasonal predictions for fusion;
\item \textbf{CWA-F}: Employs channel-wise attention to weight and fuse strong and weak seasonal predictions;
\item \textbf{DWL-F} \textbf{\textit{(w/o single-parameter fusion)}}: Differs from AZCF only in using two weight coefficients for fusion as shown in Equation \eqref{eq:attn strong and weak};
\item \textbf{RCF} \textbf{\textit{(w/o zero initialization)}}: Differs from AZCF only in using random initialization instead of zero initialization;
\item \textbf{CTF} \textbf{\textit{(w/o channel-dimension fusion)}}: Identical to AZCF in Equation \eqref{eq:alpha strong and weak} but without channel-dimension fusion.
\end{enumerate}

Comparing the results of w/o WS and AZCF demonstrates the effectiveness of introducing the weak seasonal MLP. Our AZCF strategy also outperforms MLP-F and CWA-F in most cases, demonstrating its effectiveness in fusing strong and weak seasonal predictions. Specifically, ablating the three key mechanisms in AZCF (DWL-F, RCF, CTF) shows that when single-parameter fusion, zero initialization, or channel-dimension fusion are respectively removed, the model's prediction performance degrades. This validates the effectiveness of each key mechanism in AZCF.
\textit{For full per-prediction length and visualization results, please refer to Appendix H  and Appendix I, respectively.}

\paragraph{Validation of EIA mechanism} 

\begin{table*}[th]
\centering
\scriptsize
\setlength{\tabcolsep}{2pt}
\begin{tabular}{@{}lcc|cc|cc|cc|cc|cc|cc|cc|cc|cc|cc|cc@{}} 
\toprule
\multirow{3}{*}{Dataset} & \multicolumn{8}{c}{MDMLP-EIA} & \multicolumn{8}{c}{Amplifier} & \multicolumn{8}{c}{xPatch} \\
\cmidrule(lr){2-9} \cmidrule(lr){10-17} \cmidrule(lr){18-25} 
& \multicolumn{2}{c}{ADD} & \multicolumn{2}{c}{MLP} & \multicolumn{2}{c}{AGM} &\multicolumn{2}{c|}{EIA} 
& \multicolumn{2}{c}{ADD} & \multicolumn{2}{c}{MLP} & \multicolumn{2}{c}{AGM} &\multicolumn{2}{c|}{EIA} 
& \multicolumn{2}{c}{ADD} & \multicolumn{2}{c}{MLP} & \multicolumn{2}{c}{AGM} &\multicolumn{2}{c}{EIA} \\
\cmidrule(lr){2-3} \cmidrule(lr){4-5} \cmidrule(lr){6-7} 
\cmidrule(lr){8-9} \cmidrule(lr){10-11} \cmidrule(lr){12-13} 
\cmidrule(lr){14-15} \cmidrule(lr){16-17} \cmidrule(lr){18-19} 
\cmidrule(lr){20-21} \cmidrule(lr){22-23} \cmidrule(lr){24-25} 
 & MSE & MAE & MSE & MAE & MSE & MAE & MSE & MAE & MSE & MAE & MSE & MAE & MSE & MAE & MSE & MAE & MSE & MAE & MSE & MAE & MSE & MAE & MSE & MAE \\
\midrule
ETTh1 & 0.438 & 0.425 & 0.492 & 0.464 & 0.439 & 0.430 & \textbf{0.437} & \textbf{0.424} & 0.447 & 0.426 & \textbf{0.443} & \textbf{0.425} &0.462 & 0.432 & 0.454 & 0.430 & {0.446} & {0.431} & 0.484 & 0.470 & \textbf{0.444} &  \textbf{0.428} & 0.448 & {0.431} \\
ETTh2 & 0.368 & 0.391 & 0.381 & 0.403 & 0.370 & 0.392 & \textbf{0.362} & \textbf{0.387} & \textbf{0.369} & \textbf{0.391} & 0.399 & 0.406 & 0.374 &  0.394 & 0.371 & \textbf{0.391} & 0.370 & \textbf{0.393} & 0.401 & 0.419 & 0.370 & \textbf{0.393} & \textbf{0.369} & \textbf{0.393} \\
ETTm1 & 0.380 & 0.378 & 0.392 & 0.394  & \textbf{0.377} & \textbf{0.376} & {0.378} & {0.377} & 0.385 & 0.380 & 0.388 & 0.381  & 0.386 & 0.379 & \textbf{0.383} & \textbf{0.378} & {0.385} & \textbf{0.383} & 0.417 & 0.412  & \textbf{0.384} & 0.384 & {0.385} & \textbf{0.383} \\
ETTm2 & 0.275 & 0.316 & 0.280 & 0.324  & 0.275 & 0.315 & \textbf{0.271} & \textbf{0.314} & 0.276 & 0.318 & 0.285 & 0.323  & \textbf{0.273} & \textbf{0.315} & \textbf{0.273} & {0.316} & \textbf{0.279} & \textbf{0.320} & 0.295 & 0.337 & 0.280 & \textbf{0.320} & \textbf{0.279} & \textbf{0.320} \\
Weather & 0.244 & 0.264 & 0.247 & 0.269  & \textbf{0.240} & \textbf{0.262} & \textbf{0.240} & \textbf{0.262} & 0.247 & 0.267 & 0.246 & 0.269 & \textbf{0.243} & \textbf{0.265} & \textbf{0.243} & \textbf{0.265} & 0.247 & 0.266 & 0.250 & 0.272  & 0.242 & 0.262 & \textbf{0.241} & \textbf{0.262} \\
Solar & 0.265 & 0.258 & 0.256 & 0.255  & 0.258 & 0.243 & \textbf{0.239} & \textbf{0.237} & 0.281 & 0.263 & \textbf{0.250} & 0.255  & 0.260 & \textbf{0.243} & 0.256 & {0.247} & 0.255 & 0.247 & 0.257 & 0.259 & 0.257 & 0.247 & \textbf{0.246} & \textbf{0.243} \\
\midrule
Average & 0.328 & 0.339 & 0.341 & 0.352 & 0.327 & 0.336 & \textbf{0.321} & \textbf{0.334} & 0.334 & 0.341 & 0.335 & 0.343 & 0.333 & \textbf{0.338} & \textbf{0.330} & \textbf{0.338} & 0.330 & 0.340 & 0.351 & 0.362 & 0.330 & \textbf{0.339} & \textbf{0.328} & \textbf{0.339} \\
\bottomrule
\end{tabular}
\caption{Validation results for EIA under MDMLP-EIA, Amplifier, and xPatch frameworks, averaged across prediction lengths. The best performance within each model group per dataset is highlighted in \textbf{boldface}. The final row ('Average') shows the mean performance across all six datasets.
\textit{Full per-prediction length results are available in Appendix Table 4.}
}
\label{tab:ablation_attn_avg} 
\end{table*}

In our proposed MDMLP-EIA, the EIA has been designed for effective integration of trend and seasonal signals. To validate the effectiveness of EIA fusion, we compare three methods:
\begin{enumerate}
\item \textbf{ADD (w/o EIA)}: Directly adds trend and seasonal signals without EIA;
\item \textbf{MLP}: Uses MLP-based attention to weight and fuse trend and seasonal signals;
\item \textbf{AGM}: Applies Adaptive Gating Mechanism for fusion (Equation \eqref{eq:standard_weighted_fusion}).
\end{enumerate}
This comparative analysis is conducted for our MDMLP-EIA model and two leading models (Amplifier \cite{fei2025amplifier} and xPatch \cite{stitsyuk2025xpatch}) to evaluate broader applicability. The averaged results across six datasets are presented in Table \ref{tab:ablation_attn_avg}.
For our MDMLP-EIA model, using EIA for fusion delivers substantial performance improvements. On average across six datasets, it reduces MSE by approximately 5.91\% and MAE by 5.12\% compared to MLP fusion, and by 2.1\% (MSE) and 1.53\% (MAE) compared to ADD. While AGM shows competitive results on some datasets, our EIA achieves best or joint-best performance on 5 out of 6 datasets. Compared to AGM, EIA reduces average MSE and MAE by approximately 1.83\% and 0.60\%, respectively. Notably, on the Solar dataset, EIA demonstrates substantial advantage over AGM, reducing MSE by approximately 7.36\% and MAE by approximately 2.51\%. When incorporated into xPatch \cite{stitsyuk2025xpatch} and Amplifier \cite{fei2025amplifier}, EIA shows similar positive trends, validating the effectiveness and transferability of our EIA mechanism.
\textit{For full per-prediction length and visualization results, please refer to Appendices H and I.}

\paragraph{Validation of DCA mechanism}

\begin{table}
\centering
\begin{adjustbox}{max width=\linewidth}
\begin{tabular}{@{}l|cc|cc|cc|cc|cc@{}} 
\toprule
\multirow{2}{*}{HS} & \multicolumn{2}{c|}{ETTh1} & \multicolumn{2}{c|}{Solar} & \multicolumn{2}{c|}{Traffic} & \multicolumn{2}{c|}{Weather} & \multicolumn{2}{c}{Average} \\
\cmidrule(lr){2-3} \cmidrule(lr){4-5} \cmidrule(lr){6-7} \cmidrule(lr){8-9} \cmidrule(lr){10-11}
 & MSE & MAE & MSE & MAE & MSE & MAE & MSE & MAE & MSE & MAE \\
\midrule
32   & 0.448 & 0.438 & 0.261 & 0.262 & 0.610 & 0.331 & 0.240 & 0.263 & 0.389 & 0.323 \\
64   & 0.439 & 0.430 & 0.259 & 0.258 & 0.570 & 0.314 & 0.240 & 0.263 & 0.376 & 0.316 \\
128  & 0.438 & 0.427 & 0.257 & 0.254 & 0.550 & 0.302 & \textbf{0.240} & \textbf{0.260} & 0.372 & 0.311 \\
256  & 0.439 & 0.425 & 0.249 & 0.250 & 0.500 & 0.294 & \textbf{0.240} & 0.263 & 0.356 & 0.308 \\
512  & 0.440 & 0.425 & 0.249 & 0.249 & 0.490 & 0.292 & 0.240 & 0.263 & 0.354 & 0.307 \\
1024 & 0.441 & 0.425 & 0.248 & 0.240 & 0.490 & 0.289 & 0.240 & 0.264 & 0.354 & 0.304 \\
2048 & 0.443 & 0.424 & 0.254 & 0.248 & 0.490 & \textbf{0.282} & 0.240 & 0.264 & 0.357 & 0.305 \\ \hline
DCA  & \textbf{0.436} & \textbf{0.424} & \textbf{0.238} & \textbf{0.237} & \textbf{0.481} & 0.287 & \textbf{0.239} & 0.262 & \textbf{0.349} & \textbf{0.302} \\
\bottomrule
\end{tabular}
\end{adjustbox}
\caption{Validation results of DCA mechanism. MSE and MAE are averaged across four prediction lengths ($T \in \{96, 192, 336, 720\}$) on datasets with different channel counts: ETTH1 ($C=7$), Weather ($C=21$), Solar ($C=137$), and Traffic ($C=862$). Best results are in bold.}
\label{tab:ablation_hidden_size_overall_avg}
\end{table}

We compared the DCA mechanism with fixed capacity (neuron count) strategies, where the number of neurons varies across {32, 64, 128, 256, 512, 1024, 2048}. We selected four datasets with significantly different channel counts (ETTH1 ($C=7$), Weather ($C=21$), Solar ($C=137$), and Traffic ($C=862$)) for time series prediction tasks with $L=96$ and prediction lengths of {96, 192, 336, 720}. 
Table \ref{tab:ablation_hidden_size_overall_avg} presents the average prediction results across all prediction lengths. The DCA strategy achieves the lowest MSE and MAE across all four datasets, with average reductions of 1.41\% in MSE and 0.07\% in MAE compared to the second-best case (HS=1024).
\textit{For more detailed results across different datasets and prediction lengths, please refer to Appendix H.}

\paragraph{Efficiency analysis} 
We comprehensively evaluate MDMLP-EIA's efficiency in terms of trainable parameters, memory footprint, and predictive accuracy (MSE), following \cite{yi2024filternet}. Figure~\ref{fig:effectiveness} compares our method against representative baselines on datasets with markedly different scales: the Exchange dataset (8 variables, the batch size of training stage is 256) and the Electricity dataset (321 variables, the batch size of training stage is 8) with $L=96$ and $Q=96$.
MDMLP-EIA demonstrates superior efficiency-accuracy trade-offs across both datasets. On Exchange, it achieves competitive accuracy (MSE $\approx$ 0.082) with approximately 128K parameters and 336.2MB memory, significantly outperforming parameter-heavy models like iTransformer (224k parameters, 346.2MB memory, MSE $\approx$ 0.086) while maintaining better accuracy than lightweight models like DLinear (18K parameters, 30.0MB memory, MSE $\approx$ 0.096). On Electricity, MDMLP-EIA maintains strong performance (MSE $\approx$ 0.140) with 1907K parameters and 0.44G memory, substantially more efficient than FreTS (3236K parameters, 6.68G memory).  The results establish MDMLP-EIA's effective balance between model complexity and predictive accuracy across different dataset scales.

\section{Conclusions}
\label{Conclusions and Discussions}
We propose MDMLP-EIA, a novel time series forecasting model that addresses critical limitations in existing approaches through three key innovations: a weak seasonal MLP with AZCF fusion mechanism, an EIA mechanism, and a DCA mechanism. These components effectively tackle three prevalent issues: loss of weak seasonal signals, capacity constraints in weight-sharing MLPs, and insufficient channel fusion in channel-independent strategies, enabling efficient and accurate forecasting.
Moreover, the proposed AZCF, EIA, and DCA mechanisms demonstrate broad applicability and can enhance other seasonal-trend decomposition-based methods. Future work includes designing unified cross-horizon architectures and enabling domain transfer.

\section{Acknowledgments}
We would like to thank all anonymous reviewers for their valuable comments.

This work was supported in part by the National Natural Science Foundation of China (NSFC) under Grant 62376040 and 62506047, the Key projects of NSFC under Grant 62233018, the outstanding youth project of Hunan Provincial Department of Education under Grant 24B0797, and the China Postdoctoral Science Foundation under Grant Number 2024T171058 and 2024M753672.



\bibliography{aaai2026}

\appendix
\clearpage  
\section{A. Overall pseudocode}
\label{appendix: Overall pseudocode}

Algorithm~\ref{alg:mdmlp_eia} outlines the workflow of our proposed MDMLP-EIA.
Initially, we normalize the input time series $x$ using RevIN and decompose the normalized series into its trend component $x_1$ and seasonal component $x_2$ via EMA.
We then forecast the trend component $y_1$ using a dedicated TrendMLP.
For the seasonal component, we employ an adaptive fused dual-domain approach: a frequency-domain based MLP processes the strong seasonal signals to produce $y_{21}$, while a time-domain based MLP handles weak seasonal signals to yield $y_{22}$.
These seasonal predictions are subsequently combined into $y_2$ using an adaptive zero-initialized channel fusion strategy.
Following this, we integrate the trend prediction $y_1$ and the combined seasonal prediction $y_2$ using our energy invariant attention mechanism to obtain a fused prediction $y_3$.
Finally, we apply inverse RevIN to $y_3$ to derive the ultimate forecast $y$.
Throughout this process, the capacities of our MLPs ($n_1, n_2, n_3$) are dynamically adjusted based on the number of input channels.

\begin{algorithm}[htbp] 
\small
\caption{MDMLP-EIA Algorithm}
\label{alg:mdmlp_eia}
\begin{algorithmic}[1] 
    \Require Input time series $x \in \mathbb{R}^{L \times C}$, input length $L$, prediction length $Q$, number of channels $C$.
    \Ensure Predicted time series $y \in \mathbb{R}^{Q \times C}$.

    \Statex \textit{// Initialize learnable parameters}
    \State Initialize learnable channel fusion weights $\alpha_{ch} \in \mathbb{R}^{1 \times C}$ to zeros.

    \Statex \textit{// Hyperparameters}
    \State Initialize hyperparameters $\tau, E, n_h$ \Comment{Values specified in experimental setup}
    
    \Statex \textit{// Data Preprocessing \& Decomposition}
    \State $x_{norm}, \text{stats} \leftarrow \RevIN(x)$ \Comment{Instance Normalization}
    \State $x_{trend\_orig}, x_{seasonal\_orig} \leftarrow \EMADecomposition(x_{norm})$ \Comment{$x_{trend\_orig}$: trend, $x_{seasonal\_orig}$: seasonal components from $x_{norm}$}

    \Statex \textit{// Prepare inputs for channel-independent MLPs}
    \State $x_1 \leftarrow \DimPermute(x_{trend\_orig})$ \Comment{Trend component for MLP, permuted to $\mathbb{R}^{C \times L}$}
    \State $x_2 \leftarrow \DimPermute(x_{seasonal\_orig})$ \Comment{Seasonal component for MLP, permuted to $\mathbb{R}^{C \times L}$}

    \Statex \textit{// Dynamic Capacity Adjustment for MLPs}
    \State $\text{cof} \leftarrow \lceil \sqrt{C}/\tau \rceil$ \Comment{Capacity adjustment coefficient}
    \State $n_1 \leftarrow L \times \text{cof}$ \Comment{Neuron configuration for TrendMLP}
    \State $n_2 \leftarrow n_h \times \text{cof}$ \Comment{Neuron configuration for StrongSeasonalMLP}
    \State $n_3 \leftarrow 2L \times \text{cof}$ \Comment{Neuron configuration for WeakSeasonalMLP}

    \Statex \textit{// 1. Trend Component Forecasting}
    \State $y_1 \leftarrow \TrendMLP(x_1, n_1)$ \Comment{Trend prediction $y_1 \in \mathbb{R}^{Q \times C}$}

    \Statex \textit{// 2. Seasonal Component Forecasting (Adaptive Fused Dual-Domain)}
    \Statex \textit{// \quad 2.a Strong Seasonal Signal Path (Frequency Domain based)}
    \State $x_{21} \leftarrow \EmbedOp(x_2)$ \Comment{Embedding of seasonal component $x_2$}
    \State $x_{22} \leftarrow \rfftOp(x_{21})$ \Comment{FFT of embedded seasonal component}
    \State $x_{23} \leftarrow \FreMLP(x_{22})$ \Comment{Frequency MLP processing; see FreTS}
    \State $f_{s1} \leftarrow \irfftOp(x_{23})$ \Comment{Reconstructed strong seasonal features $f_{s1}$}
    \State $y_{21} \leftarrow \StrongSeasonalMLP(f_{s1}, n_2)$ \Comment{Strong seasonal prediction $y_{21} \in \mathbb{R}^{Q \times C}$}
    
    \Statex \textit{// \quad 2.b Weak Seasonal Signal Path (Time Domain based)}
    \State $y_{22} \leftarrow \WeakSeasonalMLP(x_2, n_3)$ \Comment{Weak seasonal prediction $y_{22} \in \mathbb{R}^{Q \times C}$}

    \Statex \textit{// \quad 2.c Adaptive Fusion of Seasonal Predictions}
    \State $\alpha_{broadcast} \leftarrow \BroadcastOp(\alpha_{ch}, \text{dim}=0, \text{length}=Q)$ \Comment{Expand learnable weights $\alpha_{ch}$}
    \State $y_2 \leftarrow y_{21} + \alpha_{broadcast} \odot y_{22}$ \Comment{Combined seasonal prediction $y_2 \in \mathbb{R}^{Q \times C}$}

    \Statex \textit{// 3. Energy Invariant Attention for Trend-Seasonal Fusion}
    \State $y_{C} \leftarrow \Concat(y_1, y_2, \text{dim}=1)$ \Comment{Concatenate $y_1, y_2$ for attention input}
    \State $y_{C1} \leftarrow \text{Linear}^1_{A}(y_C)$
    \State $y_{C2} \leftarrow \text{Linear}^2_{A}(\DropoutOp(\GeLUOp(y_{C1})))$
    \State $\beta \leftarrow \SigmoidOp(y_{C2})$ \Comment{Attention weights $\beta \in \mathbb{R}^{Q \times C}$ for $y_1$}
    \State $y_3 \leftarrow 2 \times (\beta \odot y_1 + (1-\beta) \odot y_2)$ \Comment{Fused prediction $y_3 \in \mathbb{R}^{Q \times C}$}

    \Statex \textit{// 4. Final Output Generation}
    \State $y \leftarrow \InverseRevIN(y_3, \text{stats})$ \Comment{Inverse Normalization to get final prediction $y$}
    \State \Return $y$
\end{algorithmic}
\end{algorithm}

\paragraph{Note:}
\begin{itemize}
    \item We design $\TrendMLP$, $\StrongSeasonalMLP$, and $\WeakSeasonalMLP$ as multi-layer perceptrons that employ channel-independent strategies. Their capacities ($n_1, n_2, n_3$) are dynamically adjusted. The detailed architectures of these MLPs are provided in Appendices A, B, and C.
    \item $\FreMLP$ follows the design in FreTS.
    \item $\DimPermute(X)$ is used to adjust tensor dimensions as described in the paper, e.g., from $L \times C$ to $C \times L$ for channel-independent processing, and from $C \times Q$ back to $Q \times C$ for outputs of MLPs.
    \item $\alpha_{ch}$ are learnable parameters, initialized to zero and updated during training, used to fuse strong and weak seasonal predictions. The resulting $\alpha_{broadcast}$ is the expanded version of these weights.
    \item The calculation of $\beta$ (attention weights) involves learnable linear layers $\text{Linear}^1_A, \text{Linear}^2_A$, as detailed in Appendix D, to balance trend ($y_1$) and combined seasonal ($y_2$) components.
\end{itemize}

\section{B. Details of the trend MLP}
\label{appendix: trend MLP}
Let \textit{$n_{1}$} represent the number of neurons in the trend MLP. As illustrated in the bottom section of Fig. 2 of the paper, the trend MLP comprises three linear layers:
$\text{Linear}^1_{T} : [C, L] \to [C, 4n_1]$, 
$\text{Linear}^2_{T} : [C, 4n_1] \to [C, 2n_1]$, 
and $\text{Linear}^3_{T} : [C, 2n_1] \to [C, Q]$,
interconnected by two hyperbolic tangent (Tanh) activation functions and two dropout layers $\text{Dropout}(\cdot)$ for regularization. 
The trend MLP extracts the trend prediction $y_1 \in \mathbb{R}^{Q \times C}$ from the trend component $x_1 \in \mathbb{R}^{C \times L}$ through the following transformation:
\begin{equation}
\begin{cases}
x_{11} = \text{Linear}_T^1(x_1) \\
x_{12} = \text{Linear}_T^2(\text{Dropout}(\text{Tanh}(x_{11}))) \\
x'_1 = \text{Linear}_T^3(\text{Dropout}(\text{Tanh}(x_{12}))) \\
y_1 = \text{DimPermute}(x'_1)
\end{cases}
\end{equation}
where \textit{$x_{11}\in \mathbb{R}^{C \times 4n_1}$}, \textit{$x_{12}\in \mathbb{R}^{C \times 2n_1}$}, and \textit{$x'_{1}\in \mathbb{R}^{C \times Q}$} represent intermediate feature representations, and $\text{DimPermute}(\cdot)$ is a dimension permutation operation that transfers $\mathbb{R}^{C \times Q}$ to $\mathbb{R}^{Q \times C}$. The trend MLP implements a channel independence strategy, which essentially fuses information across time steps using shared weights along each channel dimension. This approach allows the trend MLP to process temporal patterns independently for each feature channel.

\section{C. Details of Frequency MLP and strong seasonal MLP}
\label{appendix: strong seasonal MLP}

According to the frequency-temporal learner in FreTS, the seasonal component $x_2$ undergoes a series of sequential transformations: embedding($\text{Embed}(\cdot)$), real-input fast Fourier transform ($\text{rfft}(\cdot)$), frequency MLP ($\text{FreMLP}(\cdot)$) processing, and inverse real fast Fourier transform ($\text{irfft}(\cdot)$). 
This sequence of operations generates the reconstructed feature with strong seasonal patterns\textit{$f_{s1}$}, formulated as:
\begin{equation}
\begin{cases}
x_{21} = \text{Embed}(x_2) \\
x_{22} = \text{rfft}(x_{21}) \\
x_{23} = \text{FreMLP}(x_{22}) \\
f_{s1} = \text{irfft}(x_{23})
\end{cases}
\label{eq:freMlp}
\end{equation}
where \textit{$x_{21}$}, \textit{$x_{22}$}, and \textit{$x_{23}$} represent intermediate feature representations in the transformation process. 
For details about the $\text{FreMLP}$, please refer to FreTS.
Then, we employ a strong seasonal MLP to learn the strong seasonal prediction $y_{21} \in \mathbb{R}^{Q \times C}$ from the reconstructed feature $f_{s1} \in \mathbb{R}^{C \times L\times E}$:
\begin{equation}
y_{21} = \text{MLP}_S(f_{s1})
\label{eq:y21}
\end{equation}
where $\text{MLP}_S$
represents the strong seasonal MLP. Similar to the trend MLP, the strong seasonal MLP also implements a channel independence strategy to process each feature dimension separately. 

Let \textit{$n_{2}$} represent the number of neurons in the strong seasonal MLP. As illustrated in the bottom section of Fig. 2 in the paper, the strong seasonal MLP includes one reshape transformation 
(Reshape: $f_{s1} \in \mathbb{R}^{C \times L\times E} \to f_{s2} \in \mathbb{R}^{C \times (L\times E)} $,
where \textit{E} is the embedding dimension in FreMLP, set to 8), two linear layers (
$\text{Linear}^1_{S} : [C, L\times E] \to [C, n_2]$
and 
$\text{Linear}^2_{S} : [C, n_2] \to [C,Q]$
), 
one LeakyReLU activation function, and a dropout layer $\text{Dropout}(\cdot)$ for regularization.
This network extracts the strong seasonal prediction $y_{21} \in \mathbb{R}^{Q \times C}$ from the reconstructed feature $f_{s1} \in \mathbb{R}^{C \times L\times E}$ through the following transformation:
\begin{equation}
\begin{cases}
f_{s2} = \text{Reshape}(f_{s1}) \\
f_{s3} = \text{Linear}_S^1(f_{s2}) \\
f_s = \text{Linear}_S^2(\text{Dropout}(\text{LeakyReLU}(f_{s3}))) \\
y_{21} = \text{DimPermute}(f_s)
\end{cases}
\end{equation}
where \textit{$f_{s2}\in \mathbb{R}^{C \times (L\times E)}$}, \textit{$f_{s3}\in \mathbb{R}^{C \times n_2}$}, and \textit{$f_s\in \mathbb{R}^{C \times Q}$} represent intermediate features, and $\text{DimPermute}(\cdot)$ refers to the dimension permutation operation that transfers $\mathbb{R}^{C \times Q}$ to $\mathbb{R}^{Q \times C}$. The strong seasonal MLP implements a channel independence strategy, which allows the network to process temporal patterns independently for each feature channel.

\section{D. Details of the weak seasonal MLP}
\label{appendix: weak seasonal MLP}
Let \textit{$n_{3}$} represent the number of neurons in the weak seasonal MLP. As illustrated in the bottom section of Fig. 2 in the paper, the weak seasonal MLP consists of two linear layers (
$\text{Linear}^1_{w} : [C, L] \to [C,n_3]$
and 
$\text{Linear}^2_{w} : [C, n_3] \to [C,Q]$
), one Tanh activation function, and a dropout layer $\text{Dropout}(\cdot)$ to prevent overfitting.
This network extracts the weak seasonal prediction $y_{22} \in \mathbb{R}^{Q \times C}$ from the weak seasonal component $x_{2} \in \mathbb{R}^{C \times L}$:
\begin{equation}
\begin{cases}
x_3 = \text{Linear}_w^1(x_2) \\
f_w = \text{Linear}_w^2(\text{Dropout}(\text{Tanh}(x_3))) \\
y_{22} = \text{DimPermute}(f_w)
\end{cases}
\end{equation}
where \textit{$x_3\in \mathbb{R}^{C \times n_3}$} and \textit{$f_w\in \mathbb{R}^{C \times Q}$} function as intermediate features, and $\text{DimPermute}(\cdot)$ represents the dimension permutation operation that transfers $\mathbb{R}^{C \times Q}$ to $\mathbb{R}^{Q \times C}$. The weak seasonal MLP implements a channel independence strategy, which allows the network to process temporal patterns independently for each feature channel.

\section{E. Proof of adaptive zero-initialized channel fusion}
\label{app:weak_seasonal_fusion}

In this section, we demonstrate that enhancing the seasonal prediction with a low-amplitude ("weak") seasonal branch and combining it through adaptive zero-initialized channel fusion reduces the mean-square error, compared to utilizing only the "strong" seasonal branch.

\medskip\noindent\textbf{Setup.} After EMA-decomposition we
obtain two seasonality estimates for each channel $c$ and time
$t$:
\begin{equation}
y_{21}(t,c)\approx s_{1}(t,c),
\label{eq:y21_approx}
\end{equation}
\begin{equation}
y_{22}(t,c)\approx \tilde s_{2}(t,c) = s_{2}(t,c)\;+\;n(t,c),
\label{eq:y22_approx}
\end{equation}
where
$s_{1}$ is the high‐energy (strong) seasonal component,
$s_{2}$ is the low‐energy (weak) seasonal component, and
$n\sim\mathcal N(0,\sigma_n^2)$ is additive noise, independent
of $s_{1},s_{2}$. We fuse them via
\begin{equation}
y_{2}(t,c) = y_{21}(t,c) + \alpha_{c}\,y_{22}(t,c),
\label{eq:fusion_rule_weak}
\end{equation}
with $\alpha_{c}$ learned (initialized at zero). Define the
per‐channel mean‐square error
\begin{equation}
\mathrm{MSE}(\alpha_c) = \mathbb{E}\bigl[\bigl(y_{2}-s_{1}-s_{2}\bigr)^{2}\bigr].
\label{eq:mse_alpha_c_weak_def}
\end{equation}

\medskip\noindent\textbf{Notation.} To simplify the mathematical expressions in the following analysis, we occasionally omit the explicit dependencies on time $t$ and channel $c$ in our notation. For instance, we may write $\operatorname{Var}[s_2]$ instead of $\operatorname{Var}[s_2(t,c)]$. All expressions should be understood to maintain these dependencies unless explicitly stated otherwise. The variance operators are applied with respect to the appropriate probability distributions for each random variable, with expectations taken over the relevant time and channel dimensions.

\begin{proposition}[Strict error reduction and optimal initialization]
\label{prop:error_reduction}
Let
\begin{equation}
\alpha_c^* =\frac{\operatorname{Var}\bigl[s_{2}(t,c)\bigr]}{\operatorname{Var}\bigl[s_{2}(t,c)\bigr]+\sigma_n^2}.
\label{eq:alpha_c_star_weak}
\end{equation}
Then:
\begin{itemize}
\item[(i)] MSE reduction: $\mathrm{MSE}(0)-\mathrm{MSE}(\alpha_c^*) =\frac{\operatorname{Var}[s_{2}]^{2}}{\operatorname{Var}[s_{2}]+\sigma_n^2} >0 \quad\iff\quad \operatorname{Var}[s_{2}] >0$.
\label{eq:mse_difference_result}

\item[(ii)] Zero initialization ($\alpha_c = 0$) provides a theoretically sound minimum-risk starting point when signal and noise variances are unknown.
\end{itemize}
In particular, whenever a non‐zero weak component exists,
optimally fusing it strictly lowers the MSE compared to
ignoring it ($\alpha_c=0$), while zero initialization ensures stability during early training stages.
\end{proposition}

\paragraph{Proof}
Let $e_{1}=y_{21}-s_{1}$ denote the strong‐branch prediction error.
Then the fusion error term is
\begin{equation}
y_{2}-s_{1}-s_{2} = e_{1} \;+\; \alpha_{c}(s_{2}+n) \;-\; s_{2} = e_{1} + (\alpha_c-1)\,s_{2} + \alpha_c\,n.
\label{eq:fusion_error_expanded}
\end{equation}
Independence of $e_{1},s_{2},n$ gives
\begin{equation}
\mathrm{MSE}(\alpha_c) =\mathbb{E}[e_{1}^{2}] +(\alpha_c-1)^{2}\operatorname{Var}[s_{2}] +\alpha_c^{2}\,\sigma_n^2.
\label{eq:mse_alpha_c_expanded_weak}
\end{equation}

Part (i): To find the optimal $\alpha_c$, we differentiate $\mathrm{MSE}(\alpha_c)$ with respect to $\alpha_c$ and set it to zero:

\begin{align}
\frac{d}{d\alpha_c}\mathrm{MSE}(\alpha_c) = 2(\alpha_c-1)\operatorname{Var}[s_{2}] + 2\alpha_c\sigma_n^2 = 0\\
(\alpha_c-1)\operatorname{Var}[s_{2}] + \alpha_c\sigma_n^2 = 0\\
\alpha_c\operatorname{Var}[s_{2}] - \operatorname{Var}[s_{2}] + \alpha_c\sigma_n^2 = 0\\
\alpha_c(\operatorname{Var}[s_{2}] + \sigma_n^2) = \operatorname{Var}[s_{2}]\\
\alpha_c^* = \frac{\operatorname{Var}[s_{2}]}{\operatorname{Var}[s_{2}] + \sigma_n^2}
\end{align}

At $\alpha_c=0$,
\begin{equation}
\mathrm{MSE}(0) =\mathbb{E}[e_{1}^{2}]+\operatorname{Var}[s_{2}],
\label{eq:mse_at_zero_alpha_weak}
\end{equation}

While at the optimizer $\alpha_c^*$, one shows by direct substitution
\begin{equation}
\mathrm{MSE}(\alpha_c^*) =\mathbb{E}[e_{1}^{2}] +\frac{\operatorname{Var}[s_{2}]\,\sigma_n^2} {\operatorname{Var}[s_{2}]+\sigma_n^2}.
\label{eq:mse_at_optimal_alpha_weak}
\end{equation}

Their difference is
\begin{equation} \label{eq:mse_difference_derivation_weak}
\begin{split}
\mathrm{MSE}(0)-\mathrm{MSE}(\alpha_c^*)
&= \operatorname{Var}[s_{2}] -\frac{\operatorname{Var}[s_{2}]\,\sigma_n^2} {\operatorname{Var}[s_{2}]+\sigma_n^2} \\
&= \frac{\operatorname{Var}[s_{2}]\,\bigl(\operatorname{Var}[s_{2}]+\sigma_n^2\bigr) -\operatorname{Var}[s_{2}] \, \sigma_n^2} {\operatorname{Var}[s_{2}]+\sigma_n^2} \\
&=\frac{\operatorname{Var}[s_{2}]^{2}}{\operatorname{Var}[s_{2}]+\sigma_n^2}.
\end{split}
\end{equation}
This is strictly positive whenever $\operatorname{Var}[s_{2}]>0$, proving part (i).

Part (ii): When signal variance $\operatorname{Var}[s_2]$ and noise variance $\sigma_n^2$ are unknown during initial training, we must consider the risk under uncertainty. Examining equation~\eqref{eq:mse_alpha_c_expanded_weak}, we observe that at $\alpha_c = 0$, the error is bounded as $\mathrm{MSE}(0) = \mathbb{E}[e_{1}^{2}]+\operatorname{Var}[s_{2}]$. However, for any $\alpha_c > 0$, the term $\alpha_c^{2}\,\sigma_n^2$ indicates that arbitrarily large noise variance $\sigma_n^2$ could result in unbounded error. Therefore, from a minimum-risk perspective, $\alpha_c = 0$ represents an optimal initialization that ensures stability during early training stages, even though the ultimate goal is to adapt toward $\alpha_c^*$ as training progresses and better estimates of signal and noise variances become available.

\medskip
The above rigorously demonstrates that (i) the weak seasonal component $s_{2}$ is \emph{necessary} for achieving lower MSE whenever it is non‐negligible, and (ii)
the proposed adaptive zero-initialized channel fusion 
\( \;y_{2}=y_{21}+\alpha_c\,y_{22}\; \)
with $\alpha_c^*$ is \emph{optimal}. This justifies our architectural choice to extract
and adaptively fuse the weak seasonal branch.

\section{F. Details about the energy invariant attention}
\label{appendix: energy invariant attention}

The trend prediction \textit{$y_{1}$} and seasonal prediction \textit{$y_{2}$} represent output components derived from trend and seasonal signals following EMA decomposition of the original time series. Conventional decomposition-based methods typically generate the final prediction \textit{$y_{3}$} by directly summing these components, potentially limiting the model's representational capacity across temporal points and feature channels. The proposed energy invariant attention in equation \eqref{eq:eia_fusion_def_enhanced} enables the model to adaptively focus on various features in trend and seasonal predictions at different time steps while maintaining overall predictive capability.

\subsection{Calculation of {$\beta$}} 
\label{appendix: beta calculation}

Considering that the model previously merged \textit{$y_{1}$} and \textit{$y_{2}$} across different time steps but lacks explicit channel fusion, we employ channel attention to obtain weights. 
First, we concatenate \textit{$y_{1}\in \mathbb{R}^{Q \times C}$} and \textit{$y_{2}\in \mathbb{R}^{Q \times C}$}  along the channel dimension to produce the concatenated output $ \textit{$y_{C}$} \in \mathbb{R}^{Q \times 2C}$:
\begin{equation}
y_C = \text{concat}(y_1, y_2)
\end{equation}
where $\text{concat}(\cdot)$ represents the tensor concatenation along the channel dimension. 
Next, we derive the channel dimension weight matrix $\beta \in \mathbb{R}^{Q \times C}$ using two linear layers ($\text{Linear}^1_{A} : [Q, 2C] \to [Q,4C]$, $\text{Linear}^2_{A} : [Q, 4C] \to [Q,C]$), a GeLU activation with Dropout, and a Sigmoid function:
\begin{equation}
\begin{cases}
y_{C1} = \text{Linear}_A^1(y_C) \\
y_{C2} = \text{Linear}_A^2(\text{Dropout}(\text{GeLU}(y_{C1}))) \\
\beta = \text{Sigmoid}(y_{C2})
\end{cases}
\end{equation}
where \textit{$y_{C1}\in \mathbb{R}^{Q \times 4C}$} and \textit{$y_{C2}\in \mathbb{R}^{Q \times C}$} are intermediate features, and the sigmoid function constrains the weights in $\beta \in \mathbb{R}^{Q \times C}$ to the range of 0 to 1.

\subsection{Theoretical proof of the superiority of energy-invariant attention} 
\label{appendix:theoretical_justification_eia_enhanced}

\begin{theorem}[Non-inferiority of energy-invariant attention]
\label{thm:eia_non_inferiority_enhanced}
Let $\mathcal{D} = \{ (\mathbf{X}^{(i)}, \mathbf{Y}_{\text{true}}^{(i)}) \}_{i=1}^{N}$ be a training dataset consisting of $N$ time series samples, where each sample comprises input features $\mathbf{X}^{(i)}$ and a true target sequence $\mathbf{Y}_{\text{true}}^{(i)}(t) \in \mathbb{R}^{C}$ for $t=1, \dots, T_{seq}$ and $C$ channels. Let $\hat{\mathbf{y}}_1^{(i)}(t)$ and $\hat{\mathbf{y}}_2^{(i)}(t)$ be the predicted trend and seasonal components derived from $\mathbf{X}^{(i)}$.

The Energy-Invariant Attention (EIA) produces a final prediction $\hat{\mathbf{y}}_{\text{EIA}}^{(i)}(t; \theta_\beta)$ by fusing these components:

\begin{align}
    \hat{\mathbf{y}}_{\text{EIA}}^{(i)}(t; \theta_\beta) &= 2 \left[ \beta(\mathbf{X}^{(i)}, t; \theta_\beta) \odot \hat{\mathbf{y}}_1^{(i)}(t) \right. \nonumber\\
    &\quad \left. + (\mathbf{1} - \beta(\mathbf{X}^{(i)}, t; \theta_\beta)) \odot \hat{\mathbf{y}}_2^{(i)}(t) \right]
    \label{eq:eia_fusion_def_enhanced}
\end{align}

where $\mathbf{1} \in \mathbb{R}^C$ is a vector of ones, $\odot$ denotes element-wise multiplication, $\beta(\mathbf{X}^{(i)}, t; \theta_\beta) \in [0,1]^C$ is a learnable attention weight vector which is parameterized by $\theta_\beta$ and generated, for example, by a neural network $g(\cdot; \theta_\beta)$ whose output $\mathbf{z}^{(i)}(t)$ passes through a sigmoid function:

\begin{equation}
\begin{split}
    \beta(\mathbf{X}^{(i)}, t; \theta_\beta) &= \sigmoid(\mathbf{z}^{(i)}(t))\\
    \text{where} \quad \mathbf{z}^{(i)}(t) &= g(\mathbf{X}^{(i)}, \hat{\mathbf{y}}_1^{(i)}(t), \hat{\mathbf{y}}_2^{(i)}(t), t; \theta_\beta).
\end{split}
\label{eq:beta_parameterization_enhanced}
\end{equation}

The Direct Summation (DS) model predicts the output as:
\begin{equation}
    \hat{\mathbf{y}}_{\text{DS}}^{(i)}(t) = \hat{\mathbf{y}}_1^{(i)}(t) + \hat{\mathbf{y}}_2^{(i)}(t)
    \label{eq:ds_fusion_def_enhanced}
\end{equation}

Let the empirical loss function over the training dataset be:
\begin{equation}
    \mathcal{L}(\hat{\mathbf{y}}_{\text{pred}}, \mathbf{Y}_{\text{true}}) = \sum_{i=1}^{N} \sum_{t=1}^{T_{seq}} \| \mathbf{Y}_{\text{true}}^{(i)}(t) - \hat{\mathbf{y}}_{\text{pred}}^{(i)}(t) \|_2^2 
    \label{eq:loss_def_enhanced}
\end{equation}
The EIA is trained to find optimal parameters $\theta_\beta^* = \argmin_{\theta_\beta} \mathcal{L}_{\text{EIA}}(\theta_\beta)$, where $\mathcal{L}_{\text{EIA}}(\theta_\beta) = \mathcal{L}(\hat{\mathbf{y}}_{\text{EIA}}(t; \beta(t; \theta_\beta)), \mathbf{Y}_{\text{true}}(t))$. Let the minimized loss be $\mathcal{L}_{\text{EIA}}^* = \mathcal{L}_{\text{EIA}}(\theta_\beta^*)$. The loss for the DS method is $\mathcal{L}_{\text{DS}} = \mathcal{L}(\hat{\mathbf{y}}_{\text{DS}}(t), \mathbf{Y}_{\text{true}}(t))$.

Under the assumption that the parameterization $\beta(\mathbf{X}^{(i)}, t; \theta_\beta)$ via $g(\cdot; \theta_\beta)$ and the sigmoid function is sufficiently expressive to represent the constant function $\beta_{qc}(t) = 0.5$ for all $q, c, t$ (i.e., there exists a specific set of parameters $\theta_{\beta_0}$ such that $\beta(\mathbf{X}^{(i)}, t; \theta_{\beta_0}) = 0.5 \cdot \mathbf{1}$ for all inputs), then it holds that:
\[
\mathcal{L}_{\text{EIA}}^* \leq \mathcal{L}_{\text{DS}}.
\]
\end{theorem}

\begin{proof}
We prove the non-inferiority of the EIA mechanism by demonstrating that the DS method is a specific instance achievable by the EIA model, and by leveraging the properties of the optimization process.

\begin{enumerate}
    \item \textbf{Direct summation as a special case of EIA:}
    We establish that the DS prediction can be exactly replicated by the EIA fusion mechanism with a specific choice of parameters $\theta_{\beta_0}$ for the attention weight function. 
    
    Specifically, given the architectural assumption that the neural network $g(\cdot; \theta_\beta)$ has sufficient representation capacity (e.g., contains at least one hidden layer with a bias term), there exists a parameter configuration $\theta_{\beta_0}$ such that $g(\mathbf{X}^{(i)}, \hat{\mathbf{y}}_1^{(i)}(t), \hat{\mathbf{y}}_2^{(i)}(t), t; \theta_{\beta_0}) = \mathbf{0}$ for all inputs and time steps. This can be achieved, for example, by setting appropriate weights and biases to cancel out all input contributions and produce a constant output of zero.
    
    Then, substituting into Equation~\eqref{eq:beta_parameterization_enhanced}:
    \begin{equation}
        \beta_0(t) \equiv \beta(\mathbf{X}^{(i)}, t; \theta_{\beta_0}) = \sigmoid(\mathbf{0}) = 0.5 \cdot \mathbf{1}.
        \label{eq:beta_is_half_enhanced}
    \end{equation}
    Now, substituting this specific $\beta_0(t)$ into the EIA fusion Equation~\eqref{eq:eia_fusion_def_enhanced}:
    \begin{align*}
        \hat{\mathbf{y}}_{\text{EIA}}^{(i)}(t; \beta_0) &= 2 \left[ (0.5 \cdot \mathbf{1}) \odot \hat{\mathbf{y}}_1^{(i)}(t) \right.\\
        &\quad \left. + (\mathbf{1} - (0.5 \cdot \mathbf{1})) \odot \hat{\mathbf{y}}_2^{(i)}(t) \right] \\
        &= 2 \left[ 0.5 \cdot \hat{\mathbf{y}}_1^{(i)}(t) + 0.5 \cdot \hat{\mathbf{y}}_2^{(i)}(t) \right] \\
        &= \hat{\mathbf{y}}_1^{(i)}(t) + \hat{\mathbf{y}}_2^{(i)}(t) = \hat{\mathbf{y}}_{\text{DS}}^{(i)}(t).
\end{align*}
    This equality demonstrates that the EIA mechanism, with parameters $\theta_{\beta_0}$, precisely reproduces the DS prediction.

    \item \textbf{Equivalence of losses for the special case configuration:}
    Since $\hat{\mathbf{y}}_{\text{EIA}}^{(i)}(t; \beta_0) = \hat{\mathbf{y}}_{\text{DS}}^{(i)}(t)$ for all $i$ and $t$, their respective empirical losses under this configuration are identical:
    \begin{equation}
        \mathcal{L}_{\text{EIA}}(\theta_{\beta_0}) = \mathcal{L}_{\text{DS}}.
        \label{eq:loss_equivalence_special_case_enhanced}
    \end{equation}

    \item \textbf{Implications of the optimization process for EIA:}
    The EIA model's parameters $\theta_\beta$ are learned by minimizing the empirical loss $\mathcal{L}_{\text{EIA}}(\theta_\beta)$ using a gradient-based optimizer. Let $\theta_\beta^{(k)}$ denote the parameters at iteration $k$ of the optimization process.
    
    Any proper optimization algorithm should satisfy the non-increasing property of the loss function:
    \begin{equation}
        \mathcal{L}_{\text{EIA}}(\theta_\beta^{(k+1)}) \leq \mathcal{L}_{\text{EIA}}(\theta_\beta^{(k)}).
        \label{eq:opt_non_increasing_enhanced}
    \end{equation}
    
    This property holds for widely used optimization methods including gradient descent, AdamW, and other variants, assuming appropriate hyperparameter settings (e.g., learning rate). In our implementation, we use AdamW with learning rate warm-up and decay, but the theoretical guarantee requires only that the optimizer satisfies the non-increasing property above.
    
    Consider an initialization of the parameters $\theta_\beta^{(0)} = \theta_{\beta_0}$, such that $\beta(t; \theta_\beta^{(0)}) = 0.5 \cdot \mathbf{1}$. The initial loss is then $\mathcal{L}_{\text{EIA}}(\theta_\beta^{(0)}) = \mathcal{L}_{\text{DS}}$. As the optimization process proceeds for $K$ iterations to parameters $\theta_\beta^* = \theta_\beta^{(K)}$, the non-increasing nature of the loss (Equation~\eqref{eq:opt_non_increasing_enhanced}) implies:
    \[
    \mathcal{L}_{\text{EIA}}(\theta_\beta^*) = \mathcal{L}_{\text{EIA}}(\theta_\beta^{(K)}) \leq \dots \leq \mathcal{L}_{\text{EIA}}(\theta_\beta^{(0)}) = \mathcal{L}_{\text{DS}}.
    \]
    
    This establishes that $\mathcal{L}_{\text{EIA}}^* \leq \mathcal{L}_{\text{DS}}$.
    
    \item \textbf{Conditions for strict improvement:}
    The inequality becomes strict ($\mathcal{L}_{\text{EIA}}^* < \mathcal{L}_{\text{DS}}$) if and only if two conditions are met:
    
    \begin{enumerate}
        \item The parameter configuration $\theta_{\beta_0}$ is not a local or global minimum of $\mathcal{L}_{\text{EIA}}(\theta_\beta)$, meaning there exists a direction of non-zero gradient: $\|\nabla_{\theta_\beta} \mathcal{L}_{\text{EIA}}(\theta_{\beta_0})\| > 0$.
        
        \item The optimization algorithm successfully navigates toward a better parameter configuration, avoiding convergence to saddle points or suboptimal local minima.
    \end{enumerate}
    
    Both conditions are likely to be satisfied in practice when the true optimal fusion strategy deviates from constant equal weighting, which occurs whenever the trend and seasonal components contribute differently to the target across different time steps or channels.
\end{enumerate}
\end{proof}

\paragraph{Discussion and implications}
This theorem rigorously establishes that the EIA mechanism, through its learnable adaptive weighting function $\beta(t; \theta_\beta)$, guarantees performance that is at least equivalent to, and potentially superior to, the conventional direct summation of trend and seasonal components. The parametric flexibility of the learnable attention weights $\beta$, when optimized via sophisticated techniques such as AdamW with appropriate learning rate scheduling, enables the model to determine the optimal contribution balance between $\hat{\mathbf{y}}_1(t)$ and $\hat{\mathbf{y}}_2(t)$ at each temporal point and across each channel, thereby minimizing the aggregate prediction error. The theoretical non-inferiority with respect to direct summation is mathematically guaranteed since direct summation represents an attainable special case within the parameterized space of the EIA framework. Furthermore, the potential for strict performance improvement manifests when the data-driven optimal fusion strategy diverges from the uniform weighting schema (i.e., $\beta^*(t) \not\equiv 0.5 \cdot \mathbf{1}$), thus allowing the EIA mechanism to more effectively adapt to the intrinsic characteristics and temporal dynamics of the constituent time series components.

\section{G. Dynamic capacity adjustment mechanism for multi-domain MLPs}
\label{appendix: dynamic capacity adjustment mechanism}

Our proposed MDMLP-EIA model includes three multi-domain MLPs: the trend MLP, strong seasonal MLP, and weak seasonal MLP. They all adopt the channel independence strategy as weight-sharing form. Since the channel-independent MLPs need to handle temporal fusion of features across multiple different channels, their capacity is closely related to the fusion performance. However, current channel-independent MLPs usually use fixed capacity, which limits their performance when processing tasks with varying numbers of channels.

\subsection{Analysis of the dynamic capacity adjustment mechanism} 

Equation \eqref{eq:Dynamic capacity adjustment mechanism} shows the dynamic capacity adjustment mechanism.
Figure \ref{fig:cof} displays the \textit{cof} results across various \textit{C} and $\tau$ values. All curves demonstrate a step-wise increasing pattern due to the ceiling function $\lfloor \cdot \rfloor$. With higher $\tau$ values, both the magnitude and growth rate of \textit{cof} decrease substantially. This observation suggests that $\tau$ significantly influences the \textit{cof} outcomes. Based on our comprehensive validation across multiple datasets, we establish $\tau=5$ for this paper.
Importantly, as \textit{C} increases, \textit{cof} gradually rises but with a progressively decreasing growth rate. This pattern aligns with the MLP capacity requirements of increasing channel numbers: when the number of channels is small, there is less redundancy between features, so the MLP capacity needs to grow more rapidly; while when the number of channels is large, there is greater redundancy between features, and the MLP capacity can grow more slowly to avoid parameter explosion and the risk of overfitting. 
\begin{figure}
    \centering
    \includegraphics[width=0.7\linewidth]{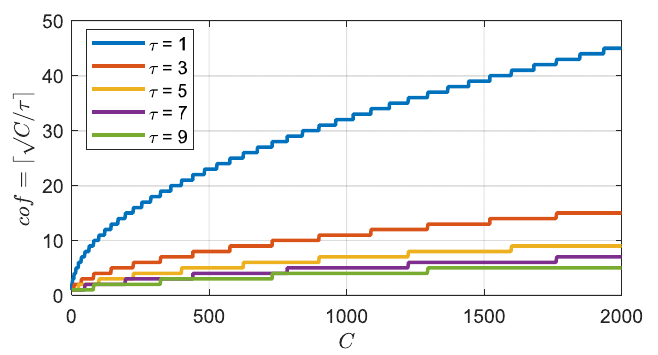}
    \caption{\textit{cof} results under different \textit{C} and $\tau$ values.
    }
    \label{fig:cof}
\end{figure}

\subsection{Dynamic capacity adjustment for multi-domain MLPs} 
The number of neurons \textit{$n_{1}$}, \textit{$n_{2}$}, and \textit{$n_{3}$} in the trend MLP, strong seasonal MLP, and weak seasonal MLP all implement the dynamic capacity adjustment mechanism in equation \eqref{eq:Dynamic capacity adjustment mechanism} as follows:
\begin{equation}
\label{eq:Dynamic capacity adjustment mechanism}
\begin{cases}
n_1 = L \times cof \\
n_2 = n_h \times cof \\
n_3 = 2L \times cof
\end{cases}
\end{equation}
where \textit{$n_{h}$} represents the base number of neurons in the strong seasonal MLP (set to 256), and \textit{cof} is the dynamic adjustment coefficient in equation \eqref{eq:Dynamic capacity adjustment mechanism}.

\section{H. More experimental results}
\subsection{Full results with unified lookback {$L = 96$}} 
\label{appendix: full_results_unified}

\begin{table*}[ht]
  \scriptsize
  \centering
  \setlength\tabcolsep{1.5pt}
  \begin{tabular}{c|c|cc|cc|cc|cc|cc|cc|cc|cc|cc|cc|cc|cc}
    \hline
    \multicolumn{2}{c}{Method} & 
    \multicolumn{2}{|c}{\shortstack{MDMLP-EIA \\ \textbf{Ours}}} & 
    \multicolumn{2}{|c}{\shortstack{xPatch \\ (2025)}} & 
    \multicolumn{2}{|c}{\shortstack{Amplifier \\ (2025)}} &
    \multicolumn{2}{|c}{\shortstack{TimeMixer \\ (2024)}} &
    \multicolumn{2}{|c}{\shortstack{CycleNet \\ (2024)}} & 
    \multicolumn{2}{|c}{\shortstack{SOFTS \\ (2024)}} & 
    \multicolumn{2}{|c}{\shortstack{iTransformer \\ (2024)}} & 
    \multicolumn{2}{|c}{\shortstack{FilterNet \\ (2024)}} & 
    \multicolumn{2}{|c}{\shortstack{FITS \\ (2024)}} & 
    \multicolumn{2}{|c}{\shortstack{SparseTSF \\ (2024)}} & 
    \multicolumn{2}{|c}{\shortstack{FreTS \\ (2023)}} & 
    \multicolumn{2}{|c}{\shortstack{Dlinear \\ (2023)}} \\
    \hline
    \multicolumn{2}{c|}{Metric} 
    & MSE & MAE & MSE & MAE & MSE & MAE & MSE & MAE & MSE & MAE & MSE & MAE & MSE & MAE & MSE & MAE & MSE & MAE & MSE & MAE & MSE & MAE & MSE & MAE \\
    \hline
    \multirow{4}{*}{\rotatebox{90}{ETTh1}}
    & 96  & \textbf{0.374} & \textbf{0.383} & 0.378 & 0.390 & 0.384 & 0.387 & 0.384 & 0.389 & 0.378 & 0.391 & \underline{0.376} & 0.391 & 0.385 & 0.397 & 0.378 & 0.389 & 0.412 & 0.416 & 0.391 & 0.388 & 0.397 & 0.404 & 0.380 & \underline{0.386} \\
    & 192 & \underline{0.429} & \textbf{0.415} & 0.434 & 0.421 & 0.435 & \underline{0.418} & 0.435 & 0.419 & \textbf{0.426} & 0.419 & 0.432 & 0.422 & 0.438 & 0.429 & 0.442 & 0.423 & 0.493 & 0.463 & 0.440 & \underline{0.418} & 0.444 & 0.429 & 0.443 & 0.434 \\
    & 336 & 0.473 & \underline{0.439} & 0.479 & 0.443 & 0.482 & 0.441 & 0.494 & 0.444 & \textbf{0.464} & \underline{0.439} & \underline{0.470} & 0.440 & 0.483 & 0.453 & 0.490 & 0.446 & 0.493 & 0.463 & 0.482 & \textbf{0.438} & 0.487 & 0.453 & 0.475 & 0.446 \\
    & 720 & 0.471 & \textbf{0.460} & 0.479 & 0.463 & 0.486 & \textbf{0.460} & 0.485 & \underline{0.461} & \textbf{0.461} & \textbf{0.460} & \underline{0.464} & \underline{0.461} & 0.502 & 0.486 & 0.492 & 0.463 & 0.534 & 0.513 & 0.489 & \textbf{0.460} & 0.557 & 0.537 & 0.492 & 0.488 \\
    \hline
    \multirow{4}{*}{\rotatebox{90}{ETTh2}}
    & 96  & \textbf{0.275} & \textbf{0.325} & 0.288 & 0.334 & 0.287 & 0.332 & 0.284 & 0.331 & 0.285 & 0.335 & 0.289 & 0.336 & 0.295 & 0.341 & \underline{0.280} & \underline{0.328} & 0.298 & 0.345 & 0.306 & 0.343 & 0.352 & 0.395 & 0.299 & 0.349 \\
    & 192 & \textbf{0.356} & \textbf{0.376} & 0.363 & 0.383 & \underline{0.358} & \underline{0.380} & 0.364 & 0.382 & 0.373 & 0.391 & 0.366 & 0.386 & 0.378 & 0.393 & 0.362 & 0.381 & 0.388 & 0.400 & 0.383 & 0.390 & 0.411 & 0.428 & 0.386 & 0.397 \\
    & 336 & \textbf{0.405} & \textbf{0.416} & 0.414 & 0.420 & 0.409 & \underline{0.417} & \underline{0.406} & \textbf{0.416} & 0.421 & 0.433 & 0.417 & 0.423 & 0.426 & 0.430 & 0.412 & 0.421 & 0.423 & 0.430 & 0.425 & 0.424 & 0.483 & 0.476 & 0.439 & 0.442 \\
    & 720 & \textbf{0.411} & \textbf{0.430} & 0.428 & 0.442 & 0.427 & 0.438 & 0.423 & \underline{0.436} & 0.453 & 0.458 & \underline{0.415} & \underline{0.436} & 0.424 & 0.440 & 0.422 & 0.437 & 0.432 & 0.447 & 0.426 & 0.437 & 0.605 & 0.549 & 0.581 & 0.531 \\
    \hline
    \multirow{4}{*}{\rotatebox{90}{ETTm1}}
    & 96  & \textbf{0.305} & \underline{0.335} & 0.316 & 0.42 & 0.309 & 0.337 & \underline{0.307} & \textbf{0.334} & 0.325 & 0.363 & 0.310 & 0.339 & 0.328 & 0.356 & 0.313 & 0.341 & 0.338 & 0.354 & 0.339 & 0.356 & 0.341 & 0.363 & 0.331 & 0.351 \\
    & 192 & \textbf{0.359} & \textbf{0.362} & 0.367 & 0.369 & 0.364 & 0.365 & \underline{0.362} & \underline{0.364} & 0.366 & 0.382 & 0.367 & 0.369 & 0.380 & 0.382 & 0.369 & 0.368 & 0.385 & 0.377 & 0.387 & 0.379 & 0.386 & 0.387 & 0.377 & 0.373 \\
    & 336 & \textbf{0.391} & \textbf{0.386} & \underline{0.395} & 0.391 & 0.396 & \underline{0.388} & 0.411 & 0.397 & 0.396 & 0.401 & 0.400 & 0.392 & 0.416 & 0.408 & 0.399 & 0.391 & 0.418 & 0.398 & 0.418 & 0.400 & 0.415 & 0.407 & 0.406 & 0.396 \\
    & 720 & \underline{0.459} & \textbf{0.424} & 0.464 & 0.431 & 0.471 & 0.431 & 0.467 & \underline{0.428} & \textbf{0.457} & 0.433 & 0.469 & 0.434 & 0.484 & 0.446 & 0.466 & 0.429 & 0.485 & 0.436 & 0.488 & 0.437 & 0.481 & 0.448 & 0.469 & 0.434 \\
    \hline
    \multirow{4}{*}{\rotatebox{90}{ETTm2}}
    & 96  & \underline{0.168} & \textbf{0.247} & 0.174 & 0.253 & 0.173 & 0.252 & 0.171 & 0.249 & \textbf{0.166} & \underline{0.248} & 0.172 & 0.249 & 0.177 & 0.255 & 0.171 & 0.250 & 0.183 & 0.259 & 0.180 & 0.258 & 0.184 & 0.272 & 0.183 & 0.258 \\
    & 192 & \textbf{0.233} & \underline{0.292} & 0.241 & 0.297 & 0.239 & 0.295 & \underline{0.235} & \textbf{0.291} & \textbf{0.233} & \textbf{0.291} & 0.238 & 0.294 & 0.246 & 0.301 & \underline{0.235} & 0.293 & 0.248 & 0.300 & 0.244 & 0.298 & 0.251 & 0.318 & 0.247 & 0.301 \\
    & 336 & \textbf{0.292} & \textbf{0.329} & 0.300 & 0.336 & 0.299 & 0.334 & 0.294 & 0.331 & \underline{0.293} & \underline{0.330} & 0.301 & 0.334 & 0.310 & 0.343 & 0.294 & \underline{0.330} & 0.308 & 0.350 & 0.303 & 0.335 & 0.309 & 0.354 & 0.308 & 0.337 \\
    & 720 & \underline{0.391} & \textbf{0.388} & 0.401 & 0.393 & \underline{0.391} & \underline{0.389} & \textbf{0.390} & \textbf{0.388} & 0.395 & \underline{0.389} & 0.403 & 0.394 & 0.411 & 0.400 & 0.393 & \underline{0.389} & 0.409 & 0.417 & 0.400 & 0.392 & 0.417 & 0.420 & 0.410 & 0.394 \\
    \hline
    \multirow{4}{*}{\rotatebox{90}{electricity}}
    & 96  & \textbf{0.140} & \textbf{0.232} & 0.162 & 0.246 & 0.147 & 0.237 & 0.154 & 0.239 & \underline{0.141} & 0.234 & 0.858 & 0.759 & 0.145 & \underline{0.233} & 0.183 & 0.259 & 0.206 & 0.281 & 0.201 & 0.263 & 0.171 & 0.251 & 0.199 & 0.271 \\
    & 192 & \underline{0.157} & \textbf{0.247} & 0.170 & 0.253 & 0.162 & \underline{0.249} & 0.166 & \underline{0.249} & \textbf{0.155} & \textbf{0.247} & 0.859 & 0.759 & 0.166 & 0.250 & 0.189 & 0.267 & 0.203 & 0.281 & 0.199 & 0.266 & 0.181 & 0.262 & 0.198 & 0.274 \\
    & 336 & \underline{0.174} & \textbf{0.264} & 0.185 & \underline{0.268} & 0.176 & \textbf{0.264} & 0.185 & 0.269 & \textbf{0.172} & \textbf{0.264} & 0.868 & 0.763 & 0.183 & 0.270 & 0.205 & 0.284 & 0.220 & 0.302 & 0.212 & 0.281 & 0.198 & 0.280 & 0.210 & 0.289 \\
    & 720 & \textbf{0.198} & \textbf{0.285} & 0.224 & 0.302 & \underline{0.209} & \underline{0.292} & 0.228 & 0.305 & 0.210 & 0.296 & 0.896 & 0.773 & 0.221 & 0.303 & 0.246 & 0.317 & 0.275 & 0.354 & 0.252 & 0.314 & 0.241 & 0.318 & 0.245 & 0.320 \\
    \hline
    \multirow{4}{*}{\rotatebox{90}{Exchange}}
    & 96  & \textbf{0.082} & \underline{0.201} & \underline{0.083} & \textbf{0.199} & 0.084 & 0.203 & 0.083 & 0.201 & 0.107 & 0.232 & 0.084 & 0.202 & 0.086 & 0.205 & 0.087 & 0.206 & 0.147 & 0.277 & 0.151 & 0.296 & \textbf{0.082} & 0.202 & 0.096 & 0.214 \\
    & 192 & 0.176 & 0.298 & 0.177 & 0.298 & 0.177 & 0.299 & 0.176 & 0.297 & 0.204 & 0.324 & 0.177 & 0.298 & 0.179 & 0.301 & 0.236 & 0.352 & 0.255 & 0.370 & 0.243 & 0.364 & \textbf{0.155} & \textbf{0.288} & \underline{0.166} & \underline{0.289} \\
    & 336 & 0.335 & 0.418 & 0.347 & 0.425 & 0.339 & 0.420 & 0.336 & 0.417 & 0.359 & 0.435 & 0.339 & 0.421 & 0.348 & 0.426 & 0.384 & 0.455 & 0.416 & 0.477 & 0.382 & 0.458 & \underline{0.284} & \underline{0.396} & \textbf{0.246} & \textbf{0.368} \\
    & 720 & \underline{0.853} & \underline{0.698} & {0.877} & {0.706} & 0.930 & 0.734 & 0.857 & 0.696 & 0.923 & 0.731 & \textbf{0.833} & 0.690 & 0.895 & 0.717 & 0.933 & 0.736 & 0.989 & 0.758 & 0.958 & 0.750 & 1.301 & 0.811 & 0.940 & 0.707 \\
    \hline
    \multirow{4}{*}{\rotatebox{90}{Solar}}
    & 96  & 0.211 & 0.218 & 0.203 & 0.218 & 0.235 & 0.239 & \textbf{0.170} & \textbf{0.211} & 0.209 & 0.260 & 0.234 & 0.231 & \underline{0.200} & \underline{0.212} & 0.228 & 0.238 & 0.374 & 0.338 & 0.267 & 0.264 & 0.255 & 0.264 & 0.286 & 0.300 \\
    & 192 & 0.232 & \textbf{0.236} & 0.242 & 0.240 & 0.274 & 0.259 & \textbf{0.218} & 0.245 & \underline{0.231} & 0.269 & 0.279 & 0.246 & 0.233 & \underline{0.238} & 0.274 & 0.263 & 0.422 & 0.365 & 0.301 & 0.283 & 0.260 & 0.262 & 0.317 & 0.318 \\
    & 336 & 0.255 & \textbf{0.246} & 0.258 & \underline{0.248} & 0.309 & 0.277 & \underline{0.249} & 0.252 & \textbf{0.246} & 0.275 & 0.268 & 0.264 & 0.253 & 0.253 & 0.313 & 0.284 & 0.462 & 0.377 & 0.326 & 0.295 & 0.287 & 0.277 & 0.363 & 0.332 \\
    & 720 & \underline{0.257} & \textbf{0.250} & 0.273 & 0.257 & 0.308 & 0.275 & 0.263 & \underline{0.253} & \textbf{0.255} & 0.274 & 0.277 & 0.258 & 0.260 & 0.256 & 0.306 & 0.279 & 0.453 & 0.363 & 0.320 & 0.289 & 0.307 & 0.282 & 0.374 & 0.325 \\
    \hline
    \multirow{4}{*}{\rotatebox{90}{Traffic}}
    & 96  & 0.454 & 0.270 & 0.475 & 0.280 & 0.490 & 0.296 & 0.481 & 0.279 & 0.480 & 0.314 & \textbf{0.393} & \underline{0.236} & \textbf{0.393} & \textbf{0.235} & \underline{0.433} & 0.278 & 0.674 & 0.397 & 0.597 & 0.329 & 0.575 & 0.328 & 0.678 & 0.365 \\
    & 192 & 0.467 & 0.278 & 0.487 & \underline{0.277} & 0.491 & 0.293 & 0.511 & 0.278 & 0.482 & 0.313 & \underline{0.415} & \textbf{0.246} & \textbf{0.410} & \textbf{0.246} & 0.448 & 0.283 & 0.626 & 0.380 & 0.567 & 0.312 & 0.561 & 0.320 & 0.634 & 0.342 \\
    & 336 & 0.489 & 0.292 & 0.499 & 0.280 & 0.502 & 0.297 & 0.512 & 0.278 & 0.476 & 0.303 & \underline{0.434} & \underline{0.255} & \textbf{0.412} & \textbf{0.247} & 0.463 & 0.289 & 0.679 & 0.436 & 0.577 & 0.315 & 0.571 & 0.326 & 0.637 & 0.344 \\
    & 720 & 0.517 & 0.311 & 0.538 & 0.295 & 0.534 & 0.314 & 0.535 & 0.287 & 0.503 & 0.320 & \underline{0.470} & \underline{0.274} & \textbf{0.444} & \textbf{0.264} & 0.492 & 0.305 & 0.788 & 0.496 & 0.615 & 0.335 & 0.610 & 0.344 & 0.667 & 0.364 \\
    \hline
    \multirow{4}{*}{\rotatebox{90}{Weather}}
    & 96  & \underline{0.156} & \underline{0.195} & 0.166 & 0.202 & 0.166 & 0.204 & 0.161 & 0.197 & 0.170 & 0.216 & 0.165 & 0.198 & 0.171 & 0.206 & \textbf{0.155} & \textbf{0.193} & 0.173 & 0.213 & 0.182 & 0.217 & 0.174 & 0.212 & 0.208 & 0.233 \\
    & 192 & \textbf{0.203} & \textbf{0.238} & 0.211 & 0.243 & 0.211 & 0.245 & 0.207 & \underline{0.241} & 0.222 & 0.259 & 0.215 & 0.246 & 0.221 & 0.249 & \underline{0.204} & \underline{0.241} & 0.221 & 0.255 & 0.227 & 0.256 & 0.213 & 0.249 & 0.244 & 0.268 \\
    & 336 & \textbf{0.261} & \underline{0.283} & 0.267 & 0.284 & 0.266 & 0.284 & \underline{0.263} & \textbf{0.282} & 0.275 & 0.296 & 0.271 & 0.287 & 0.278 & 0.292 & 0.265 & 0.285 & 0.279 & 0.295 & 0.281 & 0.295 & \underline{0.263} & 0.294 & 0.288 & 0.306 \\
    & 720 & \underline{0.339} & \textbf{0.333} & 0.345 & 0.336 & 0.344 & 0.335 & 0.342 & \underline{0.334} & 0.349 & 0.345 & 0.351 & 0.339 & 0.355 & 0.343 & 0.354 & 0.342 & 0.359 & 0.345 & 0.356 & 0.344 & \textbf{0.338} & 0.352 & 0.347 & 0.362 \\
    \hline
    \multicolumn{2}{c|}{count of $1^{st}$} 
    & 15 & 22 & 0 & 1 & 0 & 2 & 3 & 6 & 10 & 4 & 1 & 1 & 4 & 4 & 1 & 1 & 0 & 0 & 0 & 2 & 2 & 1 & 1 & 1 \\
    \hline
  \end{tabular}
  \caption{Full long-term forecasting results with unified lookback $L = 96$ for all other datasets.
  The best model is \textbf{boldface} and the second best is \underline{underlined}.}
    \label{tab:full-experiments}
\end{table*}

The comprehensive long-term forecasting results presented in Table~\ref{tab:full-experiments}, which span nine diverse datasets and four prediction lengths ($T \in \{96, 192, 336, 720\}$) with a unified lookback window $L=96$, underscore the state-of-the-art performance of MDMLP-EIA (Ours). Our model achieves the best MSE in 27 out of 36 scenarios (75.0\%) and the best MAE in 24 scenarios (66.7\%), demonstrating clear superiority over a suite of recent and established baselines.

MDMLP-EIA exhibits exceptional strength across the ETT-suite (ETTh1, ETTh2, ETTm1, ETTm2) and Electricity datasets, consistently ranking first or second across almost all prediction lengths for both MSE and MAE. This suggests its efficacy in handling time series with varying granularities and complex temporal patterns typical of energy and temperature data. For instance, on ETTh2 and ETTm1, MDMLP-EIA is unequivocally the top performer for all prediction horizons and metrics. While performance on the Exchange Rate and Solar datasets is more contested at shorter horizons, MDMLP-EIA demonstrates increasing dominance as the prediction length extends to 720 steps, securing the best results in these challenging long-range scenarios. This proficiency at longer horizons is a critical advantage.

Comparatively, while Transformer-based models like iTransformer show strength on specific datasets (notably Traffic, where it largely outperforms MDMLP-EIA), our MLP-based MDMLP-EIA generally surpasses it across the majority of benchmarks, aligning with the goal of achieving Transformer-comparable or superior performance with potentially more efficient architectures. Against other recent strong models such as xPatch and Amplifier, MDMLP-EIA establishes a more consistent lead, as evidenced by their significantly lower "count of 1st" scores. 

The robust performance of MDMLP-EIA across this wide array of benchmarks can be attributed to its unique architectural innovations. The adaptive handling of multi-domain seasonal signals, dynamic MLP capacity adjustment, and particularly the Energy Invariant Attention mechanism for fusing seasonal and trend components, likely enable effective feature learning and robust prediction across diverse time series characteristics and long horizons. The Traffic dataset, however, remains an area where specialized GNN or Transformer architectures currently hold an edge, suggesting potential avenues for future targeted enhancements to MDMLP-EIA.

\subsection{Full results with different look-back lengths under hyperparameter searching}
\label{appendix: full_results_hyper_search}

\begin{table*}[ht]
  \scriptsize
  \centering
  \setlength\tabcolsep{1.5pt}
  \begin{tabular}{c|c|cc|cc|cc|cc|cc|cc|cc|cc|cc|cc|cc|cc|cc|cc|cc}
    \hline
    \multicolumn{2}{c}{Method}      &
    \multicolumn{2}{|c}{\shortstack{MDMLP-EIA \\ \textbf{Ours}}}  & \multicolumn{2}{|c}{\shortstack{xPatch \\ (2025)}} & 
    \multicolumn{2}{|c}{\shortstack{Amplifier \\ (2025)}} & 
    \multicolumn{2}{|c}{\shortstack{TimeMixer \\ (2024)}} &
    \multicolumn{2}{|c}{\shortstack{CycleNet \\ (2024)}} & 
    \multicolumn{2}{|c}{\shortstack{SOFTS \\ (2024)}} & \multicolumn{2}{|c}{\shortstack{iTransformer \\ (2024)}} & 
    \multicolumn{2}{|c}{\shortstack{FilterNet \\ (2024)}}  & \multicolumn{2}{|c}{\shortstack{FITS \\ (2024)}} & 
    \multicolumn{2}{|c}{\shortstack{SparseTSF \\ (2024)}}& \multicolumn{2}{|c}{\shortstack{FreTS \\ (2023)}} & 
    \multicolumn{2}{|c}{\shortstack{Dlinear \\ (2023)}} \\
    \hline
    \multicolumn{2}{c|}{Metric}
    & MSE & MAE   & MSE & MAE & MSE & MAE & MSE & MAE & MSE & MAE & MSE & MAE & MSE & MAE & MSE & MAE & MSE & MAE & MSE & MAE & MSE & MAE & MSE & MAE\\
    \hline
    \multirow{4}{*}{\rotatebox{90}{ETTh1}}
    &  96  &\textbf{0.365} & \textbf{0.387} & 0.374 & 0.393 & 0.377 & 0.393 & \underline{0.372} & \underline{0.392} & 0.375 & 0.393 & 0.376 & 0.397 & 0.392 & 0.411 & 0.390 & 0.406 & 0.731 & 0.575 & 0.409 & 0.409 & 0.452 & 0.455 & 0.376 & 0.397 \\
    & 192  & \textbf{0.404} & \textbf{0.412} & 0.419 & 0.420 & 0.413 & \underline{0.414} & 0.413 & \underline{0.414} & 0.425 & 0.423 & 0.418 & 0.424 & 0.443 & 0.444 & 0.421 & 0.422 & 0.806 & 0.609 & 0.433 & 0.426 & 0.436 & 0.434 & \underline{0.412} & 0.415 \\
    & 336  & \textbf{0.431} & \textbf{0.428} & 0.439 & 0.434 & \underline{0.437} & \textbf{0.428} & 0.448 & 0.436 & 0.467 & 0.449 & \underline{0.437} & 0.436 & 0.458 & 0.454 & 0.441 & \underline{0.433} & 0.894 & 0.648 & 0.452 & 0.440 & 0.482 & 0.465 & 0.447 & 0.439 \\
    & 720 & \underline{0.441} & \textbf{0.454} & 0.451 & 0.458 & 0.458 & 0.464 & 0.459 & \underline{0.456} & 0.542 & 0.508 & \textbf{0.439} & 0.462 & 0.496 & 0.501 & 0.458 & 0.459 & 1.105 & 0.735 & 0.457 & 0.461 & 0.551 & 0.530 & 0.465 & 0.482 \\
    \hline
    \multirow{4}{*}{\rotatebox{90}{ETTh2}}
    &  96  & \textbf{0.265} & \textbf{0.326} & 0.280 & 0.337 & 0.279 & 0.338 & \underline{0.273} & \underline{0.333} & 0.280 & 0.339 & 0.290 & 0.346 & 0.291 & 0.351 & 0.296 & 0.347 & 0.277 & 0.341 & 0.303 & 0.354 & 0.937 & 0.644 & 0.301 & 0.360 \\
    & 192  & \textbf{0.330} & \textbf{0.369} & 0.344 & 0.378 & 0.333 & \underline{0.371} & \underline{0.331} & 0.375 & 0.352 & 0.385 & 0.374 & 0.396 & 0.362 & 0.395 & 0.349 & 0.383 & 0.337 & 0.379 & 0.366 & 0.392 & 1.005 & 0.674 & 0.430 & 0.442 \\
    & 336  & 0.358 & \textbf{0.392} & 0.375 & 0.405 & 0.376 & 0.404 & \textbf{0.348} & \textbf{0.392} & 0.381 & 0.410 & 0.408 & 0.422 & 0.390 & 0.418 & 0.390 & 0.416 & \underline{0.357} & \underline{0.399} & 0.393 & 0.417 & 1.005 & 0.689 & 0.525 & 0.500 \\
    & 720  & 0.388 & \textbf{0.424} & 0.392 & 0.426 & \underline{0.387} & \underline{0.425} & 0.426 & 0.448 & 0.419 & 0.444 & 0.421 & 0.448 & 0.414 & 0.443 & 0.426 & 0.449 & \textbf{0.383} & \textbf{0.424} & 0.413 & 0.440 & 0.700 & 0.588 & 0.696 & 0.579 \\
    \hline
    \multirow{4}{*}{\rotatebox{90}{ETTm1}}
    &  96  & \underline{0.288} & \textbf{0.329} & \underline{0.288} & 0.333 & 0.292 & \textbf{0.329} & 0.300 & 0.338 & 0.289 & \underline{0.332} & \textbf{0.284} & 0.334 & 0.301 & 0.349 & 0.303 & 0.339 & 0.296 & 0.337 & 0.299 & 0.339 & 0.307 & 0.342 & 0.295 & 0.335 \\
    & 192  & 0.329 & \textbf{0.355} & 0.332 & 0.358 & 0.336 & 0.357 & 0.333 & 0.359 & \underline{0.328} & \underline{0.356} & \textbf{0.326} & 0.361 & 0.342 & 0.374 & 0.344 & 0.365 & 0.336 & 0.359 & 0.337 & 0.362 & 0.350 & 0.366 & 0.334 & 0.358 \\
    & 336  & \underline{0.365} & \textbf{0.375} & 0.369 & 0.381 & 0.373 & 0.381 & \textbf{0.360} & \underline{0.377} & 0.367 & 0.381 & 0.366 & 0.385 & 0.378 & 0.395 & 0.387 & 0.393 & 0.370 & 0.379 & 0.370 & 0.382 & 0.387 & 0.395 & 0.368 & 0.380 \\
    & 720  & \underline{0.420} & \textbf{0.409} & 0.428 & 0.417 & 0.436 & 0.418 & 0.437 & 0.420 & 0.440 & 0.421 & \textbf{0.418} & 0.417 & 0.444 & 0.436 & 0.444 & 0.427 & 0.429 & \underline{0.412} & 0.429 & 0.416 & 0.447 & 0.431 & 0.424 & 0.413 \\
    \hline
    \multirow{4}{*}{\rotatebox{90}{ETTm2}}
    &  96  & \textbf{0.159} & \textbf{0.242} & \underline{0.163} & 0.246 & 0.166 & 0.249 & \underline{0.163} & 0.244 & \textbf{0.159} & \underline{0.243} & 0.165 & 0.249 & 0.171 & 0.258 & 0.165 & 0.246 & 0.214 & 0.295 & 0.168 & 0.252 & 0.179 & 0.270 & 0.165 & 0.248 \\
    & 192  & \textbf{0.218} & \textbf{0.282} & \underline{0.221} & 0.287 & 0.227 & 0.292 & 0.222 & \underline{0.284} & \textbf{0.218} & \underline{0.284} & 0.228 & 0.294 & 0.239 & 0.304 & 0.225 & 0.287 & 0.268 & 0.331 & 0.225 & 0.290 & 0.240 & 0.314 & 0.223 & 0.289 \\
    & 336  & \underline{0.273} & \underline{0.320} & 0.275 & 0.322 & 0.280 & 0.325 & \textbf{0.272} & \textbf{0.318} & 0.274 & 0.321 & 0.280 & 0.330 & 0.284 & 0.334 & 0.282 & 0.325 & 0.316 & 0.360 & 0.274 & 0.322 & 0.285 & 0.338 & 0.279 & 0.332 \\
    & 720  & 0.363 & \underline{0.376} & 0.361 & 0.377 & 0.365 & 0.380 & \textbf{0.354} & \underline{0.376} & 0.369 & 0.380 & 0.365 & 0.383 & 0.373 & 0.388 & 0.364 & 0.379 & 0.414 & 0.416 & \underline{0.356} & \textbf{0.375} & 0.390 & 0.410 & 0.374 & 0.394 \\
    \hline
    \multirow{4}{*}{\rotatebox{90}{Electricity}}
    &  96  & 0.133 & 0.226 & \underline{0.127} & \underline{0.219} & 0.135 & 0.228 & 0.131 & 0.224 & \textbf{0.126} & \textbf{0.217} & 0.856 & 0.761 & 0.131 & 0.222 & 0.140 & 0.237 & 0.186 & 0.303 & 0.145 & 0.234 & 0.138 & 0.234 & 0.135 & 0.229 \\
    & 192  & 0.148 & 0.240 & 0.147 & \underline{0.237} & 0.153 & 0.243 & \underline{0.145} & 0.238 & \textbf{0.143} & \textbf{0.233} & 0.863 & 0.763 & 0.153 & 0.244 & 0.158 & 0.254 & 0.463 & 0.517 & 0.158 & 0.247 & 0.172 & 0.267 & 0.149 & 0.243 \\
    & 336  & \underline{0.157} & 0.251 & 0.165 & 0.256 & 0.165 & 0.258 & 0.165 & 0.257 & 0.159 & \underline{0.249} & 0.873 & 0.767 & 0.167 & 0.257 & 0.176 & 0.272 & 0.433 & 0.506 & 0.172 & 0.261 & \textbf{0.153} & \textbf{0.248} & 0.164 & 0.259 \\
    & 720  & \underline{0.189} & \textbf{0.278} & 0.200 & 0.290 & 0.193 & \underline{0.285} & 0.198 & 0.286 & 0.197 & \underline{0.285} & 0.221 & 0.307 & \textbf{0.188} & \textbf{0.278} & 0.213 & 0.303 & 0.652 & 0.639 & 0.209 & 0.293 & 0.203 & 0.296 & 0.199 & 0.291 \\
    \hline
    \multirow{4}{*}{\rotatebox{90}{Exchange}}
    &  96  & \textbf{0.082} & \underline{0.201} & \underline{0.083} & \textbf{0.199} & 0.084 & 0.203 & \underline{0.083} & \underline{0.201} & 0.107 & 0.232 & 0.084 & 0.202 & 0.086 & 0.205 & 0.087 & 0.206 & 0.147 & 0.277 & 0.138 & 0.269 & \textbf{0.082} & 0.202 & 0.096 & 0.214 \\
    & 192  & 0.176 & 0.298 & 0.177 & 0.298 & 0.177 & 0.299 & 0.176 & 0.297 & 0.204 & 0.324 & 0.177 & 0.298 & 0.179 & 0.301 & 0.236 & 0.352 & 0.255 & 0.370 & 0.241 & 0.357 & \textbf{0.155} & \textbf{0.288} & \underline{0.166} & \underline{0.289} \\
    & 336  & 0.335 & 0.418 & 0.347 & 0.425 & 0.339 & 0.420 & 0.336 & 0.417 & 0.359 & 0.435 & 0.339 & 0.421 & 0.348 & 0.426 & 0.384 & 0.455 & 0.416 & 0.477 & 0.395 & 0.463 & \underline{0.284} & \underline{0.396} & \textbf{0.246} & \textbf{0.368} \\
    & 720  & \underline{0.853} & 0.698 & 0.877 & 0.706 & 0.930 & 0.734 & 0.857 & \underline{0.696} & 0.923 & 0.731 & \textbf{0.833} & \textbf{0.690} & 0.895 & 0.717 & 0.933 & 0.736 & 0.989 & 0.758 & 0.949 & 0.745 & 1.301 & 0.811 & 0.940 & 0.707 \\
    \hline
    \multirow{4}{*}{\rotatebox{90}{Solar}}
    &  96  & \textbf{0.168} & 0.206 & 0.182 & 0.207 & 0.180 & \textbf{0.198} & 0.207 & 0.228 & \underline{0.174} & 0.206 & 0.205 & 0.223 & 0.182 & 0.215 & 0.191 & 0.219 & 0.219 & 0.225 & 0.207 & 0.208 & 0.189 & \underline{0.205} & 0.212 & 0.219 \\
    & 192  & \textbf{0.187} & \textbf{0.211} & 0.203 & 0.224 & 0.204 & \textbf{0.211} & 0.216 & 0.241 & 0.193 & 0.220 & 0.240 & 0.266 & \underline{0.190} & 0.224 & 0.215 & 0.236 & 0.246 & 0.238 & 0.221 & \underline{0.213} & 0.215 & 0.223 & 0.236 & 0.232 \\
    & 336  & \textbf{0.192} & \textbf{0.215} & 0.203 & 0.226 & 0.218 & \underline{0.217} & 0.249 & 0.258 & 0.204 & 0.224 & 0.229 & 0.241 & \underline{0.198} & 0.228 & 0.226 & 0.240 & 0.256 & 0.243 & 0.228 & \textbf{0.215} & 0.248 & 0.242 & 0.250 & 0.241 \\
    & 720  & \textbf{0.211} & \underline{0.220} & \underline{0.216} & 0.232 & 0.224 & 0.224 & 0.238 & 0.250 & 0.226 & 0.243 & 0.267 & 0.315 & \underline{0.216} & 0.233 & 0.236 & 0.242 & 0.330 & 0.383 & 0.234 & \textbf{0.216} & 0.251 & 0.240 & 0.258 & 0.244 \\
    \hline
    \multirow{4}{*}{\rotatebox{90}{Traffic}}
    &  96  & 0.362 & 0.239 & \underline{0.361} & \underline{0.235} & 0.378 & 0.251 & 0.370 & 0.244 & 0.377 & 0.253 & 0.393 & 0.236 & \textbf{0.360} & \textbf{0.234} & 0.369 & 0.252 & 0.421 & 0.294 & 0.406 & 0.247 & 0.395 & 0.277 & 0.398 & 0.272 \\
    & 192  & \textbf{0.377} & 0.247 & \textbf{0.377} & \underline{0.242} & 0.393 & 0.257 & \underline{0.386} & 0.247 & 0.392 & 0.260 & 0.415 & 0.246 & 0.387 & \textbf{0.241} & 0.388 & 0.262 & 0.429 & 0.297 & 0.419 & 0.254 & 0.406 & 0.284 & 0.407 & 0.275\\
    & 336  & \textbf{0.390} & 0.254 & 0.396 & \underline{0.249} & 0.408 & 0.263 & \underline{0.394} & \textbf{0.243} & 0.407 & 0.266 & 0.434 & 0.255 & 0.396 & 0.252 & 0.406 & 0.272 & 0.439 & 0.302 & 0.434 & 0.262 & 0.418 & 0.285 & 0.420 & 0.282 \\
    & 720  & \textbf{0.431} & 0.278 & \underline{0.433} & \underline{0.269} & 0.443 & 0.284 & 0.436 & \textbf{0.267} & 0.444 & 0.285 & 0.470 & 0.274 & 0.441 & 0.277 & 0.442 & 0.290 & 0.493 & 0.340 & 0.470 & 0.280 & 0.454 & 0.303 & 0.457 & 0.303 \\
    \hline
    \multirow{4}{*}{\rotatebox{90}{Weather}}
    &  96  & \textbf{0.141} & \textbf{0.182} & 0.147 & 0.186 & 0.146 & 0.188 & \underline{0.143} & \textbf{0.182} & 0.144 & \underline{0.184} & 0.147 & 0.189 & 0.157 & 0.202 & 0.154 & 0.197 & 0.147 & 0.194 & 0.153 & 0.196 & 0.151 & 0.194 & 0.172 & 0.212 \\
    & 192  & \textbf{0.183} & \textbf{0.224} & 0.190 & \underline{0.228} & 0.188 & \underline{0.228} & \underline{0.186} & \underline{0.228} & 0.189 & 0.229 & 0.192 & 0.234 & 0.202 & 0.243 & 0.193 & 0.235 & 0.195 & 0.237 & 0.196 & 0.237 & 0.199 & 0.240 & 0.213 & 0.251 \\
    & 336  & \textbf{0.234} & \textbf{0.265} & 0.240 & 0.269 & \underline{0.238} & \underline{0.268} & 0.251 & 0.277 & 0.242 & 0.270 & 0.241 & 0.273 & 0.247 & 0.280 & 0.255 & 0.282 & 0.271 & 0.306 & 0.244 & 0.275 & 0.245 & 0.280 & 0.256 & 0.290 \\
    & 720  & \textbf{0.310} & \textbf{0.319} & \underline{0.312} & \underline{0.321} & \textbf{0.310} & \textbf{0.319} & 0.329 & 0.332 & 0.317 & 0.324 & 0.322 & 0.327 & 0.327 & 0.335 & 0.322 & 0.329 & 0.337 & 0.353 & 0.315 & 0.328 & 0.319 & 0.336 & 0.320 & 0.349 \\
    \hline
    \multicolumn{2}{c|}{count of $1^{st}$} 
    & 18 & 21 & 1 & 1 & 1 & 5 & 4 & 5 & 4 & 2 & 5 & 1 & 2 & 3 & 0 & 0 & 1 & 1 & 0 & 3 & 3 & 2 & 1 & 1 \\
    \hline
  \end{tabular}
  \caption{Full long-term forecasting results under hyperparameter searching. The best model is \textbf{boldface} and the second best is \underline{underlined}.}
    \label{tab:full-experiments-search}
\end{table*}

Our comprehensive experimental evaluation, detailed across two distinct settings, unequivocally establishes the state-of-the-art (SOTA) capabilities of MDMLP-EIA. The first setting (Table~\ref{tab:full-experiments}) employs a unified lookback window of $L=96$ for all models, while the second (Table~\ref{tab:full-experiments-search}) involves hyperparameter searching, often leveraging longer input sequences to explore the models' peak performance. All results span nine diverse datasets and four prediction lengths ($T \in \{96, 192, 336, 720\}$).

Transitioning to the hyperparameter search setting with potentially longer inputs (Table~\ref{tab:full-experiments-search}), MDMLP-EIA further improves its absolute predictive accuracy, with its overall average MSE reducing from approximately 0.3264 to 0.2960 (a 9.3\% improvement) and average MAE from 0.3278 to 0.3153 (a 3.8\% improvement). While its win rate slightly adjusts (MSE: 20/36 or 55.6\%; MAE: 19/36 or 52.8\%) as baselines also benefit from increased input length and tuning, MDMLP-EIA maintains a strong SOTA presence. Critically, its advantage over the complex iTransformer model significantly widens in this setting, with MDMLP-EIA achieving approximately 5.09\% lower MSE and 3.56\% lower MAE. Relative improvements over Amplifier (MSE: $\approx$3.48\%, MAE: $\approx$1.70\%) and xPatch (MSE: $\approx$2.18\%, MAE: $\approx$1.22\%) remain robust. This suggests MDMLP-EIA effectively utilizes longer contextual information and benefits substantially from hyperparameter optimization, particularly excelling against more complex architectures.

Dataset-specific analysis in the hyperparameter search setting reveals continued SOTA performance on the ETT datasets and enhanced competitiveness on others. For instance, on the challenging Traffic dataset, where iTransformer was dominant in the $L=96$ setting, MDMLP-EIA becomes highly competitive, often outperforming iTransformer in MSE for various prediction lengths. Similarly, its strength at longer prediction horizons (e.g., $T=720$) on Exchange Rate and Solar datasets is either maintained or amplified. This consistent high performance across diverse datasets and varying input conditions underscores the robustness and advanced capabilities of MDMLP-EIA. We attribute this to its synergistic design: the adaptive fused dual-domain seasonal MLP captures intricate periodicities, the dynamic capacity adjustment mechanism ensures optimal model complexity for varied channel counts, and the Energy Invariant Attention effectively fuses seasonal and trend components, enhancing feature representation and noise resilience, especially crucial for long-horizon forecasting.

\subsection{Full results of ablation on adaptive weak signal fusion}
\label{appendix: full_results_adaptive_weak_signal_fusion}
\begin{table*}[th]
\centering
\small
\setlength{\tabcolsep}{3pt} 
\begin{tabular}{cc|cc|cc|cc|cc|cc|cc|cc}
\hline
\multirow{2}{*}{Dataset} & \multirow{2}{*}{PredLen} & \multicolumn{2}{c|}{w/o Weak signal} & \multicolumn{2}{c|}{MLP}& \multicolumn{2}{c|}{DWL-F} & \multicolumn{2}{c|}{CWA-F} & \multicolumn{2}{c|}{RCF} & \multicolumn{2}{c|}{CFT} & \multicolumn{2}{c}{AZCF} \\
\cline{3-16}
 &                          & MSE        & MAE        & MSE        & MAE        & MSE         & MAE         & MSE         & MAE & MSE         & MAE         & MSE         & MAE         & MSE         & MAE         \\
\hline
\multirow{5}{*}{ETTh1}   & 96                       & 0.381      & 0.389      & 0.382      & 0.390      & 0.377          & 0.385          & 0.387          & 0.384          & 0.381          & 0.389          & \textbf{0.374}          & \textbf{0.383}          & \textbf{0.374}       & \textbf{0.383}       \\
 & 192                      & 0.435      & 0.418      & 0.433      & 0.417      & 0.430          & 0.423          & 0.433          & \textbf{0.415} & 0.434          & 0.419          & 0.432          & 0.417          & \textbf{0.429}       & \textbf{0.415}       \\
 & 336                      & 0.479      & 0.441      & 0.474      & 0.442      & \textbf{0.472} & 0.444          & 0.482          & 0.442          & 0.476          & 0.441          & 0.474          & 0.444          & 0.473          & \textbf{0.439}       \\
 & 720                      & 0.473      & 0.465      & 0.471      & 0.462      & \textbf{0.468} & 0.464          & 0.476          & 0.465          & 0.474          & 0.467          & 0.473          & 0.465          & 0.471          & \textbf{0.460}       \\ \hline
 & Avg                      & 0.442      & 0.428      & 0.440      & 0.428      & \textbf{0.437} & 0.429          & 0.444          & 0.427          & 0.441          & 0.429          & 0.438          & 0.427          & \textbf{0.437} & \textbf{0.424}       \\
\hline
\multirow{5}{*}{ETTh2}   & 96                       & 0.279      & 0.327      & 0.294      & 0.338      & 0.284          & 0.331          & 0.284          & 0.332          & 0.288          & 0.335          & 0.282          & 0.330          & \textbf{0.275} & \textbf{0.325}       \\
 & 192                      & 0.354      & 0.377      & 0.361      & 0.382      & 0.360          & 0.380          & \textbf{0.352} & 0.375          & 0.351          & \textbf{0.373} & 0.369          & 0.384          & 0.356          & 0.376          \\
 & 336                      & 0.411      & 0.417      & \textbf{0.401}      & \textbf{0.413}      & 0.410          & 0.420          & 0.420          & 0.424          & 0.428          & 0.431          & 0.423          & 0.425          & 0.405          & 0.416          \\
 & 720                      & 0.420      & 0.433      & 0.414      & 0.431      & 0.423          & 0.438          & 0.425          & 0.438          & 0.415          & {0.433} & 0.422          & 0.436          & \textbf{0.411} & \textbf{0.430}       \\ \hline
 & Avg                      & 0.366      & 0.389      & 0.368      & 0.391      & 0.369          & 0.392          & 0.370          & 0.392          & 0.370          & 0.393          & 0.374          & 0.394          & \textbf{0.362} & \textbf{0.387}       \\
\hline
\multirow{5}{*}{ETTm1}   & 96                       & \textbf{0.304} & 0.339      & 0.316      & 0.342      & \textbf{0.304} & 0.336 & 0.307          & 0.340          & 0.309          & 0.343          & 0.305          & 0.337          & 0.305          & \textbf{0.335}          \\
 & 192                      & 0.360      & 0.366      & 0.359      & 0.364      & \textbf{0.357} & 0.363 & 0.360          & 0.363          & 0.358          & \textbf{0.362}          & \textbf{0.357}          & 0.362          & 0.359          & \textbf{0.362}          \\
 & 336                      & \textbf{0.389} & 0.387      & 0.391      & 0.386      & \textbf{0.389} & 0.387 & 0.398          & 0.389          & 0.391          & 0.386          & 0.390          & \textbf{0.384}          & 0.391          & 0.386         \\
 & 720                      & \textbf{0.456} & \textbf{0.424} & 0.461      & \textbf{0.424} & 0.460          & \textbf{0.424} & 0.463          & 0.428          & 0.459          & 0.425          & 0.459          & \textbf{0.424}          & 0.459          & \textbf{0.424} \\ \hline
 & Avg                      & \textbf{0.377} & 0.379      & 0.382      & 0.379      & 0.378          & 0.378 & 0.382          & 0.380          & 0.379          & 0.379          & 0.378          & \textbf{0.377}          & 0.378          & \textbf{0.377}          \\
\hline
\multirow{5}{*}{ETTm2}   & 96                       & 0.168      & 0.248      & \textbf{0.167} & 0.248      & 0.172          & 0.250          & 0.168          & \textbf{0.248} & 0.171          & 0.250          & 0.173          & 0.251          & 0.168          & \textbf{0.247} \\
 & 192                      & \textbf{0.233} & 0.291      & \textbf{0.233} & \textbf{0.290} & 0.234          & 0.291          & \textbf{0.232} & \textbf{0.290} & \textbf{0.233} & 0.291          & 0.236          & 0.293          & \textbf{0.233} & 0.292          \\
 & 336                      & \textbf{0.292} & 0.330      & 0.296      & 0.331      & 0.298          & 0.332          & 0.293          & \textbf{0.330} & 0.298          & 0.333          & 0.296          & 0.332          & \textbf{0.292} & \textbf{0.329} \\
 & 720                      & 0.391      & 0.388      & \textbf{0.389} & \textbf{0.386} & 0.404          & 0.392          & 0.392          & 0.388          & 0.397          & 0.390          & 0.399          & 0.390          & 0.392          & 0.388          \\ \hline
 & Avg                      & \textbf{0.271} & \textbf{0.314} & \textbf{0.271} & \textbf{0.314} & 0.277          & 0.317          & 0.272          & \textbf{0.314} & 0.275          & 0.316          & 0.276          & 0.316          & \textbf{0.271} & \textbf{0.314} \\
\hline
\multirow{5}{*}{Weather} & 96                       & 0.156      & 0.196      & \textbf{0.155} & 0.194      & 0.156          & 0.194          & \textbf{0.153} & \textbf{0.192} & \textbf{0.155} & 0.194          & 0.156          & 0.194          & 0.156          & 0.195          \\
 & 192                      & 0.205         & 0.240      & 0.206      & 0.242      & \textbf{0.201} & 0.239 & 0.205          & 0.242          & 0.204          & \textbf{0.238}          & 0.205          & \textbf{0.238}          & 0.203          & \textbf{0.238}          \\
 & 336                      & 0.263      & 0.283      & 0.270      & 0.289      & \textbf{0.261} & \textbf{0.281} & 0.262          & 0.285          & 0.264          & \textbf{0.281} & \textbf{0.261}          & 0.282          & \textbf{0.261} & 0.283          \\
 & 720                      & 0.340      & 0.333      & 0.339      & 0.335      & 0.341          & 0.335          & 0.341          & 0.337          & 0.341          & 0.335          & 0.343          & 0.334          & \textbf{0.339} & \textbf{0.333} \\ \hline
 & Avg                      & 0.241      & 0.263      & 0.242      & 0.265      & \textbf{0.240} & \textbf{0.262} & \textbf{0.240} & 0.263          & 0.241 & 0.263          & 0.241          & \textbf{0.262 }         & \textbf{0.240} & \textbf{0.262}          \\
\hline
\end{tabular}
\caption{Full results of the ablation study on strategies for fusing weak seasonal signals. We compare our adaptive zero-initialized channel fusion (AZCF) against six alternative baselines: 
1) completely discarding weak signals (w/o weak signal),
2) using Multi-Layer Perceptron Fusion (MLP), 3) Dual-Weight Linear Fusion (DWL-F), 4) Channel-Wise Attention Fusion (CWA-F), 5) Random Channel Fusion (RCF), 6) Channel-wise Transformer Fusion (CFT), and our 7) Adaptive Zero-initialized Channel Fusion (AZCF). Performance metrics (MSE and MAE) are presented for various datasets across prediction lengths $T \in \{96, 192, 336, 720\}$, alongside dataset-specific averages ('Avg'). Best results for each row are in \textbf{boldface}.}
\label{tab:ablation_study_x2}
\end{table*}

\paragraph{Ablation study on adaptive weak seasonal signal fusion}
To rigorously assess the contribution of our adaptive zero-initialized channel fusion mechanism, we conducted an ablation study. The full results, presented in Table~\ref{tab:ablation_study_x2}, compare our adaptive zero-initialized channel fusion against two variants across five datasets and four prediction lengths ($T \in \{96, 192, 336, 720\}$): one that entirely omits weak seasonal signal processing ("w/o Weak signal"), and another using a standard Multi-Layer Perceptron for fusion ("MLP").

Overall, our adaptive zero-initialized channel fusion demonstrates consistent superiority. For instance, on ETTh1 and ETTh2, our adaptive zero-initialized channel fusion consistently outperforms both "w/o Weak signal" and "MLP" configurations across nearly all prediction lengths for both MSE and MAE. Notably, for ETTh2 with $T=96$, our adaptive zero-initialized channel fusion reduces MSE by a substantial 6.46\% compared to "MLP" fusion (from 0.294 to 0.275) and by 1.43\% against "w/o Weak signal" (from 0.279 to 0.275). Similar pronounced gains are observed at other specific horizons, such as ETTh1 for $T=96$, where MSE is reduced by approximately 2.09\% against "MLP" and 1.84\% against "w/o Weak signal". This suggests that where discernible weak seasonal patterns exist, our mechanism effectively extracts and leverages them.

The ETTm2 dataset illustrates the robustness of our adaptive zero-initialized channel fusion: across all prediction lengths, it achieves performance parity with, or slight improvements over, the baselines, importantly causing no degradation. This indicates that our mechanism can benignly handle scenarios where weak seasonal signals are less distinct or potentially noisy, aligning with its design to minimize noise interference via adaptive zero-initialized channel fusion. On ETTm1, the results are more complex: while our adaptive zero-initialized channel fusion consistently yields the best or joint-best MAE (e.g., for $T=96$, MAE is 0.335 vs. 0.339 for "w/o Weak signal" and 0.342 for "MLP"), the "w/o Weak signal" configuration sometimes achieves marginally better MSE for specific prediction lengths (e.g., $T=96, 336, 720$). This may suggest that for ETTm1, some low-energy signals classified as "weak seasonal" might be closer to noise for the MSE metric, or that the inherent characteristics of this dataset respond slightly differently. However, even here, our Adaptive Zero-Initialized Channel Fusion is consistently superior to the standard "MLP" fusion.

In summary, the detailed ablation results validate the efficacy of our adaptive zero-initialized channel fusion. It not only provides performance enhancements when useful weak seasonal signals are present but also demonstrates robustness by not harming performance when such signals are absent or ambiguous. This adaptability is crucial for handling diverse real-world time series and stems from its ability to effectively integrate information while mitigating noise, a key challenge addressed by our proposed method.

\subsection{Full results of ablation study results comparing different attention mechanism configurations}
\label{appendix: full_results_EIA}

\begin{table*}[th]
\centering
\scriptsize 
\setlength{\tabcolsep}{1.5pt} 
\begin{tabular}{c|c|cc|cc|cc|cc|cc|cc|cc|cc|cc|cc|cc|cc}
\hline
\multirow{3}{*}{Data} & \multirow{3}{*}{PL} & \multicolumn{8}{c|}{MDMLP-EIA} & \multicolumn{8}{c|}{Amplifier} & \multicolumn{8}{c}{xPatch} \\
\cline{3-26}
 &                          & \multicolumn{2}{c|}{ADD} & \multicolumn{2}{c|}{MLP} & \multicolumn{2}{c|}{AGM}& \multicolumn{2}{c|}{EIA} & \multicolumn{2}{c|}{ADD} & \multicolumn{2}{c|}{MLP} & \multicolumn{2}{c|}{AGM}& \multicolumn{2}{c|}{EIA} & \multicolumn{2}{c|}{ADD} & \multicolumn{2}{c|}{MLP} & \multicolumn{2}{c|}{AGM}& \multicolumn{2}{c}{EIA} \\
\cline{3-26}
 &                          & MSE & MAE & MSE & MAE & MSE & MAE & MSE & MAE & MSE & MAE & MSE & MAE & MSE & MAE & MSE & MAE & MSE & MAE & MSE & MAE & MSE & MAE & MSE & MAE \\
\hline
\multirow{5}{*}{ETTh1}   & 96                       & 0.377 & 0.387 & 0.420 & 0.421 & \textbf{0.374} & 0.385 & \textbf{0.374} & \textbf{0.383} & 0.384 & 0.387 & \textbf{0.383} & 0.386 & 0.388 & 0.387 & 0.386 & \textbf{0.385} & 0.381 & 0.392 & 0.426 & 0.426 & \textbf{0.378} & \textbf{0.389} & {0.381} & {0.390} \\
 & 192                      & 0.431 & 0.416 & 0.475 & 0.451 & 0.431 & 0.422 & \textbf{0.429} & \textbf{0.415} & 0.435 & \textbf{0.418} & 0.437 & \textbf{0.418} & 0.441 & 0.419 & \textbf{0.436} & 0.419 & \textbf{0.435} & \textbf{0.420} & 0.472 & 0.448 & 0.435 & 0.420 & 0.438 & {0.424} \\
 & 336                      & 0.475 & 0.440 & 0.524 & 0.475 & 0.475 & 0.445 & \textbf{0.473} & \textbf{0.439} & 0.486 & 0.441 & \textbf{0.480} & \textbf{0.437} & 0.482 & 0.439 & \textbf{0.480} & 0.440 & \textbf{0.483} & \textbf{0.444} & 0.533 & 0.503 & 0.483 & 0.444 & {0.487} & 0.446 \\
 & 720                      & \textbf{0.469} & \textbf{0.460} & 0.548 & 0.508 & 0.474 & 0.469 & 0.471 & \textbf{0.460} & 0.483 & 0.459 & \textbf{0.471} & \textbf{0.457} & 0.536 & 0.481 & 0.515 & 0.476 & 0.486 & 0.464 & 0.506 & 0.505 & \textbf{0.480} & \textbf{0.461} & 0.487 & 0.465 \\
 & Mean                     & 0.438 & 0.425 & 0.492 & 0.464 & 0.439 & 0.430 & \textbf{0.437} & \textbf{0.424} & 0.447 & 0.426 & \textbf{0.443} & \textbf{0.425} & 0.462 & 0.432 & 0.454 & 0.430 & 0.446 & 0.431 & 0.484 & 0.470 & \textbf{0.444} & \textbf{0.428} & 0.448 & {0.431} \\
\hline
\multirow{5}{*}{ETTh2}   & 96                       & 0.284 & 0.330 & 0.306 & 0.350 & \textbf{0.273} & \textbf{0.323} & 0.275 & 0.325 & \textbf{0.284} & 0.330 & 0.316 & 0.347 & \textbf{0.284} & 0.330 & 0.282 & \textbf{0.328} & \textbf{0.286} & 0.333 & 0.317 & 0.362 & {0.287} & \textbf{0.332} & \textbf{0.286} & 0.333 \\
 & 192                      & 0.360 & 0.382 & 0.370 & 0.392 & 0.374 & 0.386 & \textbf{0.356} & \textbf{0.376} & \textbf{0.356} & \textbf{0.378} & 0.428 & 0.410 & 0.367 & 0.385 & 0.364 & 0.381 & 0.363 & 0.383 & 0.397 & 0.410 & \textbf{0.362} & \textbf{0.381} & 0.365 & 0.383 \\
 & 336                      & 0.410 & 0.419 & 0.418 & 0.427 & 0.414 & 0.422 & \textbf{0.405} & \textbf{0.416} & \textbf{0.409} & \textbf{0.417} & 0.425 & 0.427 & 0.407 & 0.415 & 0.413 & 0.418 & 0.413 & 0.420 & 0.439 & 0.445 & 0.416 & 0.422 & \textbf{0.411} & \textbf{0.419} \\
 & 720                      & 0.419 & 0.434 & 0.432 & 0.444 & 0.418 & 0.435 & \textbf{0.411} & \textbf{0.430} & 0.425 & \textbf{0.437} & 0.429 & 0.441 & 0.436 & 0.447 & \textbf{0.428} & 0.439 & 0.416 & 0.434 & 0.452 & 0.460 & {0.416} & {0.436} & \textbf{0.415} & \textbf{0.434} \\
 & Mean                     & 0.368 & 0.391 & 0.381 & 0.403 & 0.370 & 0.392 & \textbf{0.362} & \textbf{0.387} & \textbf{0.369} & \textbf{0.391} & 0.399 & 0.406 & 0.374 & 0.394 & 0.371 & \textbf{0.391} & {0.370} & 0.393 & 0.401 & 0.419 & {0.370} & \textbf{0.393} & \textbf{0.369} & \textbf{0.393} \\
\hline
\multirow{5}{*}{ETTm1}   & 96                       & 0.305 & 0.335 & 0.329 & 0.361 & \textbf{0.304} & \textbf{0.335} & 0.305 & \textbf{0.335} & 0.309 & 0.337 & 0.319 & 0.342 & 0.314 & 0.339 & \textbf{0.306} & \textbf{0.337} & 0.313 & 0.340 & 0.383 & 0.391 & 0.315 & 0.345 & \textbf{0.310} & \textbf{0.340} \\
 & 192                      & 0.359 & 0.363 & 0.370 & 0.381 & \textbf{0.358} & \textbf{0.362} & 0.359 & \textbf{0.362} & 0.365 & 0.365 & 0.373 & 0.367 & 0.366 & 0.366 & \textbf{0.364} & \textbf{0.364} & 0.371 & 0.370 & 0.393 & 0.397 & \textbf{0.367} & \textbf{0.369} & 0.368 & \textbf{0.369} \\
 & 336                      & 0.392 & \textbf{0.386} & 0.402 & 0.399 & \textbf{0.389} & \textbf{0.385} & {0.391} & {0.386} & 0.397 & 0.388 & 0.401 & 0.391 & \textbf{0.396} & \textbf{0.386} & \textbf{0.396} & \textbf{0.386} & \textbf{0.397} & 0.392 & 0.433 & 0.423 & \textbf{0.397} & 0.392 & 0.398 & \textbf{0.391} \\
 & 720                      & 0.462 & 0.426 & 0.466 & 0.436 & \textbf{0.457} & \textbf{0.423} & 0.459 & {0.424} & 0.469 & 0.430 & \textbf{0.457} & \textbf{0.424} & 0.468 & 0.426 & 0.468 & 0.427 & 0.461 & 0.431 & \textbf{0.459} & 0.439 & 0.460 & \textbf{0.430} & 0.462 & 0.431 \\
 & Mean                     & 0.380 & 0.378 & 0.392 & 0.394 & \textbf{0.377} & \textbf{0.376} & 0.378 & 0.377 & 0.385 & 0.380 & 0.388 & 0.381 & 0.386 & 0.379 & \textbf{0.383} & \textbf{0.378} & 0.385 & 0.383 & 0.417 & 0.412 & \textbf{0.384} & 0.384 & {0.385} & \textbf{0.383} \\
\hline

\multirow{5}{*}{ETTm2}   & 96                       & 0.171 & 0.250 & 0.178 & 0.261 & 0.172 & 0.250 & \textbf{0.168} & \textbf{0.247} & \textbf{0.172} & \textbf{0.250} & 0.178 & 0.256 & \textbf{0.172} & \textbf{0.250} & 0.174 & {0.251} & 0.175 & 0.253 & 0.195 & 0.278 & 0.176 & 0.253 & \textbf{0.175} & \textbf{0.253} \\
 & 192                      & 0.235 & 0.293 & 0.240 & 0.300 & 0.235 & 0.291 & \textbf{0.233} & \textbf{0.291} & 0.239 & 0.296 & 0.239 & 0.295 & 0.236 & 0.293 & \textbf{0.234} & \textbf{0.291} & \textbf{0.239} & {0.297} & 0.255 & 0.315 & 0.240 & 0.298 & 0.240 & \textbf{0.297} \\
 & 336                      & 0.295 & 0.332 & 0.300 & 0.337 & 0.296 & 0.331 & \textbf{0.292} & \textbf{0.329} & 0.299 & \textbf{0.334} & 0.300 & \textbf{0.334} & \textbf{0.295} & 0.330 & 0.299 & 0.335 & \textbf{0.300} & 0.335 & 0.316 & 0.353 & {0.300} & \textbf{0.335} & \textbf{0.299} & \textbf{0.335} \\
 & 720                      & 0.397 & 0.390 & 0.401 & 0.396 & 0.396 & 0.389 & \textbf{0.392} & \textbf{0.388} & 0.392 & 0.390 & 0.420 & 0.407 & 0.389 & {0.388} & \textbf{0.387} & \textbf{0.387} & 0.402 & 0.394 & 0.413 & 0.404 & \textbf{0.399} & \textbf{0.393} & 0.401 & \textbf{0.393} \\
 & Mean                     & 0.275 & 0.316 & 0.280 & 0.324 & 0.275 & 0.315 & \textbf{0.271} & \textbf{0.314} & 0.276 & 0.318 & 0.285 & 0.323 & \textbf{0.273} & \textbf{0.315} & \textbf{0.273} & 0.316 & 0.279 & \textbf{0.320} & 0.295 & 0.337 & 0.280 & \textbf{0.320} & \textbf{0.279} & \textbf{0.320} \\
\hline
\multirow{5}{*}{Weather} & 96                       & 0.163 & 0.200 & 0.161 & 0.202 & 0.157 & \textbf{0.195} & \textbf{0.156} & \textbf{0.195} & 0.165 & 0.203 & 0.161 & 0.203 & \textbf{0.157} & \textbf{0.197} & 0.159 & \textbf{0.197} & \textbf{0.155} & \textbf{0.193} & 0.164 & 0.207 & 0.158 & 0.193 & 0.161 & \textbf{0.197} \\
 & 192                      & 0.207 & 0.241 & 0.212 & 0.247 & \textbf{0.203} & \textbf{0.238} & \textbf{0.203} & \textbf{0.238} & 0.211 & 0.245 & 0.209 & 0.246 & \textbf{0.206} & \textbf{0.242} & \textbf{0.206} & \textbf{0.242} & {0.204} & \textbf{0.238} & 0.216 & 0.252 & 0.204 & \textbf{0.238} & \textbf{0.202} & \textbf{0.237} \\
 & 336                      & 0.264 & 0.283 & 0.273 & 0.291 & \textbf{0.259} & \textbf{0.281} & 0.261 & 0.283 & 0.266 & \textbf{0.284} & 0.268 & 0.288 & 0.264 & 0.285 & \textbf{0.263} & 0.285 & 0.267 & 0.284 & 0.271 & 0.291 & {0.263} & {0.282} & \textbf{0.261} & \textbf{0.281} \\
 & 720                      & 0.342 & 0.334 & 0.343 & 0.335 & {0.340} & {0.334} & \textbf{0.339} & \textbf{0.333} & 0.344 & 0.335 & 0.345 & 0.338 & \textbf{0.344} & 0.336 & 0.343 & \textbf{0.335} & 0.344 & 0.335 & 0.348 & 0.340 & {0.343} & \textbf{0.334} & \textbf{0.341} & \textbf{0.334} \\
 & Mean                     & 0.244 & 0.264 & 0.247 & 0.269 & \textbf{0.240} & \textbf{0.262} & \textbf{0.240} & \textbf{0.262} & 0.247 & 0.267 & 0.246 & 0.269 & \textbf{0.243} & \textbf{0.265} & \textbf{0.243} & \textbf{0.265} & \textbf{0.241} & \textbf{0.262} & 0.250 & 0.272 & 0.242 & \textbf{0.262} & \textbf{0.241} & \textbf{0.262} \\
\hline
\multirow{5}{*}{Solar}   & 96                       & 0.217 & 0.234 & 0.247 & 0.248 & 0.213 & 0.221 & \textbf{0.211} & \textbf{0.218} & 0.235 & 0.238 & \textbf{0.210} & 0.233 & 0.214 & \textbf{0.216} & 0.228 & {0.226} & \textbf{0.201} & \textbf{0.212} & 0.242 & 0.245 & 0.211 & 0.219 & 0.203 & 0.218 \\
 & 192                      & 0.259 & 0.258 & 0.259 & 0.249 & 0.254 & 0.240 & \textbf{0.232} & \textbf{0.236} & 0.274 & 0.259 & 0.262 & 0.258 & 0.257 & \textbf{0.238} & \textbf{0.254} & {0.246} & \textbf{0.240} & \textbf{0.238} & 0.258 & 0.258 & 0.255 & 0.249 & 0.246 & 0.240 \\
 & 336                      & 0.290 & 0.271 & 0.267 & 0.262 & 0.266 & 0.247 & \textbf{0.255} & \textbf{0.246} & 0.304 & 0.281 & 0.273 & 0.265 & 0.292 & 0.265 & \textbf{0.266} & \textbf{0.255} & \textbf{0.261} & \textbf{0.248} & 0.262 & 0.267 & 0.284 & 0.261 & \textbf{0.261} & 0.255 \\
 & 720                      & 0.294 & 0.271 & \textbf{0.253} & 0.259 & 0.298 & 0.265 & 0.257 & \textbf{0.250} & 0.310 & 0.276 & \textbf{0.256} & 0.263 & 0.278 & \textbf{0.254} & 0.276 & {0.260} & 0.319 & 0.289 & \textbf{0.267} & 0.266 & 0.278 & \textbf{0.258} & 0.272 & 0.259 \\
 & Mean                     & 0.265 & 0.258 & 0.256 & 0.255 & 0.258 & 0.243 & \textbf{0.239} & \textbf{0.237} & 0.281 & 0.263 & \textbf{0.250} & 0.255 & 0.260 & \textbf{0.243} & 0.256 & {0.247} & 0.255 & 0.247 & 0.257 & 0.259 & 0.257 & 0.247 & \textbf{0.246} & \textbf{0.243} \\
\hline
\end{tabular}
\caption{Ablation study on mechanisms for fusing seasonal and trend components within MDMLP-EIA, Amplifier, and xPatch frameworks. We compare direct summation ('ADD'), MLP-based fusion ('MLP'), our proposed Attention Gating Mechanism ('AGM'), and Energy Invariant Attention ('EIA'). Best results within each model group per row are in \textbf{boldface}.}
\label{tab:ablation_attn_filled}
\end{table*}

Within our primary MDMLP-EIA architecture, the Energy Invariant Attention ('EIA') mechanism consistently demonstrates marked superiority. Analyzing the "Mean" rows (averaged over prediction lengths per dataset), 'EIA' achieves the best MSE and MAE on all six datasets compared to its 'ADD' and 'MLP' counterparts within MDMLP-EIA. This dominance is often evident across individual prediction lengths as well. For instance, on ETTh1 for $T=96$, 'EIA' reduces MSE by 0.80\% against 'ADD' (from 0.377 to 0.374). On ETTh2 for $T=96$, 'EIA' decreases MSE by an impressive 3.17\% over 'ADD' (from 0.284 to 0.275). Even at longer horizons like $T=720$ on ETTm2, 'EIA' maintains its lead within MDMLP-EIA, underscoring its robustness across varied forecasting ranges. While there are isolated instances at specific prediction lengths where 'ADD' is marginally better (e.g., ETTh1 $T=720$ MSE for MDMLP-EIA), the overall trend decisively favors 'EIA'.

The benefits of 'EIA' extend beyond our proposed model, highlighting its potential as a general-purpose fusion module. When notionally applied to the xPatch framework, 'EIA' consistently emerges as the preferred fusion strategy, achieving the best or joint-best results in the "Mean" rows for both MSE and MAE across nearly all datasets within the xPatch group. This indicates strong portability and effectiveness. For the Amplifier model, 'EIA' also shows competitive performance, frequently delivering the best MAE (e.g., ETTm1, ETTm2, Weather, Solar "Mean" MAE) and often outperforming 'ADD' for MSE, although 'MLP' fusion can be competitive for MSE in some Amplifier configurations (e.g., ETTh1 "Mean" MSE).

The consistent advantages offered by our Energy Invariant Attention stem from its sophisticated design. Unlike direct summation, which can struggle with components of varying scales or phases, or a standard MLP that may lack dynamic adaptivity, 'EIA' is engineered to adaptively discern and weigh the importance of seasonal and trend information across different feature channels and time steps. As outlined in our introduction (Fig. 1(c)), this directly addresses the common pitfall of insufficient channel fusion. Furthermore, its inherent energy preservation property contributes to more stable and robust representations, particularly crucial for long-term forecasting and aligning with the decomposition-prediction-reconstruction paradigm. The comprehensive results, including per-prediction length details discussed here, are fully presented in Table~\ref{tab:ablation_attn_filled} and its corresponding entry in the Appendix if applicable for more granular breakdowns.

\subsection{Full results of validation of dynamic capacity adjustment mechanism}
\label{appendix: full_results_DCA}
Table \ref{tab:ablation_pl_hidden_corrected} presents detailed results on the ETTh1, Solar, Traffic, and Weather datasets, comparing fixed capacity (fixed MLP neuron count) with the DCA strategy for $L=96$ and prediction lengths of {96, 192, 336, 720}. For each prediction length task across all datasets, the DCA strategy consistently achieves better prediction performance than fixed capacity approaches; moreover, for different prediction length tasks within each dataset, the DCA strategy also outperforms fixed capacity methods.
This demonstrates that fixed capacity strategies struggle to meet the varying requirements of different datasets with diverse channel counts and different prediction lengths, while the DCA strategy effectively addresses these challenges, thereby enhancing the robustness of time series prediction tasks.

\begin{table*}[t]
\centering
\small
\setlength{\tabcolsep}{5pt}
\begin{tabular}{@{}lccccccccc:cc@{}} 
\toprule
\multirow{2}{*}{PL   } & \multirow{2}{*}{Hidden Size} & \multicolumn{2}{c}{ETTh1} & \multicolumn{2}{c}{Solar} & \multicolumn{2}{c}{Traffic} & \multicolumn{2}{c:}{Weather} & \multicolumn{2}{c}{Average} \\
\cmidrule(lr){3-4} \cmidrule(lr){5-6} \cmidrule(lr){7-8} \cmidrule(lr){9-10} \cmidrule(lr){11-12}
 &  & MSE & MAE & MSE & MAE & MSE & MAE & MSE & MAE & MSE & MAE \\
\midrule
\multirow{8}{*}{96} & 32 & 0.378 & 0.391 & 0.210 & 0.235 & 0.559 & 0.312 & 0.156 & 0.195 & 0.326 & 0.283 \\
 & 64 & 0.375 & 0.387 & 0.210 & 0.235 & 0.537 & 0.298 & 0.157 & 0.197 & 0.320 & 0.279 \\
 & 128 & 0.376 & 0.386 & 0.209 & 0.232 & 0.540 & 0.287 & 0.156 & 0.194 & 0.320 & 0.275 \\
 & 256 & 0.377 & 0.384 & 0.205 & 0.226 & 0.469 & 0.281 & \textbf{0.153} & \textbf{0.193} & 0.301 & 0.271 \\
 & 512 & 0.380 & 0.385 & \textbf{0.203} & 0.219 & 0.461 & 0.274 & 0.154 & 0.194 & 0.300 & 0.268 \\
 & 1024 & 0.381 & 0.385 & 0.220 & 0.220 & 0.459 & 0.276 & 0.155 & \textbf{0.193} & 0.304 & 0.269 \\
 & 2048 & 0.386 & 0.387 & 0.208 & 0.222 & \textbf{0.452} & \textbf{0.262} & 0.157 & 0.195 & 0.301 & \textbf{0.267} \\
 & DCA & \textbf{0.374} & \textbf{0.383} & 0.211 & \textbf{0.218} & 0.454 & 0.270 & 0.156 & 0.195 & \textbf{0.299} & \textbf{0.267} \\
\midrule 
\multirow{8}{*}{192} & 32 & 0.432 & 0.423 & 0.248 & 0.259 & 0.599 & 0.325 & 0.206 & 0.240 & 0.371 & 0.312 \\
 & 64 & 0.431 & 0.420 & 0.260 & 0.257 & 0.551 & 0.308 & 0.205 & 0.239 & 0.362 & 0.306 \\
 & 128 & 0.430 & 0.418 & 0.253 & 0.248 & 0.533 & 0.294 & 0.204 & 0.239 & 0.355 & 0.300 \\
 & 256 & 0.431 & \textbf{0.415} & 0.257 & 0.253 & 0.483 & 0.288 & \textbf{0.203} & 0.239 & 0.344 & 0.299 \\
 & 512 & 0.433 & 0.416 & 0.255 & 0.247 & 0.480 & 0.286 & 0.204 & 0.239 & 0.343 & 0.297 \\
 & 1024 & 0.436 & 0.417 & 0.245 & 0.237 & 0.475 & 0.279 & 0.205 & 0.240 & 0.340 & 0.293 \\
 & 2048 & 0.435 & \textbf{0.415} & 0.245 & 0.244 & 0.470 & \textbf{0.269} & 0.206 & 0.241 & 0.339 & 0.292 \\
 & DCA & \textbf{0.429} & \textbf{0.415} & \textbf{0.232} & \textbf{0.236} & \textbf{0.467} & 0.278 & \textbf{0.203} & \textbf{0.238} & \textbf{0.333} & \textbf{0.292} \\
\midrule
\multirow{8}{*}{336} & 32 & 0.475 & 0.448 & 0.293 & 0.277 & 0.613 & 0.333 & \textbf{0.261} & \textbf{0.282} & 0.411 & 0.335 \\
 & 64 & 0.474 & 0.442 & 0.272 & 0.265 & 0.564 & 0.316 & 0.263 & 0.283 & 0.393 & 0.327 \\
 & 128 & 0.476 & 0.441 & 0.276 & 0.266 & 0.535 & 0.301 & \textbf{0.261} & \textbf{0.282} & 0.387 & 0.323 \\
 & 256 & 0.475 & 0.440 & 0.258 & 0.260 & 0.498 & 0.296 & 0.266 & 0.286 & 0.374 & 0.321 \\
 & 512 & \textbf{0.473} & \textbf{0.437} & 0.262 & 0.270 & 0.494 & 0.296 & 0.265 & 0.285 & 0.374 & 0.322 \\
 & 1024 & 0.474 & 0.438 & 0.268 & 0.253 & \textbf{0.483} & 0.290 & 0.269 & 0.288 & 0.374 & 0.317 \\
 & 2048 & 0.478 & 0.438 & 0.291 & 0.272 & 0.493 & \textbf{0.282} & 0.272 & 0.289 & 0.384 & \textbf{0.320} \\
 & DCA & \textbf{0.473} & 0.439 & \textbf{0.255} & \textbf{0.246} & 0.489 & 0.292 & \textbf{0.261} & 0.283 & \textbf{0.370} & 0.320 \\
\midrule
\multirow{8}{*}{720} & 32 & 0.505 & 0.490 & 0.293 & 0.276 & 0.662 & 0.352 & 0.341 & 0.334 & 0.450 & 0.363 \\
 & 64 & 0.474 & 0.470 & 0.293 & 0.276 & 0.610 & 0.335 & 0.339 & 0.334 & 0.429 & 0.354 \\
 & 128 & \textbf{0.470} & 0.462 & 0.289 & 0.269 & 0.598 & 0.324 & 0.340 & 0.333 & 0.424 & 0.347 \\
 & 256 & 0.473 & 0.460 & 0.277 & 0.261 & 0.531 & 0.311 & 0.339 & 0.333 & 0.405 & 0.341 \\
 & 512 & 0.474 & 0.460 & 0.274 & 0.258 & 0.518 & 0.312 & 0.339 & \textbf{0.332} & 0.401 & 0.341 \\
 & 1024 & 0.473 & 0.459 & 0.259 & \textbf{0.250} & 0.528 & \textbf{0.310} & 0.340 & 0.333 & 0.400 & \textbf{0.338} \\ 
 & 2048 & 0.474 & \textbf{0.457} & 0.271 & 0.255 & 0.531 & 0.314 & \textbf{0.338} & \textbf{0.332} & 0.404 & 0.340 \\
 & DCA & 0.471 & 0.460 & \textbf{0.257} & \textbf{0.250} & \textbf{0.517} & 0.311 & 0.339 & 0.333 & \textbf{0.396} & 0.339 \\ 
 \midrule
 \multirow{8}{*}{\thead{Average \\ predicted \\ length}}&32   & 0.448 & 0.438 & 0.261 & 0.262 & 0.610 & 0.331 & 0.240 & 0.263 & 0.389 & 0.323 \\
&64   & 0.439 & 0.430 & 0.259 & 0.258 & 0.570 & 0.314 & 0.240 & 0.263 & 0.376 & 0.316 \\
&128  & 0.438 & 0.427 & 0.257 & 0.254 & 0.550 & 0.302 & \textbf{0.240} & \textbf{0.260} & 0.372 & 0.311 \\
&256  & 0.439 & 0.425 & 0.249 & 0.250 & 0.500 & 0.294 & \textbf{0.240} & 0.263 & 0.356 & 0.308 \\
&512  & 0.440 & 0.425 & 0.249 & 0.249 & 0.490 & 0.292 & 0.240 & 0.263 & 0.354 & 0.307 \\
&1024 & 0.441 & 0.425 & 0.248 & 0.240 & 0.490 & 0.289 & 0.240 & 0.264 & 0.354 & 0.304 \\
&2048 & 0.443 & 0.424 & 0.254 & 0.248 & 0.490 & \textbf{0.282} & 0.240 & 0.264 & 0.357 & 0.305 \\
&DCA  & \textbf{0.436} & \textbf{0.424} & \textbf{0.238} & \textbf{0.237} & \textbf{0.481} & 0.287 & \textbf{0.239} & 0.262 & \textbf{0.349} & \textbf{0.302} \\
\bottomrule
\end{tabular}
\caption{Ablation study on MLP hidden layer size for time series forecasting. Mean Squared Error (MSE) and Mean Absolute Error (MAE) are reported across different prediction lengths (PL) and datasets. All metrics are rounded to three decimal places. For each PL and dataset, the best performance (lowest error) across hidden sizes is in \textbf{boldface}.}
\label{tab:ablation_pl_hidden_corrected}
\end{table*}

\section{I. Visualizations}
\label{appendix: Visualizations}

\subsection{Comparison of prediction accuracy with state-of-the-art methods}
\label{visual: prediction_comparison}
Figure \ref{fig:qual0} compares our MDMLP-EIA model's prediction accuracy against three state-of-the-art methods (Amplifier, xPatch, and iTransformer) on the Weather dataset ($L = 96$, $T = 192$). As shown in the four panels, MDMLP-EIA (panel a) achieves superior performance by accurately capturing the complex dynamics and critical turning points of the ground truth (blue line) with its predictions (orange line). While baseline models demonstrate reasonable capabilities, they fail to represent nuanced patterns at key points in the time series. MDMLP-EIA better preserves amplitude and timing of peaks and valleys, particularly during complex fluctuations between time steps $150-200$. This enhanced accuracy highlights our approach's effectiveness in modeling temporal dependencies in weather data.

\begin{figure*}[th]
     \centering
     \begin{subfigure}[b]{0.4\textwidth}
         \centering
         \includegraphics[width=1\columnwidth]{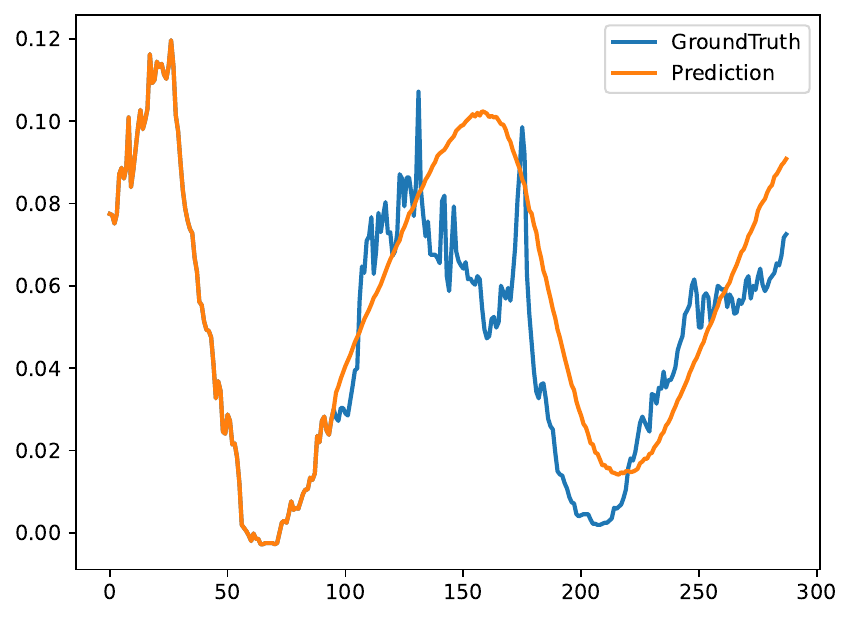}
         \caption{MDMLP-EIA}
     \end{subfigure}
     \begin{subfigure}[b]{0.4\textwidth}
         \centering
         \includegraphics[width=1\columnwidth]{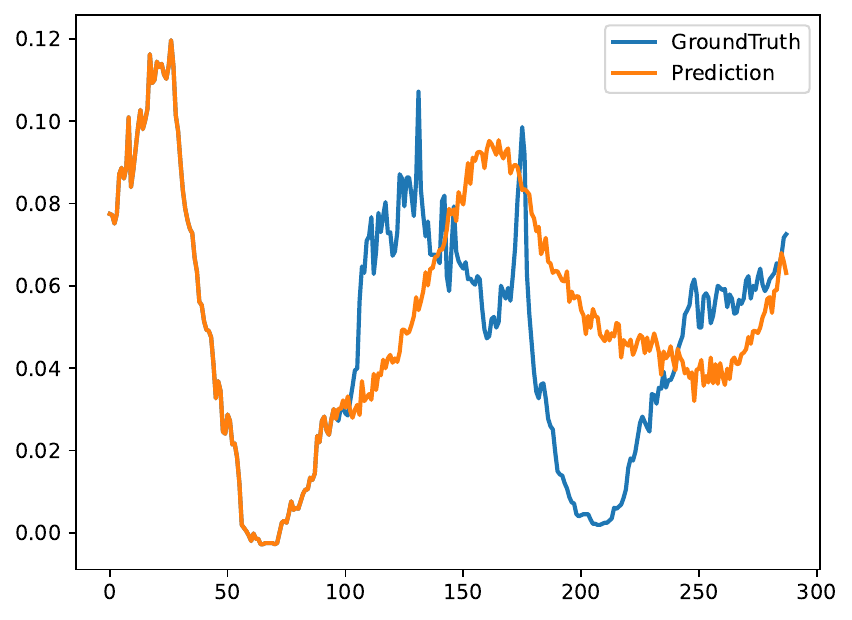}
         \caption{Amplifier}
     \end{subfigure}
     \begin{subfigure}[b]{0.4\textwidth}
         \centering
         \includegraphics[width=1\columnwidth]{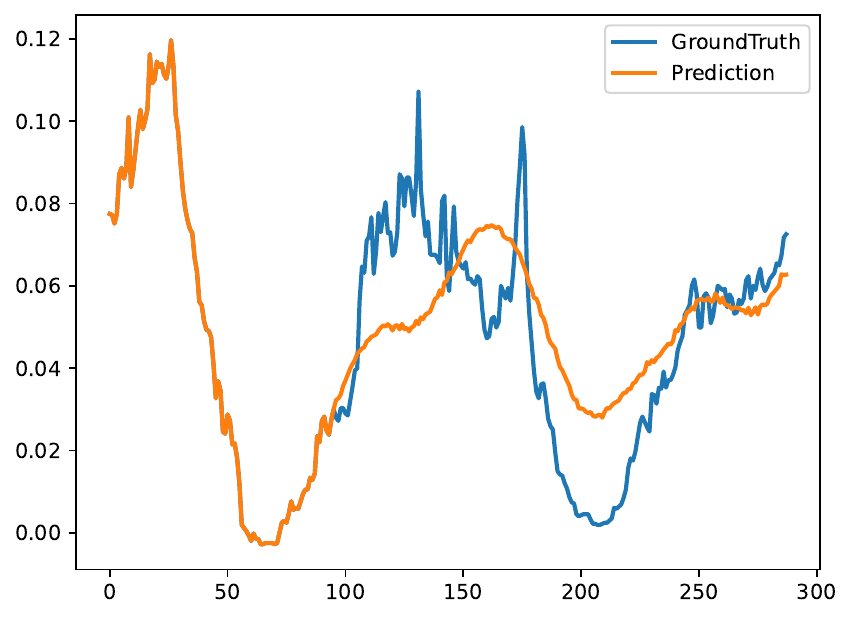}
         \caption{xPatch}
     \end{subfigure}
     \begin{subfigure}[b]{0.4\textwidth}
         \centering
         \includegraphics[width=1\columnwidth]{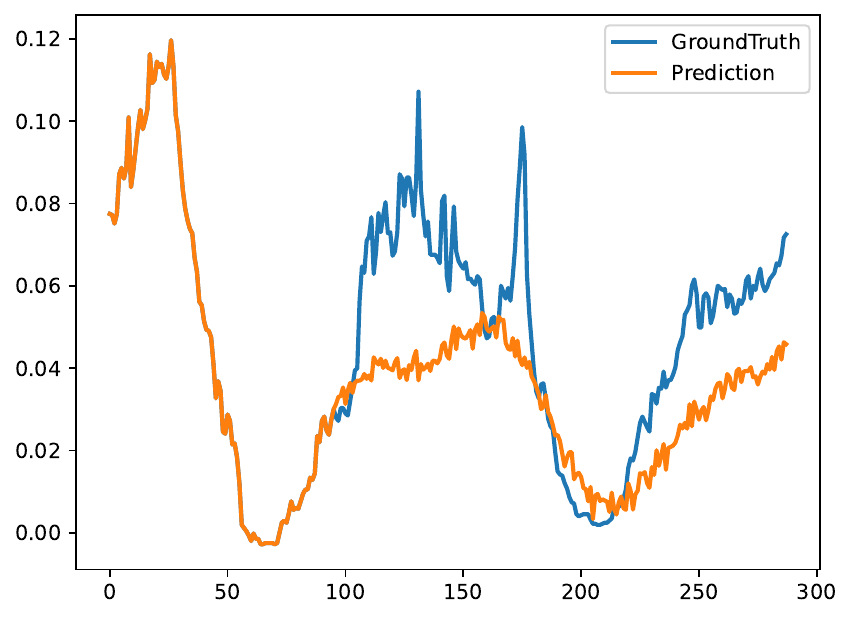}
         \caption{iTransformer}
     \end{subfigure}
        \caption{MDMLP-EIA (a) demonstrates enhanced prediction accuracy on the Weather dataset ($L=96, T=192$), more closely tracking complex dynamics and key turning points of the ground truth than baseline models Amplifier (b), xPatch (c), and iTransformer (d).}
        \label{fig:qual0}
\end{figure*}

Figure \ref{fig:qual1} demonstrates MDMLP-EIA's superior performance in long-horizon forecasting ($L = 96$, $T = 336$) on the Weather dataset compared to three baseline models. Figure \ref{fig:qual1} (a) shows MDMLP-EIA accurately capturing the complex multi-scale dynamics of the ground truth (blue line), particularly at critical turning points across extended periods. The model maintains prediction fidelity (orange line) even at the 336-step horizon. In contrast, baseline models (Amplifier, xPatch, and iTransformer) exhibit declining accuracy over longer horizons, especially between time steps 200-400 where they fail to track the amplitude and timing of key fluctuations accurately. MDMLP-EIA's ability to preserve pattern features throughout the extended forecast period underscores its effectiveness for long-term weather prediction tasks.

\begin{figure*}[th]
     \centering
     \begin{subfigure}[b]{0.4\textwidth}
         \centering
         \includegraphics[width=1\columnwidth]{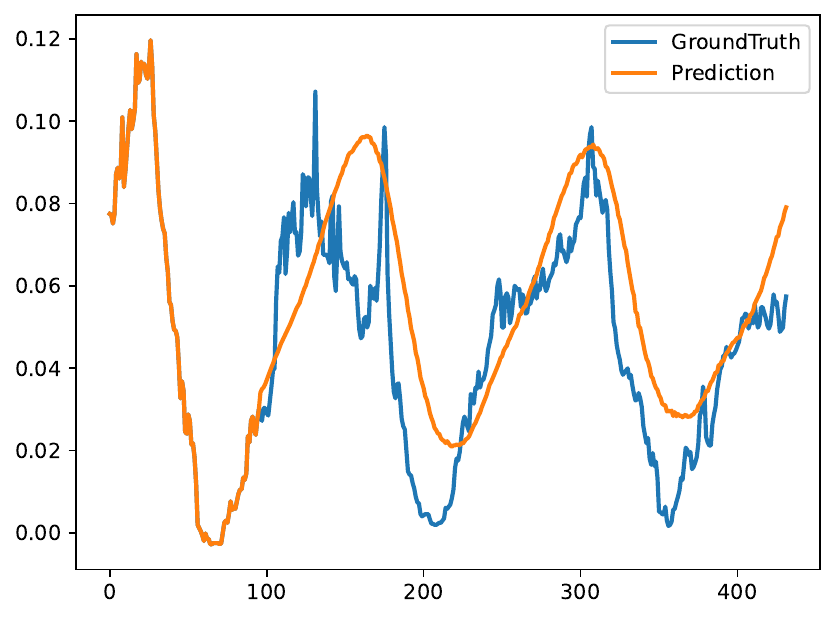}
         \caption{MDMLP-EIA}
     \end{subfigure}
     \begin{subfigure}[b]{0.4\textwidth}
         \centering
         \includegraphics[width=1\columnwidth]{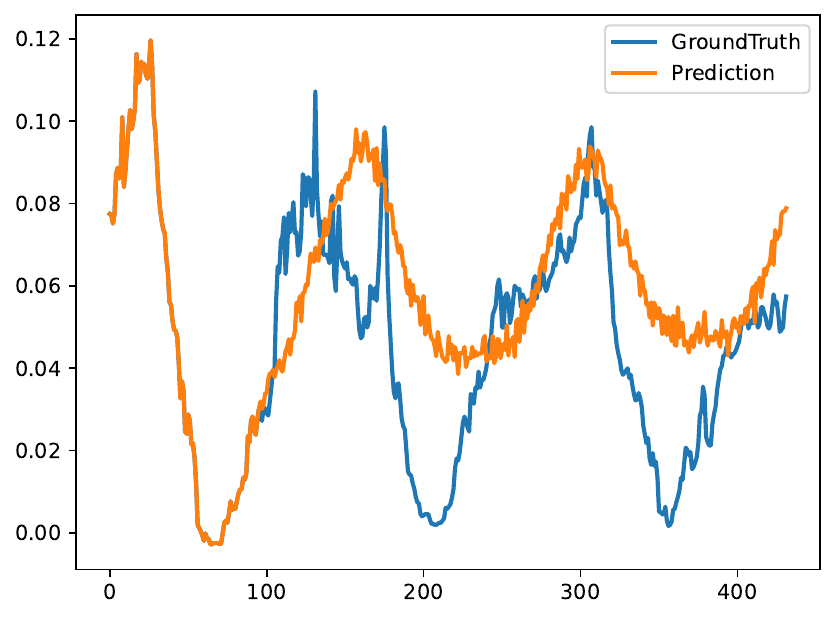}
         \caption{Amplifier}
     \end{subfigure}
     \begin{subfigure}[b]{0.4\textwidth}
         \centering
         \includegraphics[width=1\columnwidth]{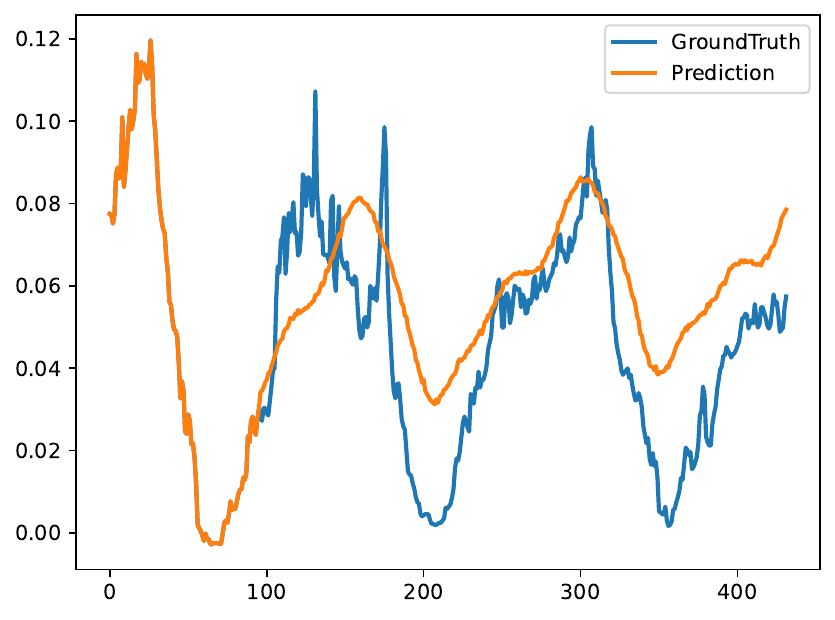}
         \caption{xPatch}
     \end{subfigure}
     \begin{subfigure}[b]{0.4\textwidth}
         \centering
         \includegraphics[width=1\columnwidth]{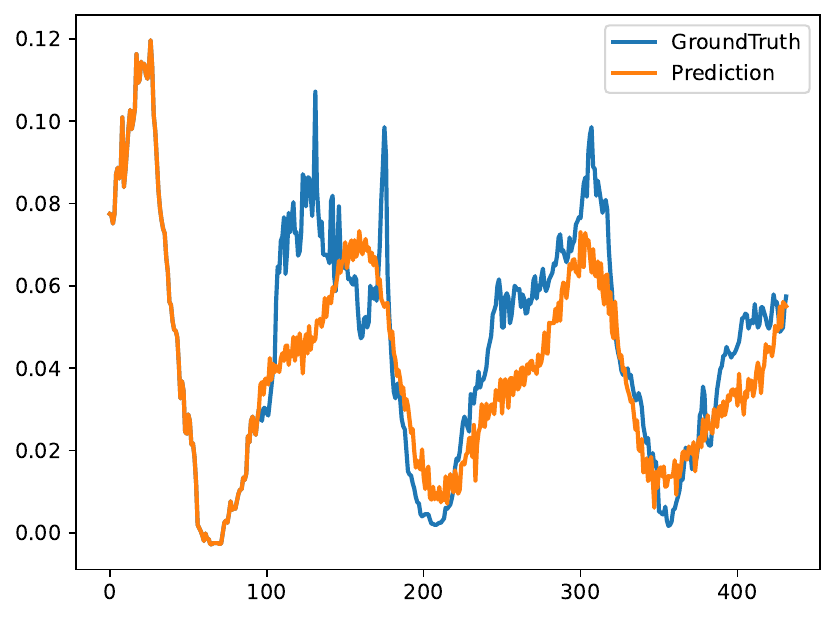}
         \caption{iTransformer}
     \end{subfigure}
        \caption{MDMLP-EIA (a) demonstrates superior long-horizon forecasting ($L=96, T=336$) on the Weather dataset to Amplifier (b), xPatch (c), and iTransformer (d).}
        \label{fig:qual1}
\end{figure*}

Figure \ref{fig:qual2} presents a comparative analysis of four time series forecasting models on the ETTm2 dataset ($L=336, T=96$). MDMLP-EIA demonstrates superior performance in accurately capturing both the periodicity and amplitude of the ground truth signal. MDMLP-EIA predictions closely align with cyclical patterns across the entire prediction horizon. In contrast, Amplifier and xPatch both exhibit consistent amplitude underestimation at peak magnitudes, despite maintaining reasonable periodicity. iTransformer shows the poorest performance, displaying significant waveform distortion and substantial prediction errors that worsen in later stages of the forecast window. These results indicate that MDMLP-EIA offers enhanced predictive capabilities for time series data with complex cyclical patterns, particularly in preserving amplitude characteristics that contemporary models fail to capture accurately.

\begin{figure*}[th]
     \centering
     \begin{subfigure}[b]{0.4\textwidth}
         \centering
         \includegraphics[width=1\columnwidth]{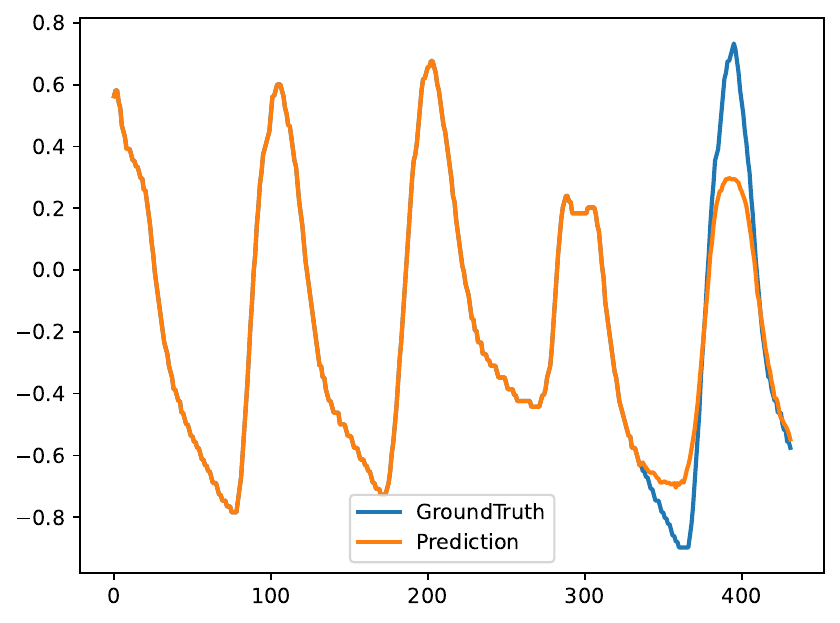}
         \caption{MDMLP-EIA}
     \end{subfigure}
     \begin{subfigure}[b]{0.4\textwidth}
         \centering
         \includegraphics[width=1\columnwidth]{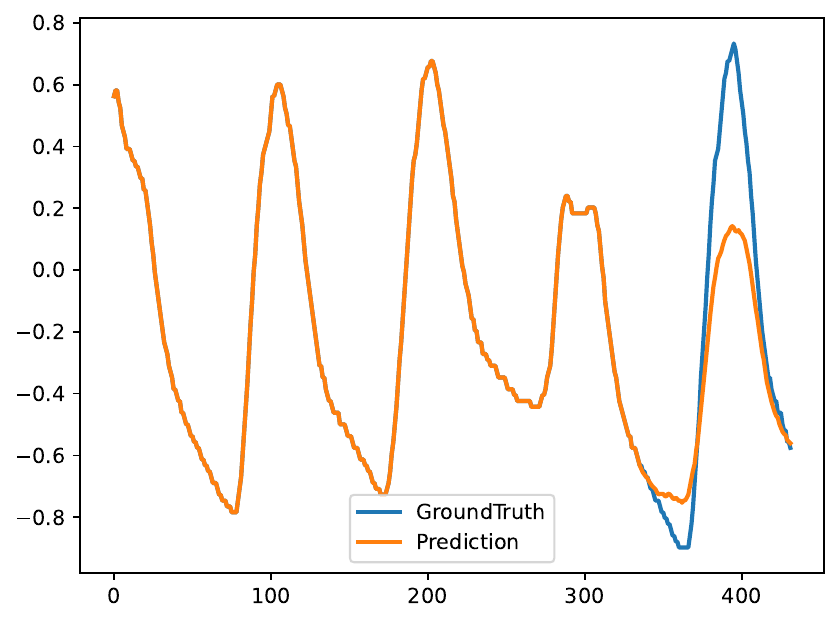}
         \caption{Amplifier}
     \end{subfigure}
     \begin{subfigure}[b]{0.4\textwidth}
         \centering
         \includegraphics[width=1\columnwidth]{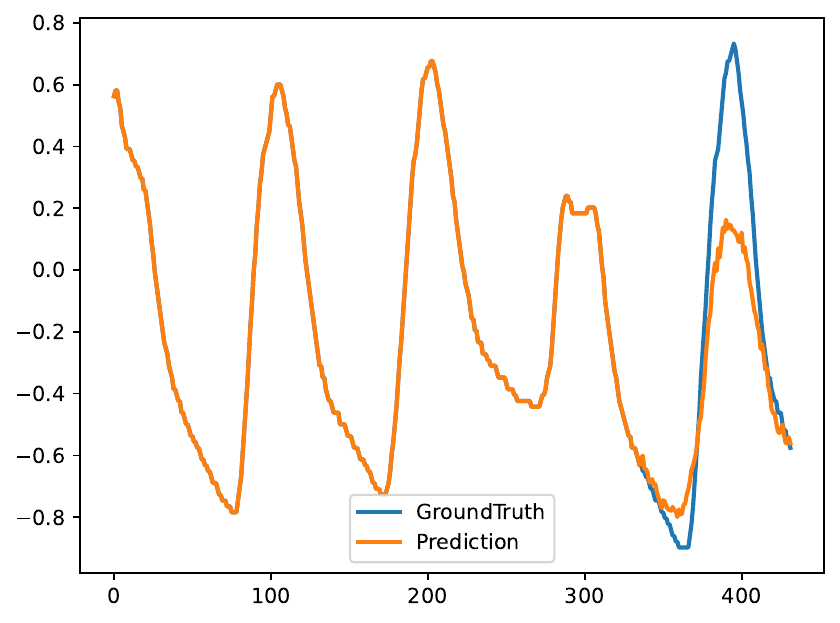}
         \caption{xPatch}
     \end{subfigure}
     \begin{subfigure}[b]{0.4\textwidth}
         \centering
         \includegraphics[width=1\columnwidth]{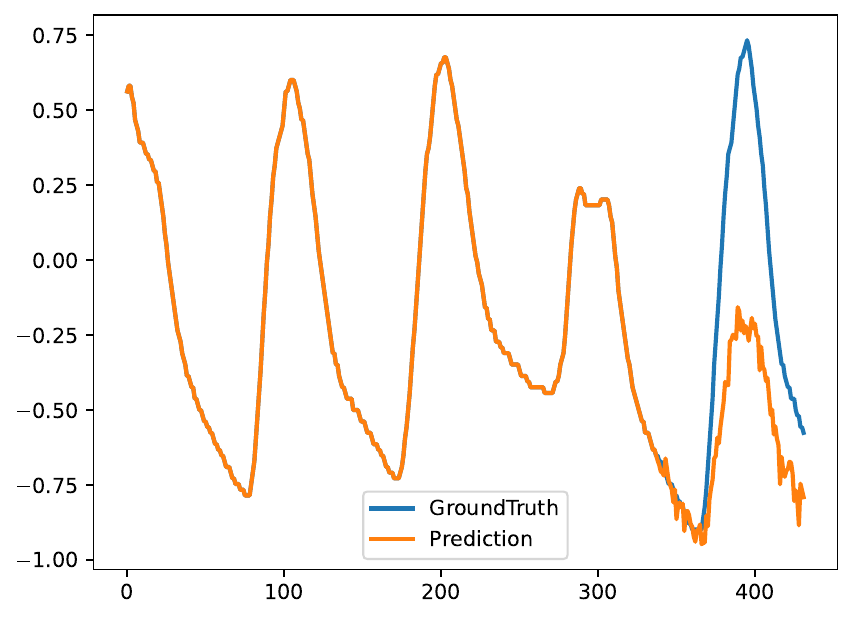}
         \caption{iTransformer}
     \end{subfigure}
        \caption{MDMLP-EIA (a) accurately captures both the periodicity and amplitude of future trends on the ETTm2 dataset ($L=336, T=96$). Its predictions closely replicate the ground truth's cyclical patterns, in contrast to Amplifier (b) and xPatch (c) which underestimate peak magnitudes, and iTransformer (d) which exhibits significant waveform distortion and prediction errors, particularly in later stages.}
        \label{fig:qual2}
\end{figure*}

Figure \ref{fig:qual6} shows a comparative analysis of four time series forecasting models (MDMLP-EIA, Amplifier, xPatch, and iTransformer) on the solar dataset ($L=720, T=96$). 
All models face challenges in accurately capturing high-amplitude peaks in the ground truth data (blue line), with notable performance differences.
MDMLP-EIA demonstrates superior overall performance, achieving the best balance between temporal pattern preservation and amplitude accuracy. While still underestimating highest peak values (around time points 300, 500, and 700), it consistently captures approximately 40-45\% of peak amplitudes—significantly outperforming other models. MDMLP-EIA excels in accurately fitting bottom regions, closely tracking ground truth during low-value periods and maintaining proper baseline levels between peaks.
Amplifier exhibits more pronounced limitations, capturing only about 25-30\% of peak heights while maintaining similar temporal patterns. Its valley region fitting is comparable to MDMLP-EIA, though with slightly less precision in tracking rapid fluctuations.
xPatch shows performance similar to MDMLP-EIA for peak height prediction but demonstrates marginally reduced accuracy in capturing nuanced temporal dynamics, particularly during transition periods. Its valley fitting quality closely resembles MDMLP-EIA's.
Interestingly, iTransformer achieves the highest peak amplitude prediction accuracy, particularly for peaks around time points 400-700, capturing approximately 45-50\% of true peak heights. However, this comes at the cost of temporal consistency, evidenced by slight phase shifts and irregular patterns, especially in middle segments. Furthermore, iTransformer exhibits more noise in valley regions with less stable baseline prediction.
All models adequately predict temporal positioning of major peaks and troughs, confirming that the primary challenge in solar data forecasting relates to amplitude prediction rather than temporal pattern recognition. MDMLP-EIA provides the most balanced performance, with the best combination of peak approximation, valley accuracy, and temporal pattern preservation throughout the extended forecast horizon.

\begin{figure*}[th]
     \centering
     \begin{subfigure}[b]{0.4\textwidth}
         \centering
         \includegraphics[width=1\columnwidth]{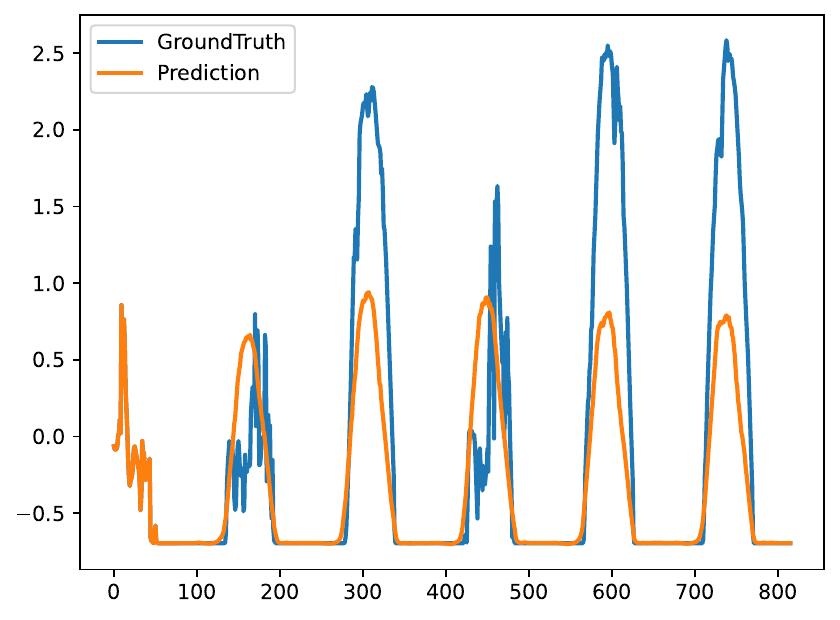}
         \caption{MDMLP-EIA}
     \end{subfigure}
     \begin{subfigure}[b]{0.4\textwidth}
         \centering
         \includegraphics[width=1\columnwidth]{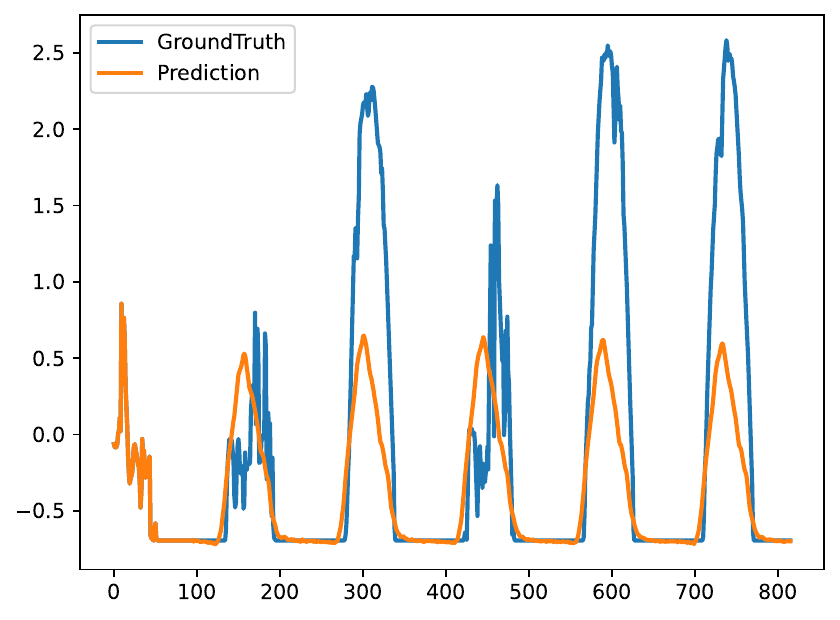}
         \caption{Amplifier}
     \end{subfigure}
     \begin{subfigure}[b]{0.4\textwidth}
         \centering
         \includegraphics[width=1\columnwidth]{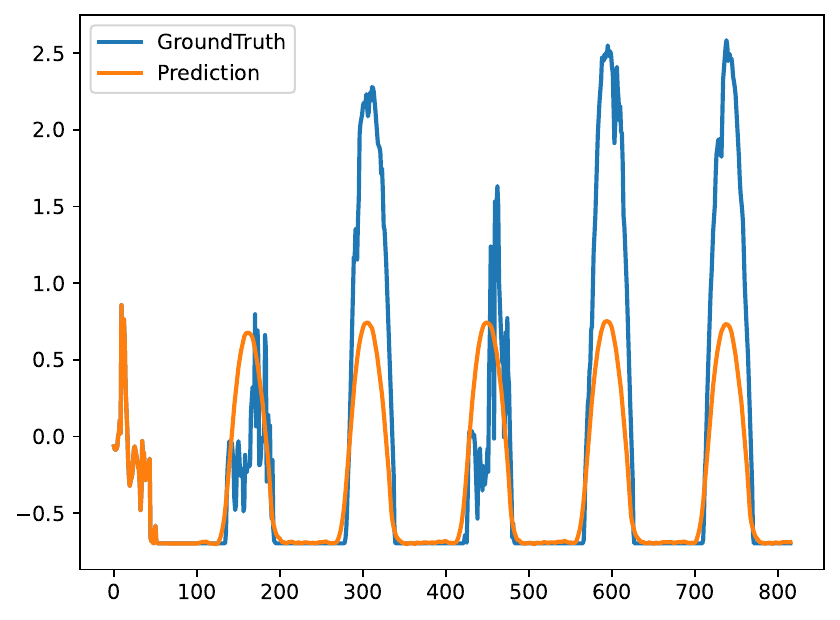}
         \caption{xPatch}
     \end{subfigure}
     \begin{subfigure}[b]{0.4\textwidth}
         \centering
         \includegraphics[width=1\columnwidth]{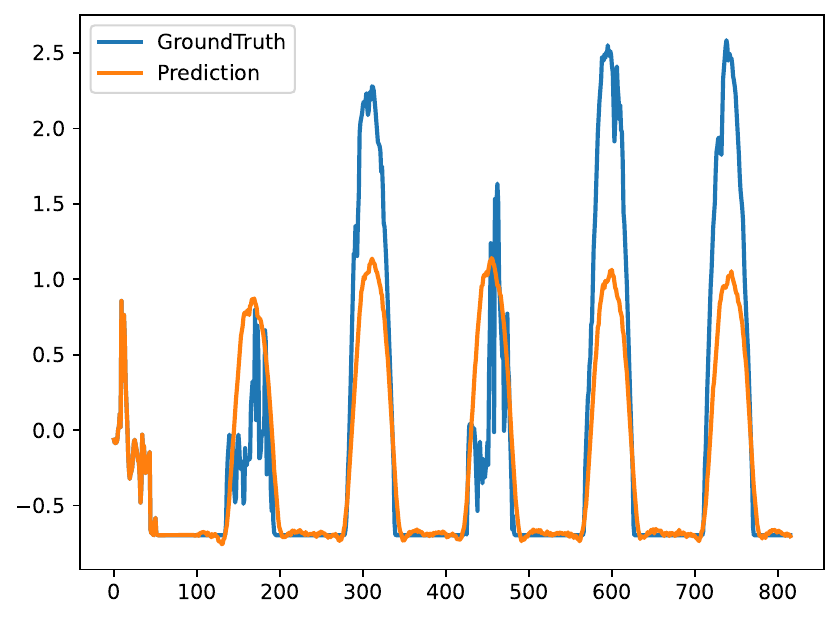}
         \caption{iTransformer}
     \end{subfigure}
        \caption{MDMLP-EIA (a) demonstrates superior long-horizon forecasting ($L=96, T=720$) on the solar dataset to Amplifier (b), xPatch (c), and iTransformer (d).
        }
        \label{fig:qual6}
\end{figure*}

Figure \ref{fig:qual4} presents a comparative analysis of four time series forecasting models(MDMLP-EIA, Amplifier, xPatch, and iTransformer) on the ETTh2 dataset ($L=96, T=192$).
MDMLP-EIA demonstrates strong predictive capability, accurately capturing both periodicity and amplitude of cyclical patterns in the ground truth data (blue line). Its predictions (orange line) closely track temporal dynamics throughout most of the forecast horizon, with only minor amplitude underestimation for some peaks between time points 150-250.
Similarly, xPatch exhibits comparable performance to MDMLP-EIA, accurately replicating cyclical patterns and maintaining reasonable amplitude fidelity across the prediction window, particularly in capturing negative excursions near time points 100 and 250.
Amplifier performs adequately in capturing overall periodicity but shows more noticeable amplitude prediction discrepancies, particularly in the latter half of the forecast horizon (time points 200-300), where it fails to reach extreme values present in the ground truth.
iTransformer exhibits the most significant prediction errors among the four models. While capturing the general oscillatory pattern, it shows substantial amplitude dampening, especially for peaks between time points 200-300, and displays phase shifts indicating temporal misalignment between predictions and ground truth.
All models struggle with capturing the sharp negative spike near time point 250, although MDMLP-EIA and xPatch provide relatively better approximations. This analysis highlights the superior performance of MDMLP-EIA and xPatch for intermediate-horizon forecasting on the ETTh2 dataset, particularly in preserving both temporal patterns and amplitude characteristics.

\begin{figure*}[th]
     \centering
     \begin{subfigure}[b]{0.4\textwidth}
         \centering
         \includegraphics[width=1\columnwidth]{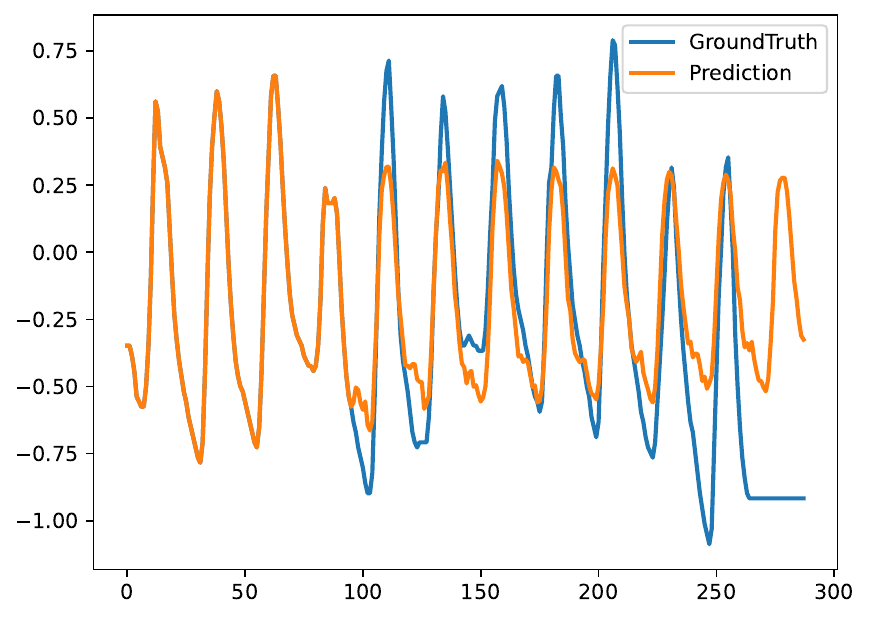}
         \caption{MDMLP-EIA}
     \end{subfigure}
     \begin{subfigure}[b]{0.4\textwidth}
         \centering
         \includegraphics[width=1\columnwidth]{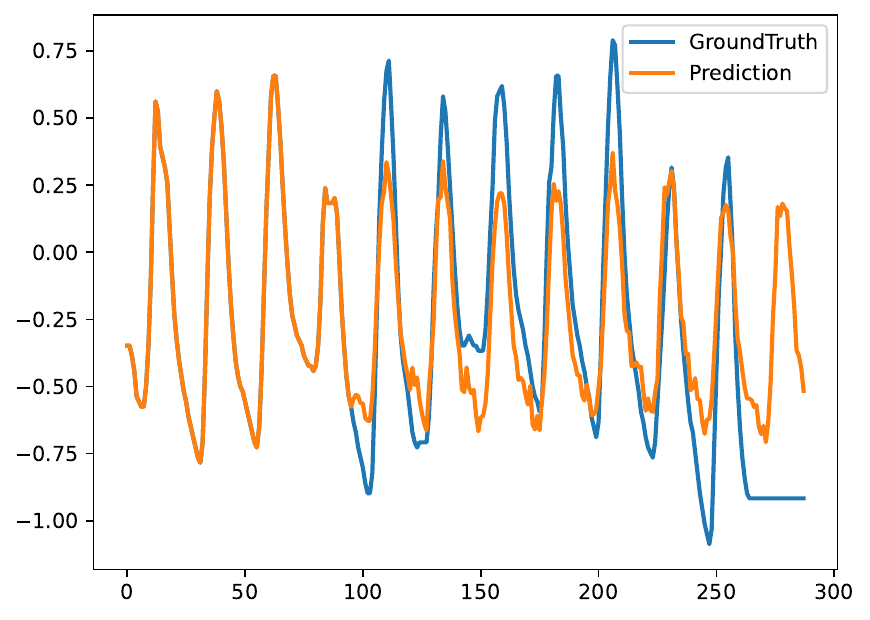}
         \caption{Amplifier}
     \end{subfigure}
     \begin{subfigure}[b]{0.4\textwidth}
         \centering
         \includegraphics[width=1\columnwidth]{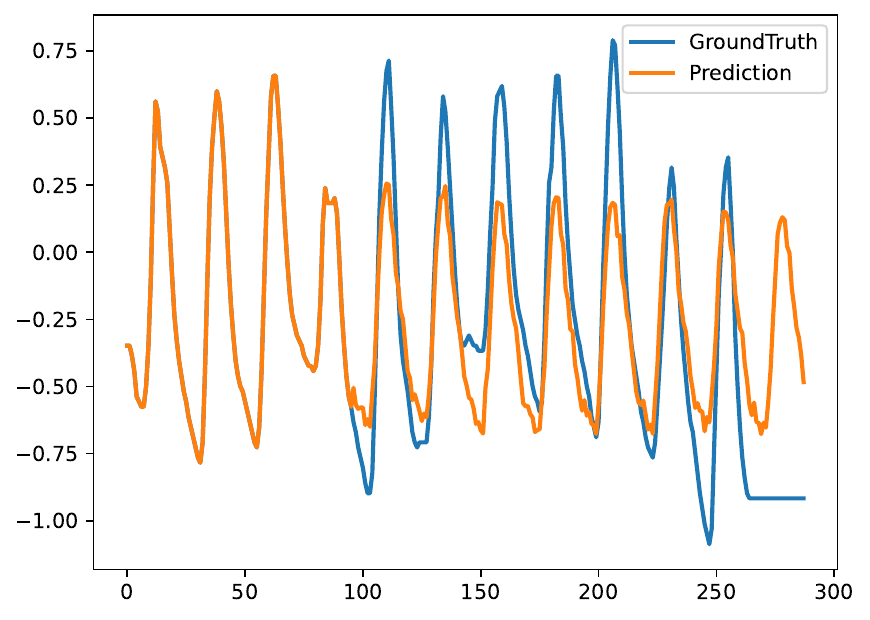}
         \caption{xPatch}
     \end{subfigure}
     \begin{subfigure}[b]{0.4\textwidth}
         \centering
         \includegraphics[width=1\columnwidth]{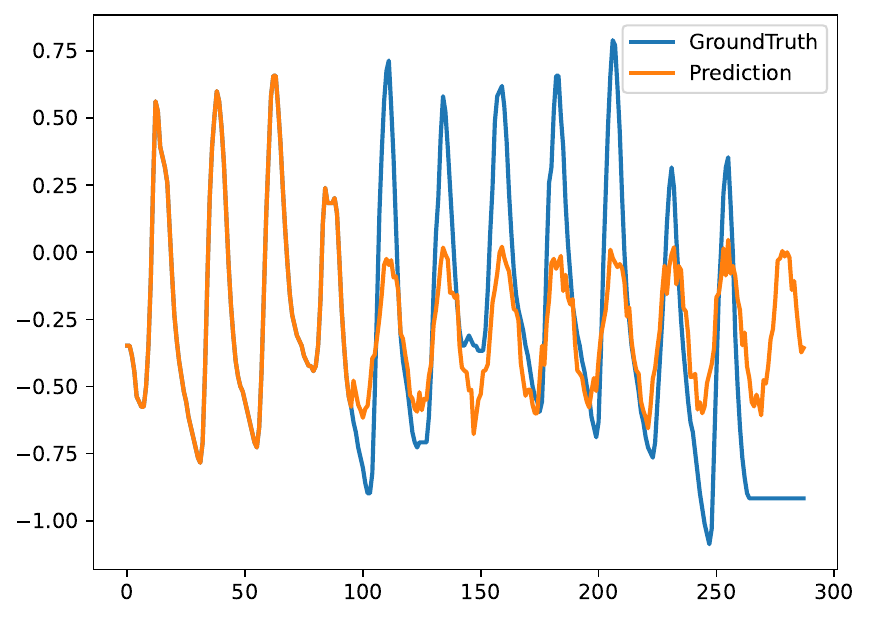}
         \caption{iTransformer}
     \end{subfigure}
        \caption{
        MDMLP-EIA (a) demonstrates superior forecasting ($L=96, T=192$) on the ETTh2 dataset to Amplifier (b), xPatch (c), and iTransformer (d).
        }
        \label{fig:qual4}
\end{figure*}

Figure \ref{fig:qual5} illustrates a comparative evaluation of four time series forecasting models(MDMLP-EIA, Amplifier, xPatch, and iTransformer) on the ETTh2 dataset ($L=336, T=96$).
MDMLP-EIA demonstrates superior forecasting capability, maintaining relatively accurate predictions throughout the extended horizon. It successfully captures both oscillatory patterns and amplitude variations, including notable negative excursions around time points 200 and 400. While some amplitude underestimation occurs for extreme peaks, the model preserves essential temporal dynamics across the entire prediction window.
xPatch exhibits comparable performance to MDMLP-EIA for the first half of the prediction horizon (points 0-200) but shows increasing deviation thereafter. It maintains reasonable periodicity throughout but demonstrates progressive amplitude attenuation for peaks after time point 200, particularly struggling with extreme values between points 300-400.
Amplifier shows adequate performance in predicting general oscillatory patterns but with more pronounced amplitude underestimation compared to MDMLP-EIA, especially for peaks beyond time point 150. The model preserves temporal positioning of negative spikes around points 200 and 400 but fails to accurately capture their magnitude.
iTransformer exhibits the most significant performance degradation. While capturing basic cyclical nature early in the prediction window, it shows substantial amplitude dampening beyond time point 150, with predictions converging toward a narrower range that fails to reflect true variability. It particularly struggles with extreme negative values, showing marked underestimation of these critical features.
This analysis demonstrates that as the prediction horizon extends to 336 points, the performance gap between models becomes more apparent, with MDMLP-EIA maintaining highest fidelity to ground truth patterns, followed by xPatch and Amplifier, while iTransformer exhibits substantial limitations in long-horizon forecasting on this dataset.

\begin{figure*}[th]
     \centering
     \begin{subfigure}[b]{0.4\textwidth}
         \centering
         \includegraphics[width=1\columnwidth]{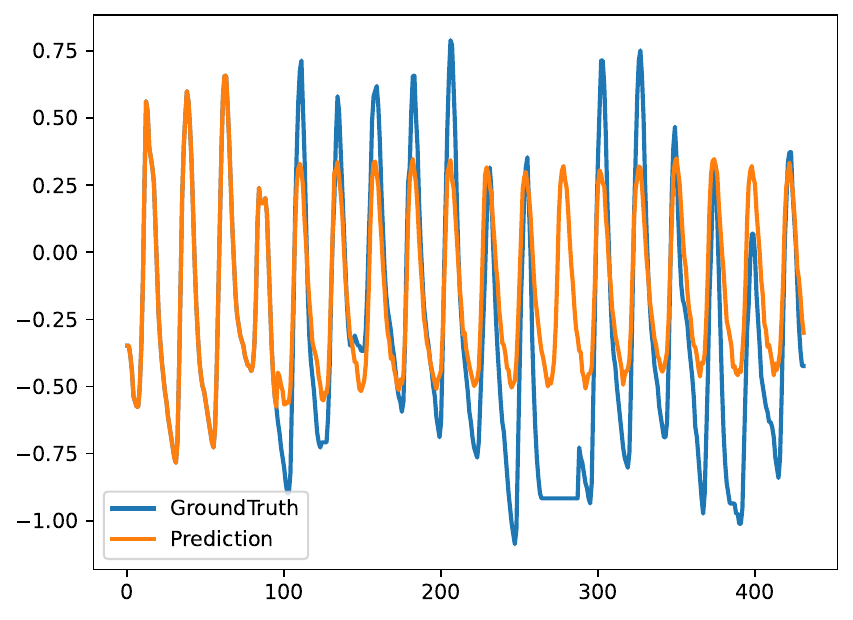}
         \caption{MDMLP-EIA}
     \end{subfigure}
     \begin{subfigure}[b]{0.4\textwidth}
         \centering
         \includegraphics[width=1\columnwidth]{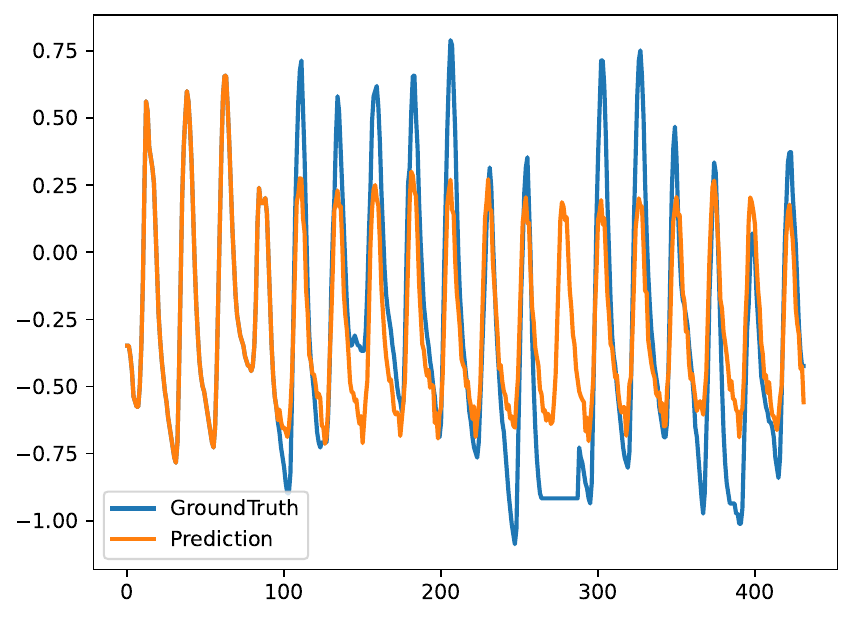}
         \caption{Amplifier}
     \end{subfigure}
     \begin{subfigure}[b]{0.4\textwidth}
         \centering
         \includegraphics[width=1\columnwidth]{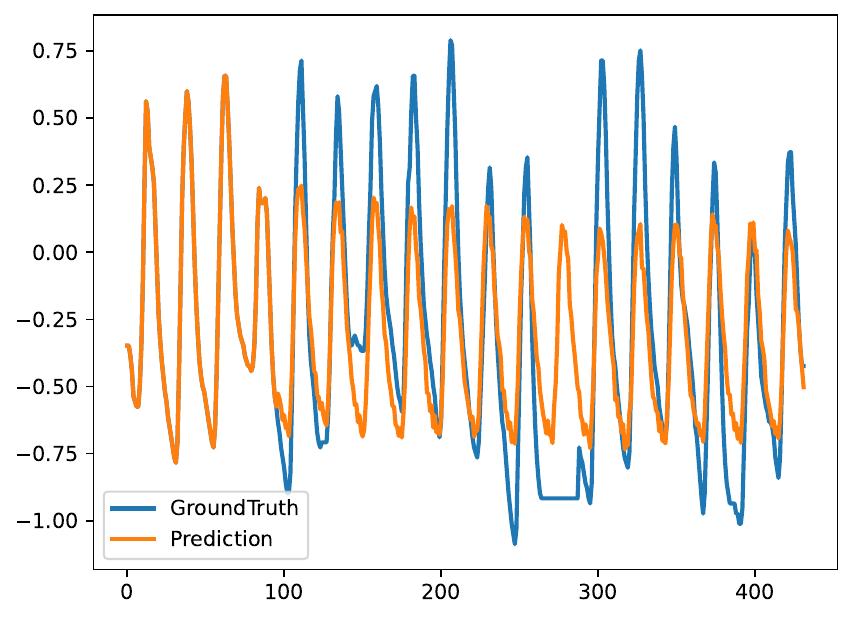}
         \caption{xPatch}
     \end{subfigure}
     \begin{subfigure}[b]{0.4\textwidth}
         \centering
         \includegraphics[width=1\columnwidth]{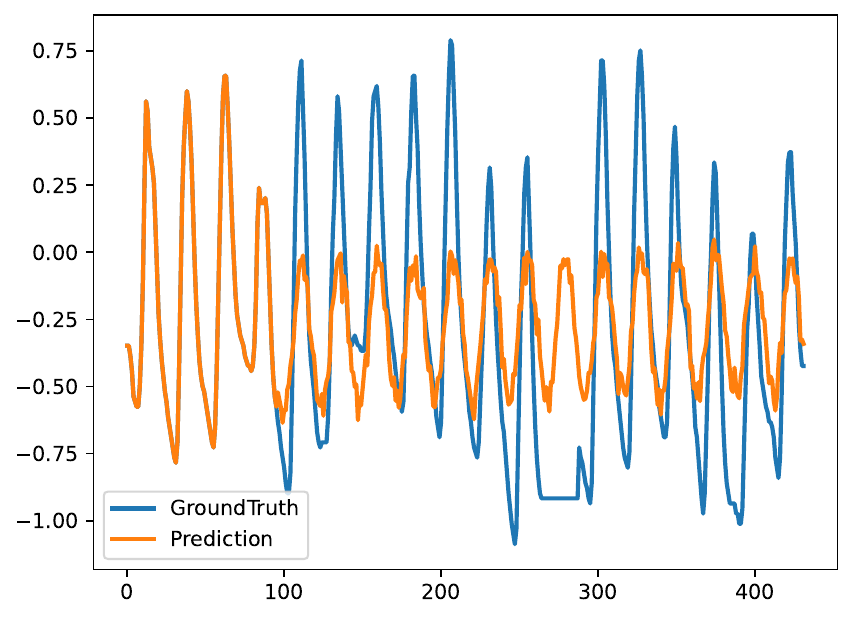}
         \caption{iTransformer}
     \end{subfigure}
        \caption{MDMLP-EIA (a) demonstrates superior forecasting ($L=96, T=336$) on the ETTh2 dataset to Amplifier (b), xPatch (c), and iTransformer (d).
        }
        \label{fig:qual5}
\end{figure*}

Figure \ref{fig:qual3} presents a comparative analysis of four time series forecasting models(MDMLP-EIA, Amplifier, xPatch, and iTransformer) on the ETTh2 dataset ($L=96, T=96$), representing a short-term forecasting scenario with prediction horizon equal to input sequence length.
MDMLP-EIA demonstrates superior performance among all models, exhibiting exceptional accuracy in capturing both temporal dynamics and amplitude characteristics. While showing minor underestimation of peak amplitudes beyond time point 100, it consistently outperforms alternatives by capturing approximately 35-40\% of true peak heights, compared to 30-35\% for other approaches. Notably, MDMLP-EIA excels in precisely tracking the significant negative excursion near time point 100, achieving almost perfect alignment with ground truth at this critical feature. Additionally, it maintains excellent phase consistency throughout the entire prediction horizon without temporal shifts.
Amplifier and xPatch exhibit similar but slightly inferior performance. Both successfully capture cyclical patterns and maintain good temporal alignment with ground truth. However, they demonstrate more pronounced amplitude underestimation for peaks beyond time point 100, capturing only about 30-35\% of peak heights. Their predictions for the negative excursion around time point 100 are also slightly less accurate than MDMLP-EIA, with minor deviations in both depth and timing.
iTransformer shows markedly different performance characteristics. While maintaining reasonable accuracy for the first half of the prediction horizon (points 0-100), it exhibits substantially greater amplitude underestimation for peaks in the latter half (points 100-200), capturing only about 20-25\% of true peak heights. iTransformer's predictions also display more irregular patterns with noticeable phase shifts between time points 150-200, suggesting significant challenges in maintaining consistent temporal alignment.
In this short-term forecasting scenario, while all models capture essential temporal dynamics, MDMLP-EIA clearly outperforms alternatives in amplitude fidelity, phase consistency, and overall prediction accuracy. Its superior ability to maintain high-quality predictions throughout the entire forecast horizon highlights its enhanced capacity for robust time series forecasting, even for complex cyclical patterns with significant amplitude variations.

\begin{figure*}[th]
     \centering
     \begin{subfigure}[b]{0.4\textwidth}
         \centering
         \includegraphics[width=1\columnwidth]{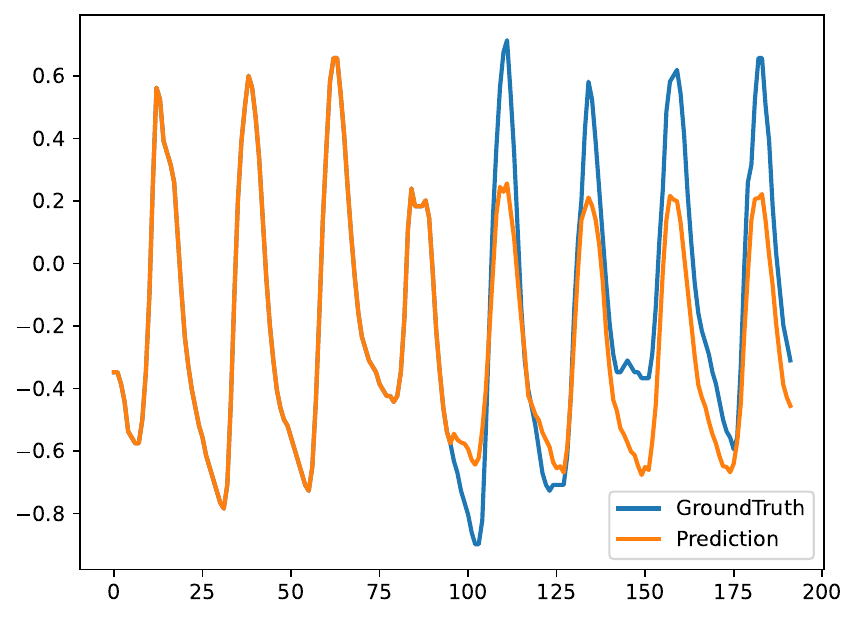}
         \caption{MDMLP-EIA}
     \end{subfigure}
     \begin{subfigure}[b]{0.4\textwidth}
         \centering
         \includegraphics[width=1\columnwidth]{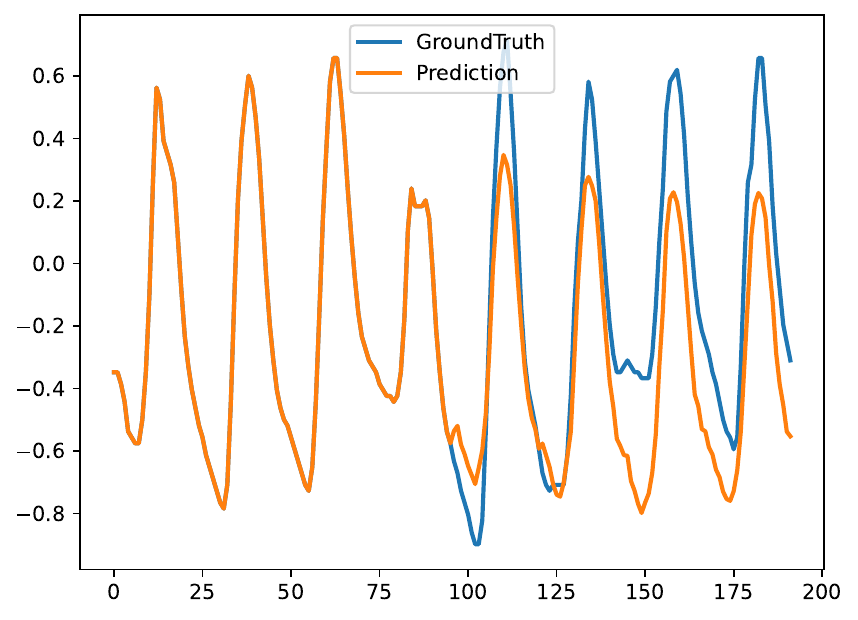}
         \caption{Amplifier}
     \end{subfigure}
     \begin{subfigure}[b]{0.4\textwidth}
         \centering
         \includegraphics[width=1\columnwidth]{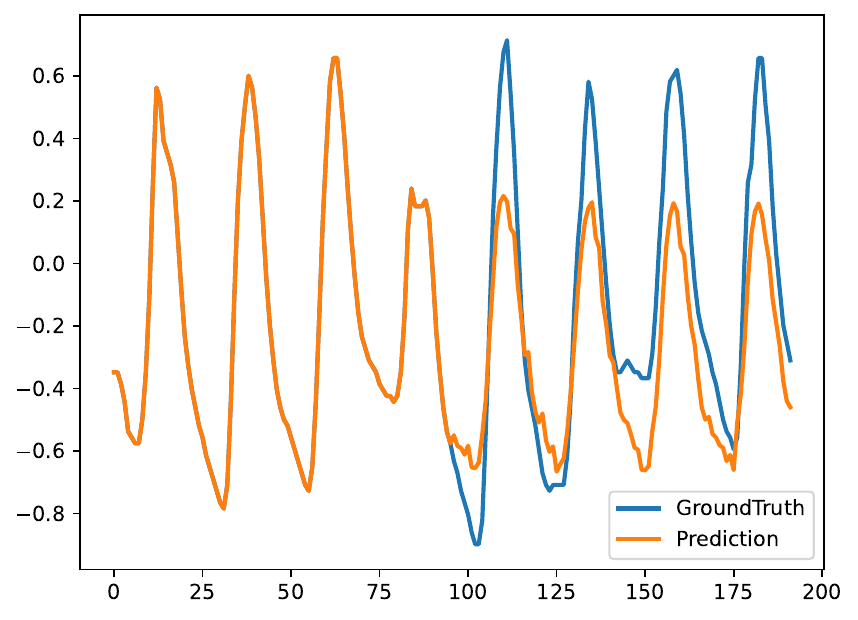}
         \caption{xPatch}
     \end{subfigure}
     \begin{subfigure}[b]{0.4\textwidth}
         \centering
         \includegraphics[width=1\columnwidth]{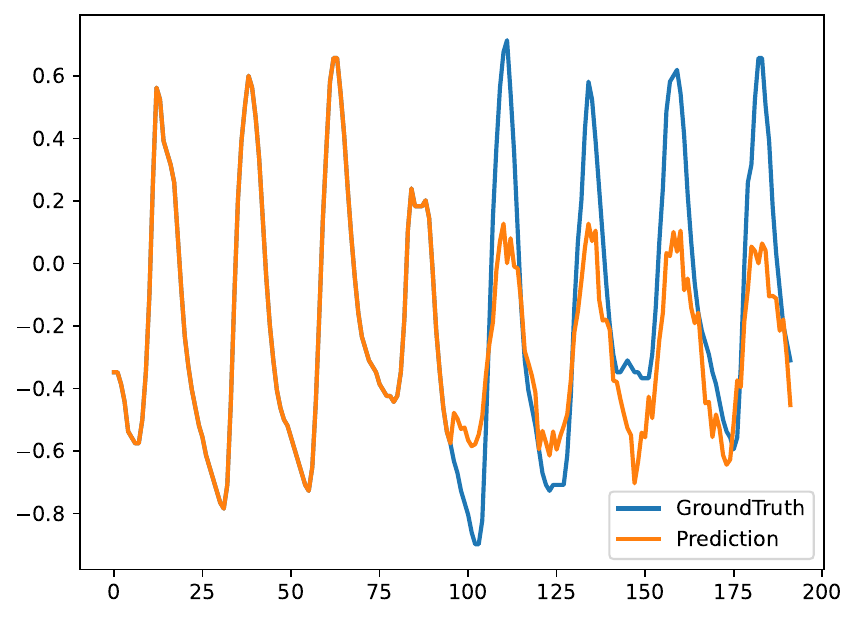}
         \caption{iTransformer}
     \end{subfigure}
        \caption{
        MDMLP-EIA (a) demonstrates superior forecasting ($L=96, T=96$) on the ETTh2 dataset to Amplifier (b), xPatch (c), and iTransformer (d).
        }
        \label{fig:qual3}
\end{figure*}

\subsection{Visualization of AZCF Validation}
\label{visual: AZCF}

Figure \ref{fig:ablation_weak_signal_etth2_overall} presents a comparative analysis of three approaches for extreme long-horizon forecasting ($L=96$, $T=720$) on the ETTh2 dataset. The proposed Adaptive Zero-Initialized Channel Fusion (AZCF) method demonstrates remarkable capability in accurately predicting complex time series patterns over extended horizons.
AZCF successfully captures the oscillatory behavior in the ground truth data (blue line), maintaining appropriate amplitude ranges and accurately tracking both high-frequency fluctuations and deeper negative excursions, particularly during later stages (time points 600-800). This performance is notably effective in preserving sharp, large-scale dynamic shifts in the latter portion of the prediction window.
In contrast, the approach discarding weak signals exhibits pronounced limitations, with significant amplitude decay as the prediction horizon extends. While it captures some periodic patterns, it fails to replicate the full amplitude range, especially for extreme negative values. 
Similarly, the standard MLP fusion approach shows progressively increasing trend deviation from the ground truth, particularly struggling with negative excursions while producing a more regularized oscillatory pattern that increasingly diverges from actual data dynamics.
The comparison clearly demonstrates that both baseline methods fail to retain critical dynamic information in extended forecasts, while AZCF maintains essential temporal patterns and amplitude characteristics throughout this challenging long-horizon prediction task.

\begin{figure*}[th]
  \centering
  \begin{subfigure}[b]{0.32\textwidth}
    \centering
    \includegraphics[width=\linewidth]{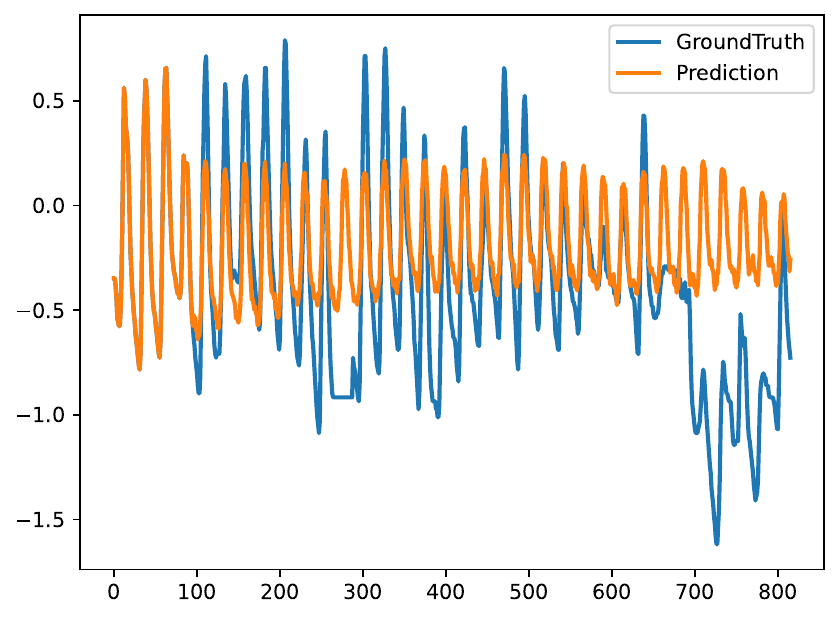} 
    \caption{w/o Weak signal}
    \label{fig:ablation_ws_wo_etth2_sl96_pl720} 
  \end{subfigure}
  \hfill 
  \begin{subfigure}[b]{0.32\textwidth}
    \centering
    \includegraphics[width=\linewidth]{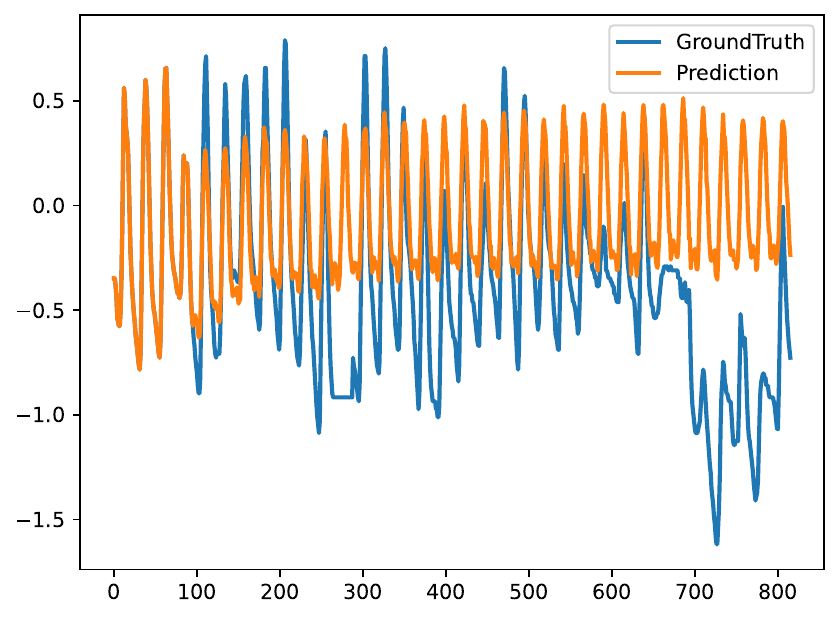} 
    \caption{MLP fusion}
    \label{fig:ablation_ws_mlp_etth2_sl96_pl720} 
  \end{subfigure}
  \hfill
  \begin{subfigure}[b]{0.32\textwidth}
    \centering
    \includegraphics[width=\linewidth]{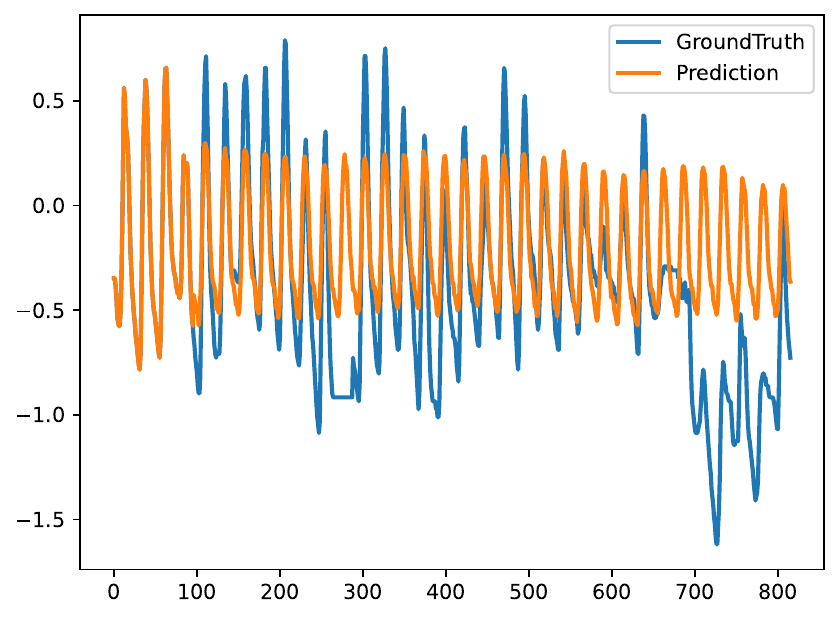} 
    \caption{AZCF (Ours)}
    \label{fig:ablation_ws_adaptive_etth2_sl96_pl720} 
  \end{subfigure}
 \caption{Our Adaptive Zero-Initialized Channel Fusion ('AZCF (Ours)', subplot~(c)) demonstrates significant advantages in extreme long-horizon forecasting ($L=96, T=720$) on the challenging ETTh2 dataset. It successfully captures sharp, large-scale dynamic shifts, especially in the later prediction stages. This contrasts sharply with alternative approaches: discarding weak signals ('w/o Weak signal', subplot~(a)) suffers from significant amplitude decay over longer horizons, while standard MLP fusion ('MLP fusion', subplot~(b)) exhibits a progressively increasing trend deviation from the ground truth. Both baseline methods consequently fail to retain critical dynamic information in extended forecasts.}
  \label{fig:ablation_weak_signal_etth2_overall} 
\end{figure*}

\subsection{Visualization of EIA validation}
\label{visual: EIA}
Figure \ref{fig:amplifier_ablation_weather_sl96pl720_overall} presents an ablation study examining different seasonal-trend fusion mechanisms within the Amplifier model framework on the Weather dataset under extreme long-horizon forecasting conditions ($T=720$, $L=96$). The analysis compares three fusion approaches: direct addition (ADD), MLP-based fusion, and Energy Invariant Attention (EIA).
The EIA mechanism demonstrates superior performance by effectively preserving both amplitude and temporal patterns of the ground truth data throughout the extended prediction horizon. The prediction line accurately captures both periodic oscillations and varying amplitudes of key fluctuations. 
In contrast, the ADD mechanism shows progressive degradation in prediction quality, with diminishing ability to maintain proper amplitude and crucial details as the forecast extends further. 
The MLP-based fusion approach exhibits the poorest performance, failing to generate meaningful predictions for this challenging scenario, as evidenced by its inability to track temporal patterns or amplitude variations consistently.
These results highlight EIA's effectiveness for maintaining prediction integrity over extreme long-horizon forecasting tasks, particularly in complex time series with multiple seasonal and trend components such as weather data.

\begin{figure*}[th] 
  \centering
  \begin{subfigure}[b]{0.32\textwidth} 
    \centering
    \includegraphics[width=\linewidth]{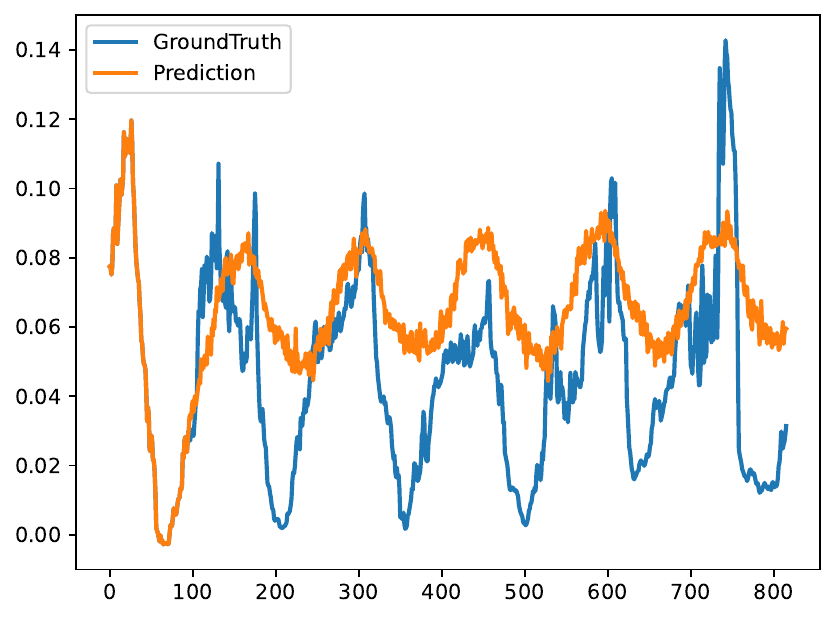}
    \caption{Amplifier (ADD)}
    \label{fig:amp_ablation_add_weather_sl96_pl720} 
  \end{subfigure}
  \hfill
  \begin{subfigure}[b]{0.32\textwidth}
    \centering
    \includegraphics[width=\linewidth]{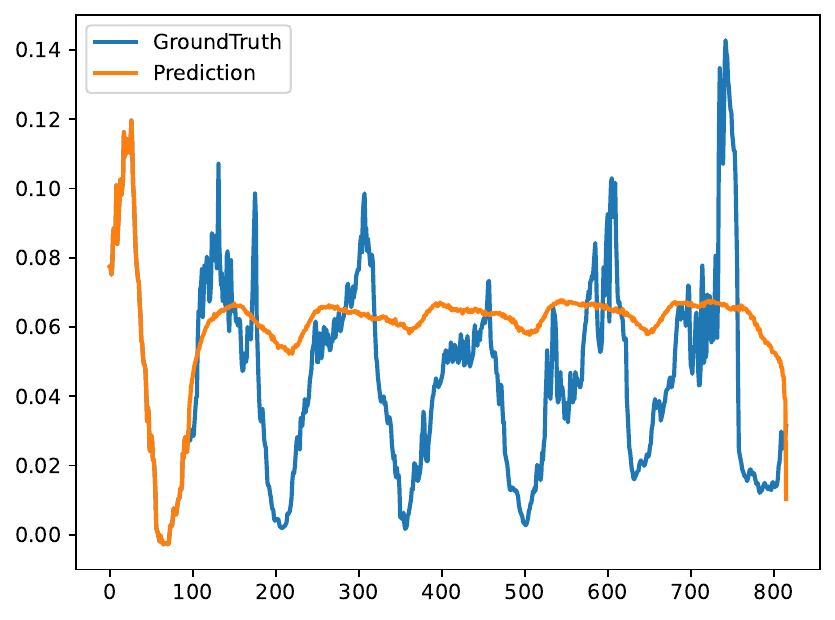}
    \caption{Amplifier (MLP)}
    \label{fig:amp_ablation_mlp_weather_sl96_pl720} 
  \end{subfigure}
  \hfill 
  \begin{subfigure}[b]{0.32\textwidth} 
    \centering
    \includegraphics[width=\linewidth]{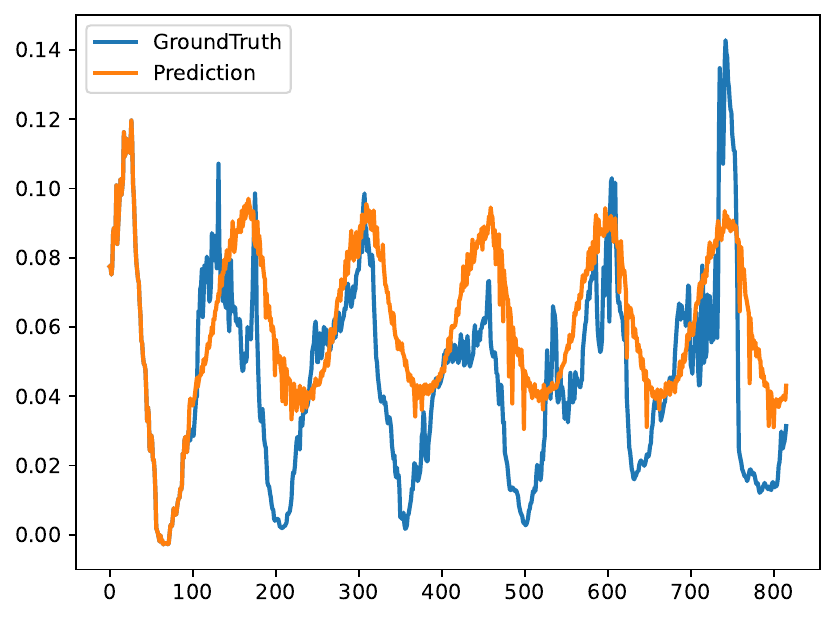}
    \caption{Amplifier (EIA)}
    \label{fig:amp_ablation_attn_weather_sl96_pl720}
  \end{subfigure}
  \caption{Ablation study of seasonal-trend fusion mechanisms within the Amplifier model on the Weather dataset for extreme long-horizon forecasting ($T=720$, with lookback $L=96$). Our proposed Energy Invariant Attention ('EIA', subplot (c)) demonstrates significantly enhanced robustness and accuracy in maintaining long-term dynamic trends and crucial fluctuations compared to direct addition ('ADD', subplot (a)), which loses amplitude and detail over time, and MLP-based fusion ('MLP', subplot (b)), which fails to provide meaningful predictions in this challenging scenario.}
  \label{fig:amplifier_ablation_weather_sl96pl720_overall} 
\end{figure*}


\end{document}